\documentclass[11pt]{article}
 \pdfoutput=1
\usepackage[margin=1in]{geometry}
\newenvironment{proof}[1][Proof]{\begin{trivlist}
\item[\hskip \labelsep {\bfseries #1}]}{\end{trivlist}}
\newcommand{\qed}{\hfill \ensuremath{\Box}}

\title{ \bf
An Incremental Sampling-based Algorithm for Stochastic Optimal Control
}

 \author{Vu Anh Huynh \and Sertac Karaman \and  Emilio Frazzoli
 \thanks{The authors are with the Laboratory of Information and Decision Systems,  Massachusetts Institute of Technology, 77 Massachusetts Ave., Cambridge, MA 02139. {\tt\small $\lbrace\tt vuhuynh,sertac,frazzoli \rbrace$@mit.edu}}
 }

\date{ }

\usepackage{amsmath}
\usepackage{amssymb}
\usepackage[vlined,linesnumbered,boxruled]{algorithm2e}
\usepackage{graphicx}
\usepackage{epsfig}
\usepackage{subfigure}
\usepackage{multicol}
\usepackage{color}

\newcommand{\reals}{\ensuremath{\mathbb{R}}}
\newcommand{\probmeasure}{\ensuremath{\mathbb{P}}}
\newcommand{\naturals}{\ensuremath{\mathbb{N}}}
\newcommand{\argmin}{\ensuremath{\text{argmin}}}
\newcommand{\integers}{\ensuremath{\mathbb{Z}}}

\newcommand{\PP}{\ensuremath{\mathbb{P}}}
\newcommand{\EE}{\ensuremath{\mathbb{E}}}
\newcommand{\Cov}{\ensuremath{\mathrm{Cov}}}
\newcommand{\given}{\ensuremath{\,\vert\,}}

\newcommand{\FirstExit}[1]{{\ensuremath{T_{#1}}}}
\newcommand{\CostToGo}[2]{{\ensuremath{J_{#2}(#1)}}}
\newcommand{\OptimalCostToGo}[1]{{\ensuremath{J^*(#1)}}}

\newcommand{\DummyS}{Y}

\newcommand{\dispersion}{\zeta}
\newcommand{\dummyrate}{\varsigma}
\newcommand{\dummyasyn}{\theta}

\definecolor{darkgreen}{rgb}{0,0.5,0} 
\definecolor{fullred}{rgb}{0.85,.0,.1} 
\definecolor{brown}{rgb}{0.65,0.16,0.16} \newcommand{\br}{\color{brown}}
\newcommand{\notesk}[1]{\marginpar{\tiny\noindent{\raggedright\color{blue} {[SK]}\br{ #1}} \par}}\renewcommand{\notesk}[1]{}

\newcommand{\todofootnote}[1]{{\let\thefootnote\relax\footnotetext{#1}}}

\usepackage{theorem}
\newtheorem{theorem}{Theorem}
\newtheorem{lemma}[theorem]{Lemma}

\begin{document}

\maketitle

\begin{abstract}
In this paper, we consider a class of continuous-time, continuous-space stochastic optimal control problems. Building upon recent advances in Markov chain approximation methods and sampling-based algorithms for deterministic path planning,
we propose a novel algorithm called the incremental Markov Decision Process (iMDP) to compute incrementally control policies that approximate arbitrarily well an optimal policy in terms of the expected cost.
The main idea behind the algorithm is to generate a sequence of finite discretizations of the original problem through random sampling of the state space. At each iteration, the discretized problem is a Markov Decision Process that serves as an incrementally refined model of the original problem. 
We show that with probability one, (i) the sequence of the optimal value functions for each of the discretized problems converges uniformly to the optimal value function of the original stochastic optimal control problem, and (ii) the original optimal value function can be computed efficiently in an incremental manner using asynchronous value iterations. Thus, the proposed algorithm provides an anytime approach to the computation of optimal control policies of the continuous problem. The effectiveness of the proposed approach is demonstrated on motion planning and control problems in cluttered environments in the presence of process noise.

\end{abstract}

\section{Introduction}

Stochastic optimal control has been an active research area for several decades with many applications in diverse fields ranging from finance, management science and economics~\cite{fleming.stein.jour_bank_finance04,sethi.thompson.book06} to biology~\cite{todorov.neural_comp05} and robotics~\cite{thrun2001}. 
Unfortunately, general continuous-time, continuous-space stochastic optimal control problems do not admit closed-form or exact algorithmic solutions and are known to be computationally challenging~\cite{Vincent.Tsitsiklis.Complexity.2000}. 
Many algorithms are available to compute approximate solutions of such problems. For instance, a popular approach is based on the numerical solution of the associated Hamilton-Jacobi-Bellman PDE  (see, e.g.,~\cite{Lars.FiniteDifference.HJB.1997,Wang:2003:NSH:859448.859457,Boulbrachene.FiniteEllement.HJB.2004}). Other methods approximate a continuous problem with a discrete Markov Decision Process (MDP), for which an exact solution can be computed in finite time~\cite{chow.tsitsiklis.autocontrol.1991,munos.machinelearning.2001}. However, the complexity of these two classes of deterministic algorithms scales exponentially with the dimension of the state and control spaces, due to discretization. 
Remarkably, algorithms based on random (or quasi-random) sampling of the state space provide a possibility to alleviate the curse of dimensionality in the case in which the control inputs take values from a finite set, as noted in~\cite{Rust-97,Rust:97b, Vincent.Tsitsiklis.Complexity.2000}. 

Algorithms based on random sampling of the state space have recently been shown to be very effective, both in theory and in practice, for computing solutions to deterministic path planning problems in robotics and other disciplines. For example, the Probabilistic RoadMap (PRM) algorithm first proposed by Kavraki et al.~\cite{Kavraki96probabilisticroadmaps} was the first practical planning algorithm that could handle high-dimensional path planning problems. Their incremental counterparts, such as RRT~\cite{Lavalle.98}, later emerged as sampling-based algorithms suited for online applications and systems with differential constraints on the solution (e.g., dynamical systems). The RRT algorithm has been used in many applications and demonstrated on various robotic platforms~\cite{DBLP:conf/icra/KimO03,Kuwata.Teo.ea:TCST09}.
Recently, optimality properties of such algorithms were analyzed in~\cite{Karaman.Frazzoli:IJRR10}. In particular, it was shown that the RRT algorithm fails to converge to optimal solutions with probability one. The authors have proposed the RRT$^*$ algorithm which guarantees almost-sure convergence to globally optimal solutions without any substantial computational overhead when compared to the RRT. 

Although the RRT$^*$ algorithm is asymptotically optimal and computationally efficient (with respect to RRT), it can not handle problems involving systems with uncertain dynamics. In this work, building upon the Markov chain approximation method~\cite{Kushner2000} and the rapidly-exploring sampling technique~\cite{Lavalle.98}, we introduce a novel algorithm called the incremental Markov Decision Process (iMDP) to approximately solve a wide class of stochastic optimal control problems. More precisely, we consider a continuous-time optimal control problem with continuous state and control spaces, full state information, and stochastic process noise. 
In iMDP, we iteratively construct a sequence of discrete Markov Decision Processes (MDPs) as discrete approximations to the original continuous  problem,  as follows. 
Initially, an empty MDP model is created.
At each iteration, the discrete MDP is refined by adding new states sampled from the boundary as well as from the interior of the state space. Subsequently, new stochastic transitions are constructed to connect the new states to those already in the model. For the sake of efficiency, stochastic transitions are computed only when needed. Then, an anytime policy for the refined model is computed using an incremental value iteration algorithm, based on the value function of the previous model. The policy for the discrete system is finally converted to a policy for the original continuous problem. This process is iterated until convergence.

Our work is mostly related to the Stochastic Motion Roadmap (SMR) algorithm~\cite{alterovitz.simeon.ea.robotics.2008} and Markov chain approximation methods~\cite{Kushner2000}. The SMR algorithm constructs an MDP over a sampling-based roadmap representation to maximize the probability of reaching a given goal region. However, in SMR, actions are discretized, and the algorithm does not offer any formal optimality guarantees. On the other hand, while available Markov chain approximation methods~\cite{Kushner2000} provide formal optimality guarantees under very general conditions, a sequence of \textit{a priori} discretizations of state and control spaces still impose expensive computation. The iMDP algorithm addresses this issue by sampling in the state space and sampling or discovering necessary controls.

The main contribution of this paper is  a method to incrementally refine a discrete model of the original continuous problem in a way that ensures convergence to optimality while maintaining low time and space complexity. We show that with probability one, the sequence of optimal value functions induced by optimal control policies for each of the discretized problems converges uniformly to the optimal value function of the original stochastic control problem. In addition, the optimal value function of the original problem can be computed efficiently in an incremental manner using asynchronous value iterations. Thus, the proposed algorithm provides an anytime approach to the computation of optimal control policies of the continuous problem. Distributions of approximating trajectories and control processes returned from the iMDP algorithm approximate arbitrarily well distributions of optimal trajectories and optimal control processes of the original problem. Each iteration of the iMDP algorithm can be implemented with the time complexity $O(k^{\theta}\log{k})$ where $0< \theta \leq 1$ while the space complexity is $O(k)$, where $k$ is the number of states in an MDP model in the algorithm which increases linearly due to the sampling strategy. Thus, the entire processing time until the algorithm stops can be implemented in $O(k^{1+\theta}\log{k})$. Hence, the above space and time complexities make iMDP a practical incremental algorithm. The effectiveness of the proposed approach is demonstrated on motion planning and control problems in cluttered environments in the presence of process noise.

This paper is organized as follows. In Section \ref{section:problem}, a formal problem definition is given. The Markov chain approximation methods and the iMDP algorithm are described in Sections~\ref{section:algorithm:cmc} and ~\ref{section:algorithm}. The analysis of the iMDP algorithm is presented in Section~\ref{section:analysis}. Section~\ref{section:experiments} is devoted to simulation examples and experimental results. The paper is concluded with remarks in Section~\ref{secConclusions}. We provide additional notations and preliminary results as well as proofs for theorems and lemmas in Appendix.

\section{Problem Definition} \label{section:problem}

In this section, we present a generic stochastic optimal control problem. Subsequently, we discuss how the formulation extends the standard motion planning problem of reaching a goal region while avoiding collision with obstacles.

\paragraph*{Stochastic Dynamics}

Let $d_x$, $d_u$, and $d_w$ be positive integers. The $d_x$-dimensional and $d_u$-dimensional Euclidean spaces are $\reals^{d_x}$ and $\reals^{d_u}$ respectively. Let $S$ be a compact subset  of $\reals^{d_x}$, which is the closure of its interior $S^o$ and has a smooth boundary $\partial S$. The state of the system at time $t$ is $x(t) \in S$, which is fully observable at all times. We also define a compact subset ${U}$ of $\reals^{d_u}$ as a control set.

Suppose that a stochastic process $\{w(t); t \ge 0\}$ is a $d_w$-dimensional Brownian motion, also called a Wiener process, on some probability space $(\Omega,\mathcal{F},\mathcal{P})$. Let a control process $\{u(t) ; t \ge 0\}$ be a $U$-valued, measurable process also defined on the same probability space. We say that the control process $u(\cdot)$ is nonanticipative with respect to the Wiener process $w(\cdot)$ if there exists a filtration $\{\mathcal{F}_t ; t \ge 0\}$ defined on $(\Omega,\mathcal{F},\mathcal{P})$ such that $u(\cdot)$ is $\mathcal{F}_t$-adapted, and $w(\cdot)$ is an $\mathcal{F}_t$-Wiener process. In this case, we say that $u(\cdot)$ is an admissible control inputs with respect to $w(\cdot)$, or the pair $(u(\cdot),w(\cdot))$ is admissible. Let $\reals^{d_x \times d_w}$ denote the set of all $d_x$ by $d_w$ real matrices. We consider stochastic dynamical systems, also called controlled diffusions, of the form
\begin{align}
dx(t) = f(x(t),u(t)) \, dt + F(x(t),u(t)) \, dw(t), \ \ \forall t \ge 0 \label{eqn:system}
\end{align}
where $f: S \times {U} \to \reals^{d_x}$ and $F : S \times U \to \reals^{d_x \times d_w}$ are bounded measurable and continuous functions as long as $x(t) \in S^o$. The matrix $F(\cdot,\cdot)$ is assumed to have full rank. More precisely, a solution to the differential form given in Eq.~\eqref{eqn:system} is a stochastic process  $\{x(t); t \ge 0\}$ such that $x(t)$ equals the following stochastic integral in all sample paths:
\begin{eqnarray}
x(t)=x(0) + \int_0^t f(x(\tau),u(\tau))\, d\tau + \int_0^t F(x(\tau),u(\tau)) dw(\tau), \label{eqn:systemIto} 
\end{eqnarray}
until $x(\cdot)$ exits $S^o$, where the last term on the right hand side is the usual It\^o integral (see, e.g.,~\cite{Oksendal:1992}). When the process $x(\cdot)$ hits $\partial S$, the process $x(\cdot)$ is stopped.

\paragraph*{Weak Existence and Weak Uniqueness of Solutions}
Let $\Gamma$ be the sample path space of admissible pairs $(u(\cdot),w(\cdot))$. Suppose we are given probability measures $\Lambda$ and $P_0$ on $\Gamma$ and on $S$ respectively. We say that solutions of \eqref{eqn:systemIto} exist in the weak sense if there exists a probability space $(\Omega,\mathcal{F},\mathcal{P})$, a filtration $\{\mathcal{F}_t ; t \ge 0\}$, an $\mathcal{F}_t$-Wiener process $w(\cdot)$, an $\mathcal{F}_t$-adapted control process $u(\cdot)$, and an $\mathcal{F}_t$-adapted process $x(\cdot)$ satisfying Eq.~\eqref{eqn:systemIto}, such that  $\Lambda$ and $P_0$ are the distributions of $(u(\cdot),w(\cdot))$ and $x(0)$ under $\mathcal{P}$. We call such tuple $\{ (\Omega,\mathcal{F},\mathcal{P}),\mathcal{F}_t,w(\cdot),u(\cdot),x(\cdot) \}$ a weak sense solution of Eq.~\eqref{eqn:system}~\cite{Karatzas1991, Kushner2000}.

Assume that we are given weak sense solutions
$
\{ (\Omega_i,\mathcal{F}_i,\mathcal{P}_i),\mathcal{F}_{t,i},w_i(\cdot),u_i(\cdot),x_i(\cdot) \}, i=1,2,
$
to Eq.~\eqref{eqn:system}. We say solutions are weakly unique if equality of the joint distributions of $(w_i(\cdot),u_i(\cdot),x_i(0))$ under $\mathcal{P}_i$, $i=1,2$, implies the equality of the distributions $(x_i(\cdot),w_i(\cdot),u_i(\cdot),x_i(0))$ under $\mathcal{P}_i$, $i=1,2$~\cite{Karatzas1991, Kushner2000}.

In this paper, given the boundedness of the set $S$, and the definition of the functions $f$ and $F$ in Eq.~\eqref{eqn:system}, we have a weak solution to Eq.~\eqref{eqn:system} that is unique in the weak sense~\cite{Karatzas1991}. The boundedness requirement is naturally satisfied in many applications and is also needed for the implementation of the proposed numerical method. We will also handle the case in which $f$ and $F$ are discontinuous with extra mild technical assumptions to ensure asymptotic optimality in Section \ref{section:algorithm:cmc}.

\paragraph*{Policy and Cost-to-go Function}
A particular class of admissible controls, called {\em Markov controls}, depends only on the current state, i.e., $u(t)$ is a function only of $x(t)$, for all $t \ge 0$. It is well known that in control problems with full state information, the best Markov control performs as well as the best admissible control (see, e.g.,~\cite{Oksendal:1992,Karatzas1991}).
A Markov control defined on $S$ is also called a {\em policy}, and is represented by the function $\mu : {S} \to {U}$. The set of all policies is denoted by $\Pi$. Define the {\em first exit time} $\FirstExit{\mu}: \Pi \rightarrow [0,+\infty]$ under policy $\mu$ as 
$$
\FirstExit{\mu} = \inf\big\{ t : x(t) \notin {S^o} \mbox{ and Eq.~\eqref{eqn:system} and } u(t)=\mu(x(t)) \big\}.
$$
Intuitively, $\FirstExit{\mu}$ is the first time that the trajectory of the dynamical system given by Eq.~\eqref{eqn:system} with $u(t) = \mu(x(t))$ hits the boundary $\partial S$ of $S$. By definition, $\FirstExit{\mu}=+\infty$ if $x(\cdot)$ never exits $S^o$. Clearly, $\FirstExit{\mu}$ is a random variable.
Then, the expected cost-to-go function under policy $\mu$ is a mapping from $S$ to $\reals$ defined as
$$
\CostToGo{z}{\mu} = {\EE} \left[\int_0^\FirstExit{\mu} \alpha^t \, g\big(x(t), \mu(x(t))\big)\,dt + h(x(\FirstExit{\mu})) \ \vert \ x(0)=z \right],
$$
where $g : S \times U \to \reals$ and $h : S \to \reals$ are bounded measurable and continuous functions, called the {\em cost rate function} and the {\em terminal cost function}, respectively, and $\alpha \in [0,1)$ is the {\em discount rate}. We further assume that $g(x,u)$ is uniformly H\"{o}lder continuous in $x$ with exponent $2\rho \in (0,1]$ for all $u \in U$. That is, there exists some constant $\mathcal{C} >0$ such that
$$
|g(x,u) - g(x',u)| \leq \mathcal{C}||x-x'||^{2\rho}_{2}, \ \ \forall x,x' \in S.
$$
We will address the discontinuity of $g$ and $h$ in Section \ref{section:algorithm:cmc}. 

The {\em optimal cost-to-go function} $J^*: S \to \reals$ is defined as $\OptimalCostToGo{z} = \inf_{\mu \in \Pi} \CostToGo{z}{\mu}$ for all $z \in S$. A policy $\mu^*$ is called optimal if $J_{\mu^*}=J^*$. For any $\epsilon >0$, a policy $\mu$ is called an $\epsilon$-optimal policy if $||J_{\mu}-J^*||_{\infty} \leq \epsilon$.

In this paper, we consider the problem of computing the optimal cost-to-go function $J^*$ and an optimal policy $\mu^*$ if obtainable. Our approach, outlined in Section~\ref{section:algorithm}, approximates the optimal cost-to-go function and an optimal policy in an anytime fashion using incremental sampling-based algorithms. This sequence of approximations is guaranteed to converge uniformly to the optimal cost-to-go function and to find an $\epsilon$-optimal policy for an arbitrarily small non-negative $\epsilon$, almost surely, as the number of samples approaches infinity.

\paragraph*{Relationship with Standard Motion Planning}
The standard motion planning problem of finding a collision-free trajectory that reaches a goal region for a deterministic dynamical system can be defined as follows (see, e.g.,~\cite{Karaman.Frazzoli:IJRR10}). Let ${\cal X} \subset \reals^{d_x}$ be a compact set. Let the open sets ${\cal X}_\mathrm{obs}$ and ${\cal X}_\mathrm{goal}$ denote the obstacle region and the goal region, respectively. Define the obstacle-free space as ${\cal X}_\mathrm{free} := {\cal X} \setminus {\cal X}_\mathrm{obs}$. Let $x_\mathrm{init} \in {\cal X}_\mathrm{free}$. Consider the deterministic dynamical system $\dot{x} = f(x(t), u(t)) \,dt$, where $f : {\cal X} \times U \to \reals^{d_x}$. The {\em feasible motion planning problem} is to find a measurable control input $u : [0,T] \to {U}$ 
such that the resulting trajectory $x(t)$ is collision free
, i.e., $x(t) \in {\cal X}_\mathrm{free}$ and reaches the goal region, i.e., $x(T) \in {\cal X}_\mathrm{goal}$. The {\em optimal motion planning} problem is to find a measurable control input $u$ such that the resulting trajectory $x$ solves the feasible motion planning problem with minium trajectory cost.

The problem considered in this paper extends the classical motion planning problem with stochastic dynamics as described by Eq.~\eqref{eqn:system}. 
Given a goal set ${\cal X}_\mathrm{goal}$ and an obstacle set ${\cal X}_\mathrm{obs}$, define  $S:= {\cal X} \setminus ({\cal X}_\mathrm{goal} \cup {\cal X}_\mathrm{obs})$ and thus $\partial {\cal X}_\mathrm{goal} \cup \partial {\cal X}_\mathrm{obs} \cup \partial {\cal X} = \partial{S}$.
Due to the nature of Brownian motion, under most policies, there is some non-zero probability that collision with an obstacle set will occur. 
However, to penalize collision with obstacles in the control design process, the cost of terminating by hitting the obstacle set, i.e., $h(z)$ for $z \in \partial {\cal X}_\mathrm{obs}$, can be made arbitrarily high. Clearly, the higher this number is, the more conservative the resulting policy will be. Similarly, the terminal cost function on the goal set, i.e., $h(z)$ for $z \in \partial {\cal X}_\mathrm{goal}$, can be set to a small value to encourage terminating by hitting the goal region.

\section{Markov Chain Approximation} \label{section:algorithm:cmc}

A discrete-state Markov decision process (MDP) is a tuple ${\cal M} = (X,A,P,G,H)$ where $X$ is a finite set of states, $A$ is a set of actions that is possibly a continuous space, $P(\cdot \given \cdot,\cdot): X \times X \times A \to \reals_{\ge 0}$ is a function that denotes the transition probabilities satisfying $\sum_{\xi' \in X} P(\xi'\given \xi, v) = 1$ for all $\xi \in X$ and all $v \in A$, $G(\cdot,\cdot) : X \times A \to \reals$ is an immediate cost function, and $H: X \to \reals$ is a terminal cost funtion. If we start at time $0$ with a state $\xi_0 \in X$, and at time $i \ge 0$, we apply an action $v_i \in A$ at a state $\xi_i$ to arrive at a next state $\xi_{i+1}$ according to the transition probability function $P$, we have a controlled Markov chain $\{\xi_i; i \in \naturals \}$.
The chain $\{\xi_i; i \in \naturals\}$ due to the control sequence $\{v_i; i \in \naturals\}$ and an initial state $\xi_0$ will also be called the {\em trajectory} of ${\cal M}$ under the said sequence of controls and initial state.

Given a continuous-time dynamical system as described in Eq.~\eqref{eqn:system}, the Markov chain approximation method approximates the continuous stochastic dynamics using a sequence of MDPs $\{ {\cal M}_n\}_{n=0}^{\infty}$ in which ${\cal M}_n = (S_n,U,P_n,G_n,H_n)$ where  $S_n$ is a discrete subset of $S$, and $U$ is the original control set. We define $\partial S_n= \partial S \cap S_n$. For each $n \in \naturals$, let $\{\xi^n_i; i \in \naturals\}$ be a controlled Markov chain on ${\cal M}_n$ until it hits $\partial S_n$. We associate with each state $z$ in $S$ a non-negative interpolation interval $\Delta t_n(z)$, known as {\em a holding time}. 
We define $t^n_i= \sum_{0}^{i-1}\Delta t_n(\xi^n_i)$ for $i \ge 1$ and $t^n_0=0$. Let $\Delta \xi_i^n=\xi_{i+1}^n -\xi_i^n $. Let $u^n_i$ denote the control used at step $i$ for the controlled Markov chain. In addition, we define $G_n(z,v)=g(z,v)\Delta t_n(z)$ and $H_n(z)=h(z)$ for each $z \in S_n$ and $v \in U$. Let $\Omega_n$ be the sample space of ${\cal M}_n$. Holding times $\Delta t_n$ and transition probabilities $P_n$ are chosen to satisfy {\em the local consistency property} given by the following conditions:
\begin{enumerate}
\item For all $z \in S$, 
\begin{align}
\lim_{n \to \infty} \Delta t_n(z) = 0, \label{eqn:lc1}
\end{align}
\item For all $z \in S$ and all $v \in {U}$:
\end{enumerate}
\begin{eqnarray}
\lim_{n \to \infty} \frac{\EE_{P_n}[\Delta \xi_{i}^n \given \xi_i^n = z, u_i^n = v]}{\Delta t_n (z)} &=& f(z,v), \label{eqn:lc2}\\
\lim_{n \to \infty} \frac{\Cov_{P_n}[\Delta \xi_{i}^n \given \xi_i^n = z, u_i^n = v]}{\Delta t_n (z)}&=& F(z,v)F(z,v)^T, \label{eqn:lc3}\\
\lim_{n \to \infty} \sup_{i \in \naturals, \omega \in \Omega_n} ||\Delta \xi_{i}^n||_2 &=&0. \label{eqn:lc4}
\end{eqnarray}

The chain $\{\xi^n_i; i \in \naturals\}$ is a discrete-time process. In order to approximate the continuous-time process $x(\cdot)$ in Eq.~\eqref{eqn:systemIto}, we use an {\em approximate continuous-time interpolation}. We define the (random) continuous-time interpolation $\xi^n(\cdot)$ of the chain $\{\xi^n_i; i \in \naturals\}$ and the continuous-time interpolation $u^n(\cdot)$ of the control sequence $\{u^n_i; i \in \naturals\}$ under the holding times function $\Delta t_n$ as follows:
$\xi^n(\tau) = \xi^n_i, \text{ and } u^n(\tau) = u^n_i$ for all  $\tau \in [t^n_i, t^n_{i+1})$.
Let $D^{d_x}[0, +\infty)$ denote the set of all $\reals^{d_x}$-valued functions that are continuous from the left and has limits from the right. The process $\xi^n$ can be thought of as a random mapping from $\Omega_n$ to the function space $D^{d_x}[0, +\infty)$.

A control problem for the MDP ${\cal M}_n$ is analogous to that defined in Section~\ref{section:problem}.
Similar to previous section, a policy $\mu_n$ is a function that maps each state $z \in S_n $ to a control $\mu_n(z) \in U$. The set of all such policies is $\Pi_n$. 
Given a policy $\mu_n$, the (discounted) cost-to-go due to $\mu_n$ is:
$$
J_{n,\mu_n} (z) = \EE_{P_n}\left[ \sum_{i = 0}^{I_n-1} \alpha^{t_i^n} G_n(\xi^n_i, \mu_n(\xi^n_i)) + \alpha^{t^n_{I_n}}H_n(\xi^n_{I_n}) \ \Big\vert \ \xi^n_0 = z \right],
$$
where $\EE_{P_n}$ denotes the conditional expectation 
under $P_n$, the sequence $\{\xi^n_i ; i \in \naturals\}$ is the controlled Markov chain under the policy $\mu_n$, and $I_n$ is termination time defined as $I_n = \min\{i : \xi^n_i \in \partial S_n\}$.

The {\em optimal cost function}, denoted by $J_n^*$ satisfies 
\begin{align}
J_n^*(z) = \inf_{\mu_n \in \Pi_n} J_{n,\mu_n}(z), \ \ \forall z \in S_n. \label{eqn:discreteoptimization}
\end{align}
An {\em optimal policy}, denoted by $\mu_n^*$, satisfies 
$
J_{n,\mu_n^*}(z) =  J_n^*(z)
$
for all $z \in S_n$. For any $\epsilon >0$, $\mu_n$ is an $\epsilon$-optimal policy if $||J_{n,\mu_n}-J^*_n||_{\infty} \leq \epsilon$.

As stated in the following theorem, under mild technical assumptions, local consistency implies the convergence of continuous-time interpolations of the trajectories of the controlled Markov chain to the trajectories of the stochastic dynamical system described by Eq.~\eqref{eqn:system}.

\begin{theorem}[see Theorem 10.4.1 in~\cite{Kushner2000}] \label{theorem:convergence_traj}
Let us assume that $f(\cdot, \cdot)$ and $F(\cdot,\cdot)$ are measurable, bounded and continuous. Thus, Eq.~\eqref{eqn:system} has a weakly unique solution. Let $\{{\cal M}_n\}_{n=0}^{\infty}$ be a sequence of MDPs, and $\{\Delta t_n\}_{n=0}^{\infty}$ be a sequence of holding times that are locally consistent with the stochastic dynamical system described by Eq.~\eqref{eqn:system}. Let $\{u_i^n; i \in \naturals\}$ be a sequence of controls defined for each $n \in \naturals$. For all $n \in \naturals$, let $\{\xi^n(t); t \in \reals_{\ge 0}\}$ denote the continuous-time interpolation to the chain $\{\xi^n_i; i \in \naturals \}$ under the control sequence $\{u_i^n; i \in \naturals\}$ starting from an initial state $z_\mathrm{init}$, and $\{u^n (t); t \in \reals_{\ge 0}\}$ denote the continuous-time interpolation of $\{u_i^n; i \in \naturals\}$, according to the holding time $\Delta t_n$. 
Then, any subsequence of $\{ (\xi^n(\cdot), u^n(\cdot)) \}_{n=0}^{\infty}$ has a further subsequence that converges in distribution to $(x(\cdot), u(\cdot))$
satisfying
$$
x(t) = z_\mathrm{init} + \int_0^t f(x(\tau),u(\tau)) d\tau + \int_0^t F(x(\tau),u(\tau)) dw(\tau).
$$
Under the weak  uniqueness condition for solutions of Eq.~\eqref{eqn:system}, the sequence $\{ (\xi^n(\cdot), u^n(\cdot)) \}_{n=0}^{\infty}$ also converges to $(x(\cdot), u(\cdot))$.
\end{theorem}
Furthermore, a sequence of minimizing controls guarantees pointwise convergence of the cost function to the original optimal cost function in the following sense. 

\begin{theorem}[see Theorem 10.5.2 in~\cite{Kushner2000}] \label{theorem:convergence_cost}
Assume that $f(\cdot, \cdot)$, $F(\cdot,\cdot)$, $g(\cdot, \cdot)$ and $h(\cdot)$ are measurable, bounded and continuous. 
For any trajectory $x(\cdot)$ of the system described by Eq.~\eqref{eqn:system}, define $\hat\tau(x) := \inf\{t : x(t) \notin {S^o}\}$. Let $\{{\cal M}_n = (S_n, U, P_n,G_n,H_n)\}_{n=0}^{\infty}$ and $\{\Delta t_n\}_{n=0}^{\infty}$ be locally consistent with the system described by Eq.~\eqref{eqn:system}. 

We suppose that the function $\hat{\tau}(\cdot)$ is continuous (as a mapping from $D^{d_x}[0,+\infty)$ to the compactified interval $[0,+\infty]$) with probability one relative to the measure induced by any solution to Eq.~\eqref{eqn:system} for an initial state $z$, which is satisfied when the matrix $F(\cdot,\cdot)F(\cdot,\cdot)^T$ is nondegenerate.
Then, for any $z \in S_n$, the following equation holds:
$$
\lim_{n \rightarrow \infty} |J_n^*(z) -J^*(z)|=0.
$$
In particular, for any $z \in S_n$, for any sequence $\{\epsilon_n >0 \}_{n=0}^{\infty}$ such that $\lim_{n \rightarrow \infty} \epsilon_n=0$, and for any sequence of policies $\{\mu_n\}_{n=0}^{\infty}$ such that $\mu_n$ is an $\epsilon_n$-optimal policy of ${\cal M}_n$, we have:
$$
\lim_{n \to \infty} |J_{n,\mu_n}(z) - J^*(z)| = 0.
$$
Moreover, the sequence $\{t^n_{I_n}; n \in \naturals\}$ converges in distribution to the termination time of the optimal control problem for the system in Eq.~\eqref{eqn:system} when the system is under optimal control processes. 
\end{theorem}
Under the assumption that the cost rate $g$ is H\"{o}lder continuous \cite{EvansPDE} with exponent $2\rho$, the sequence of optimal value functions for approximating chains $J^*_n$ indeed converges uniformly to $J^*$ with a proven rate. Let us denote $||b||_{S_n}=\sup_{z \in S_n}b(x)$ as the sup-norm over $S_n$ of a function $b$ with domain containing $S_n$. Let 
\begin{align}
\dispersion_n=\max_{z \in S_n}\min_{z' \in S_n}||z'-z||_2 \label{eqn:dispersion}
\end{align}
be the dispersion of $S_n$.
\begin{theorem}[see Theorem 2.3 in~\cite{Men89} and Theorem 2.1 in~\cite{Dupuis:1998}] \label{theorem:uniform_convergence}
Consider an MDP sequence $\{{\cal M}_n = (S_n, U, P_n,G_n,H_n)\}_{n=0}^{\infty}$ and holding times $\{\Delta t_n\}_{n=0}^{\infty}$ that are locally consistent with the system described by Eq.~\eqref{eqn:system}. Let $J^*_n$ be the optimal cost of ${\cal M}_n$. Given the assumptions on the dynamics and cost rate functions in Section \ref{section:problem}, as $n$ approaches $\infty$, we have
$$
||J_n^* - J^*||_{S_n} = O(\dispersion_n^{\rho}).
$$

\end{theorem}

\subsection*{Discontinuity of dynamics and objective functions}
We note that the above theorems continue to hold even when the functions $f,F,g,$ and $h$ are discontinuous. In this case, the following conditions are sufficient to use the theorems: (i) For $r$ to be $f,F,g$, or $h$, $r(x,u)$ takes either the form $r_0(x)+r_1(u)$ or $r_0(x)r_1(u)$ where the control dependent terms are continuous and the $x$-dependent terms are measurable, and (ii) $f(x,\cdot),F(x,\cdot),g(x,\cdot)$, and $h(x)$ are nondegenerate for each $x$, and the set of discontinuity in $x$ of each function is a uniformly smooth surface of lower dimension. Furthermore, instead of uniform H\"{o}lder continuity, the cost rate $g$ can be relaxed to be locally H\"{o}lder continuous with exponent $2\rho$ on $S$ (see, e.g., page 275 in~\cite{Kushner2000}). 

Let us remark that the controlled Markov chain differs from the stochastic dynamical systems described in Section~\ref{section:problem} in that the former possesses a discrete state structure and evolves in a discrete time manner while the latter is a continuous model both in terms of its state space and the evolution of time. Yet, both models possess a continuous control space. 
It will be clear in the following discussion that the control space does not have to be discretized if a certain optimization problem can be solved numerically or via sampling. 

The above theorems assert the asymptotic optimality given a sequence of \textit{a priori} discretizations of the state space and the availability of $\epsilon$-optimal policies. In what follows, we describe an algorithm that incrementally computes the optimal cost-to-go function and an optimal control policy of the continuous problem.

\section{The iMDP Algorithm}  \label{section:algorithm}
Based on Markov chain approximation results, the iMDP algorithm incrementally builds a sequence of discrete MDPs with probability transitions and cost-to-go functions that consistently approximate the original continuous counterparts. The  algorithm refines the discrete models by using a number of primitive procedures to add new states into the current approximate model. Finally, the  algorithm improves the quality of discrete-model policies in an iterative manner by effectively using the computations inherited from the previous iterations. Before presenting the algorithm, some primitive procedures which the algorithm relies on are presented in this section.

\subsection{Primitive Procedures} \label{section:algorithm:primitiveprocedures}

\setcounter{paragraph}{0}
\subsubsection{Sampling}
The ${\tt Sample()}$ and ${\tt SampleBoundary()}$ procedures sample states independently and uniformly from the interior ${S^o}$ and the boundary $\partial {S}$, respectively.

\subsubsection{Nearest Neighbors}
Given $z \in {S}$ and a set $\DummyS \subseteq {S}$ of states. For any $k \in \naturals$, the procedure ${\tt Nearest}(z,\DummyS,k)$ returns the $k$ nearest states $z' \in \DummyS$ that are closest to $z$ in terms of the Euclidean norm. 

\subsubsection{Time Intervals}
Given a state $z \in {S}$ and a number $k \in \naturals$, the procedure ${\tt ComputeHoldingTime}(z,k)$ returns a holding time computed as follows:
$$
{\tt ComputeHoldingTime}(z,k) = \gamma_t \left(\frac{\log k}{k}\right)^{\dummyasyn \dummyrate \rho/d_x},
$$
where $\gamma_t >0 $ is a constant, and $\dummyrate, \dummyasyn$ are constants in $(0,1)$ and $(0,1]$ respectively\footnote{Typical values of $\dummyrate$ is [0.999,1).}. The parameter $\rho \in (0,0.5]$ defines the H\"{o}lder continuity of the cost rate function $g(\cdot,\cdot)$ as in Section~\ref{section:problem}.

\subsubsection{Transition Probabilities} \label{para:tranprob}
Given a state $z \in S$, a subset $Y \in S$, a control $v \in U$, and a positive number $\tau$ describing a holding time, the procedure ${\tt ComputeTranProb}(z,v,\tau,Y)$ returns (i) a finite set $Z_\mathrm{near} \subset S$ of states
such that the state $z + f(z,v)\tau$ belongs to the convex hull of $Z_\mathrm{near}$ and $||z'-z||_2= O( \tau )$ for all $z' \neq z \in Z_\mathrm{near}$, 
and (ii) a function $p$ that maps $Z_\mathrm{near}$ to a non-negative real numbers such that $p(\cdot)$ is a probability distribution over the support $Z_\mathrm{near}$. It is crucial to ensure that these transition probabilities result in a sequence of locally consistent chains in the algorithm.  

There are several ways to construct such transition probabilities. One possible construction by solving a system of linear equations can be found in \cite{Kushner2000}. 
In particular, we choose $Z_\mathrm{near}= {\tt Nearest}(z + f(z,v)\tau,Y,s)$ where $s \in \naturals$ is some constant. We define the transition probabilities $p : {Z_\mathrm{near}} \to \reals_{\ge 0}$ that satisfies:
\begin{itemize}
\item[(i)] $\sum_{z' \in Z_\mathrm{near}} p(z') (z' - z) = f(z,v) \tau + o(\tau)$, 
\item[(ii)] $\sum_{z' \in Z_\mathrm{near}} p(z') (z' - z)(z' - z)^T = F(z,v) F(z,v)^T \, \tau + f(z,v) f(z,v)^T {\tau}^2+ o(\tau)$.
\item[(iii)] $\sum_{z' \in Z_\mathrm{near}} p(z') = 1$.
\end{itemize}

An alternate way to compute the transition probabilities is to approximate using local Gaussian distributions. We choose $Z_\mathrm{near}= {\tt Nearest}(z + f(z,v)\tau,Y,s)$ where $s=\Theta(\log(|Y|))$. Let ${\cal N}_{\overline{m}, \sigma} (\cdot)$ denote the density of the (possibly multivariate) Gaussian distribution with mean $\overline{m}$ and variance $\sigma$. Define the transition probabilities as follows:
$$
p(z') = \frac{{\cal N}_{\overline{m},\sigma} (z')}{\sum_{y \in Z_\mathrm{near}} {\cal N}_{\overline{m},\sigma} (y)},
$$
where $\overline{m} = z+f(z,v) \tau$ and $\sigma = F(z,v) F(z,v)^T \tau$. This expression can be evaluated easily for any fixed $v \in {U}$. As $|Z_{near}|$ approaches infinity, the above construction satisfies the local consistency almost surely. 

As we will discuss in Section~\ref{section:algorithm:algorithm}, the size of the support $Z_\mathrm{near}$ affects the complexity of the iMDP algorithm. We note that solving a system of linear equations requires computing and handling a matrix of size $(d_x^2+d_x+1) \times |Z_\mathrm{near}|$ where $|Z_\mathrm{near}|$ is constant. When $d_x$ and $|Z_\mathrm{near}|$ are large, the constant factor of the complexity is large. In contrast, computing local Gaussian approximation requires only $|Z_\mathrm{near}|$ evaluations. Thus, although local Gaussian approximation yields higher time complexity, this approximation is more convenient to compute.

\subsubsection{Backward Extension}
Given $T >0$ and two states $z,z' \in {S}$, the procedure ${\tt ExtendBackwards}(z,z',T)$ returns a triple $(x,v,\tau)$ such that (i) $\dot{x}(t) = f(x(t), u(t)) dt$ and $u(t)=v \in U$ for all $t \in [0,\tau]$, (ii) $\tau \leq T$, (iii) $x(t) \in {S}$ for all $t \in [0,\tau]$, (iv) $x(\tau) = z$, and (v) $x(0)$ is close to $z'$. 
If no such trajectory exists, then the procedure returns failure\footnote{This procedure is used in the algorithm solely for the purpose of inheriting the ``rapid exploration'' property of the RRT algorithm~\cite{Lavalle.98,Karaman.Frazzoli:IJRR10}.}. 
We can solve for the triple $(x,v,\tau)$ by sampling several controls $v$ and choose the control resulting in $x(0)$ that is closest to $z'$.

\subsubsection{Sampling and Discovering Controls} 
The procedure ${\tt ConstructControls}(k,z,Y,T)$ returns a set of $k$ controls in $U$. We can uniformly sample $k$ controls in $U$. Alternatively, for each state $z' \in$ ${\tt Nearest}(z,Y,k)$, we solve for a control $v \in U$ such that (i) $\dot{x}(t) = f(x(t), u(t)) dt$ and $u(t)=v \in U$ for all $t \in [0,T]$, (ii) $x(t) \in {S}$ for all $t \in [0,T]$, (iii) $x(0)=z$ and $x(T) = z'$. 

\subsection{Algorithm Description} \label{section:algorithm:algorithm}
The iMDP algorithm is given in Algorithm~\ref{algorithm:main}. The algorithm incrementally refines a sequence of (finite-state) MDPs ${\cal M}_n=(S_n,U,P_n,G_n,H_n)$ and the associated holding time function $\Delta t_n$ that consistently approximates the sytem in Eq.~\eqref{eqn:system}. In particular, given a state $z\in S_n$ and a holding time $\Delta t_n(z)$, we can implicitly define the stage cost function $G_n(z,v)= \Delta t_n(z)g(z,v)$ for all $v \in U$ and terminal cost function $H_n(z)=h(z)$. We also associate with $z \in S_n$ a cost value $J_n(z)$, and a control $\mu_n(z)$. We refer to $J_n$ as a cost value function over $S_n$.
In the following discussion, we describe how to construct $S_n, P_n, J_n, \mu_n$ over iterations. We note that, in most cases, we only need to construct and access $P_n$ on demand.
 
In every iteration of the main loop (Lines~\ref{line:mainloop:start}-\ref{line:mainloop:end}), we sample an additional state from the boundary of the state space $S$. We set $J_n,\mu_n,\Delta t_n$ for those states at Line~\ref{line:main:add_boundary_state_add_values}. Subsequently, we also sample a state from the interior of $S$ (Line~\ref{line:main:sample_interior}) denoted as $z_\mathrm{s}$. We compute the nearest state $z_\mathrm{nearest}$, which is already in the current MDP, to the sampled state (Line~\ref{line:main:compute_nearest}). The algorithm computes a trajectory that reaches $z_\mathrm{nearest}$ starting at some state near $z_\mathrm{s}$ (Line~\ref{line:main:steer_back}) using a control signal $u_\mathrm{new}(0..\tau)$. The new trajectory is denoted by $x_\mathrm{new} : [0,\tau] \to S$ and the starting state of the trajectory, i.e., $x_\mathrm{new}(0)$, is denoted by $z_\mathrm{new}$. The new state $z_\mathrm{new}$ is added to the state set, and the cost value $J_n(z_\mathrm{new})$, control $\mu_n(z_\mathrm{new})$, and holding time $\Delta t_n(z_\mathrm{new})$ are initialized at Line~\ref{line:main:add_interior_state_value}. 

\begin{algorithm}[t]
$(n,S_0,J_0,\mu_0,\Delta t_0) \leftarrow (1,\emptyset,\emptyset,\emptyset,\emptyset$)\; 
\While{$n < N$}
{
$\left(S_{n}, J_{n}, \mu_{n}, \Delta t_{n} \right) \leftarrow \left(S_{n-1}, J_{n-1}, \mu_{n-1}, \Delta t_{n-1} \right)$\; \label{line:copiedton}
\vspace{0.08in}
\tcp{Add a new state to the boundary}
$z_\mathrm{s} \leftarrow {\tt SampleBoundary}()$\; \label{line:mainloop:start}
\label{line:main:add_boundary_state_start} \label{line:main:sample_boundary}
$ \left(S_{n}, J_{n}(z_\mathrm{s}), \mu_{n}(z_\mathrm{s}),\Delta t_{n} (z_\mathrm{s}) \right)  \leftarrow \left(S_{n} \cup\{z_\mathrm{s}\}, h(z_\mathrm{s}), null,0 \right)$ \;  \label{line:main:add_boundary}
\label{line:main:add_boundary_state_add_values}
	\label{line:main:add_boundary_state_end} \label{line:main:compute_terminal_cost}
	\vspace{0.08in}
	\tcp{Add a new state to the interior}
	$z_\mathrm{s} \leftarrow {\tt Sample}()$\;
	\label{line:main:add_interior_state_start} \label{line:main:sample_interior}
	$z_\mathrm{nearest} \leftarrow {\tt Nearest}(z_\mathrm{s}, S_{n},1)$\;
	\label{line:main:compute_nearest}
	\If{$(x_\mathrm{new}, u_\mathrm{new}, \tau) \leftarrow {\tt ExtendBackwards}(z_\mathrm{nearest}, z_\mathrm{s},T_0)$}{ \label{line:main:steer_back}
		$z_\mathrm{new} \leftarrow x_{new}(0)$\;
		$cost= \tau g(z_\mathrm{new},u_\mathrm{new}) + \alpha^{\tau}J_{n}(z_\mathrm{nearest})$\;
		$ (S_{n}, J_{n}(z_\mathrm{new}), \mu_{n}(z_\mathrm{new}),\Delta t_{n} (z_\mathrm{new}))  \leftarrow (S_{n} \cup \{ z_\mathrm{new}\}, cost, u_{new}, \tau)$ \; \label{line:main:add_interior_state}
		\label{line:main:add_interior_state_value}
		\vspace{0.08in}
		\tcp{Perform $L_n \ge 1$ (asynchronous) value iterations}
		\For{$i= 1 \to L_n$}
		{
			\vspace{0.08in}
			\tcp{Update $z_\mathrm{new}$ and $K_n=\Theta \big(|S_n|^{\dummyasyn} \big )$ states $\big (0 < \dummyasyn \leq 1, \ K_n < |S_n|\big )$}
			$Z_\mathrm{update} \leftarrow {\tt Nearest}(z_\mathrm{new}, S_{n}\backslash \partial S_n, K_n) \cup \{z_\mathrm{new}\}$\;
			\label{line:main:compute_update}
			\For{$z \in Z_\mathrm{update}$}
			{
				${\tt Update} (z,S_{n},J_{n},\mu_{n},\Delta t_{n})$\;
				\label{line:main:add_interior_state_end} \label{line:main:update}
			}
		}
	}
		$n \leftarrow n + 1$\;
		\label{line:mainloop:end}
}
\caption{iMDP()} \label{algorithm:main}
\end{algorithm}

\subsection*{Update of cost value and control}
The algorithm updates the cost values and controls of the finer MDP in Lines~\ref{line:main:compute_update}-\ref{line:main:update}. We perform $L_n \ge 1$ value iterations in which we update the new state $z_\mathrm{new}$ and other  $K_n=\Theta \big(|S_n|^{\dummyasyn} \big)$ states in the state set where $K_n < |S_n|$. When all states in the MDP are updated, i.e. $K_n+1=|S_n|$, $L_n$ value iterations are implemented in a synchronous manner. Otherwise, $L_n$ value iterations are implemented in an asynchronous manner. 

The set of states to be updated is denoted as $Z_\mathrm{update}$ (Line~\ref{line:main:compute_update}). To update a state $z \in Z_\mathrm{update}$ that is not on the boundary, in the call to the procedure ${\tt Update}$ (Line~\ref{line:main:update}), we solve the following Bellman equation:\footnote{Although the argument of ${\tt Update}$ at Line~\ref{line:main:update} is $J_n$, we actually process the previous cost values $J_{n-1}$ due to Line~\ref{line:copiedton}. We can implement Line~\ref{line:copiedton} by simply sharing memory for $(S_n,J_n,\mu_n,\Delta t_n)$ and $(S_{n-1},J_{n-1},\mu_{n-1},\Delta t_{n-1})$.}
\begin{eqnarray}
J_n(z)= \min_{v \in U}\{ G_n(z,v) + \alpha^{\Delta t_n(z)} \EE_{P_n}\left[ J_{n-1}(y) |z,v\right]\},\label{eqn:Bellman}
\end{eqnarray}
and set $\mu_n(z) = v^*(z)$, where $v^*(z)$ is the minimizing control of the above optimization problem. There are several ways to solve Eq.~\eqref{eqn:Bellman} over the the continuous control space $U$ efficiently. 
If $P_n(\cdot \given z,v)$ and $g(z,v)$ are affine functions of $v$, and ${U}$ is convex, the above optimization has a linear objective function and a convex set of constraints. Such problems are widely studied in the literature~\cite{Boyld.ConvexOpt}. More generally, we can uniformly sample the set of controls, called $U_n$, in the control space $U$. Hence, we can evaluate the right hand side (RHS) of Eq.~\eqref{eqn:Bellman} for each $v \in U_n$ to find the best $v^*$ in $U_n$ with the smallest RHS value and thus to update $J_n(z)$ and $\mu_n(z)$. When $\lim_{n \to \infty}|U_n|=\infty$, we can solve Eq.~\eqref{eqn:Bellman} arbitrarily well (see Theorem \ref{theoremAlmostSureConvergenceSampleControl}).

Thus, it is sufficient to construct the set $U_n$ with $\Theta(\log(|S_n|))$ controls using the procedure ${\tt ConstructControls}$ as described in Algorithm~\ref{algorithm:update} (Line \ref{line:Controls}). The set $Z_\mathrm{near}$ and  the transition probability $P_n(\cdot \given z,v)$ constructed consistently over the set $Z_\mathrm{near}$ are returned from the procedure ${\tt ComputeTranProb}$ for each $v \in U_n$ (Line \ref{line:pn}). Depending on a particular method to build $P_n$ (i.e. solving a system of linear equations or evaluating a local Gaussian distribution), the cardinality of $Z_\mathrm{near}$ is set to a constant or increases as $\Theta(\log(|S_n|))$. Subsequently, the procedure chooses the best control among the constructed controls to update $J_n(z)$ and $\mu_n(z)$ (Line~\ref{line:improve:end}). We note that in Algorithm~\ref{algorithm:update}, before making improvement for the cost value at $z$ by comparing new controls, we can re-evaluate the cost value with the current control $\mu_n(z)$ over the holding time $\Delta t_n(z)$ by adding the current control $\mu_n(z)$ to $U_n$. The reason is that the current control may be still the best control compared to other controls in $U_n$.

\begin{algorithm}[t]
$\tau \leftarrow {\tt ComputeHoldingTime}(z,|S_{n}|)$\; \label{line:holdingtime}
\vspace{0.08in}
\tcp{Sample or discover $C_n=\Theta(\log(|S_n|))$ controls}
$U_n \leftarrow {\tt ConstructControls}(C_n,z,S_n,\tau)$\;
\label{line:Controls}
\label{line:improve:begin}
\For{$v \in U_n$}
{
        \label{line:discover}
	$(Z_\mathrm{near},p_n) \leftarrow {\tt ComputeTranProb}(z,v,\tau,S_n)$\; \label{line:pn}
	$J \leftarrow \tau g(z,v) + \alpha^{\tau}\sum_{y \in Z_{\mathrm{near}}}p_n(y)J_n(y)$\;
	\If{$J < J_n(z)$}
	{ 	
	  $(J_n(z), \mu_n(z),\Delta t_n(z),\kappa_n(z)) \leftarrow (J,v,\tau,|S_n|)$\;
	  \label{line:improve:end}
 	}
}
\caption{${\tt Update}(z \in S_n, S_n, J_n, \mu_n, \Delta t_n)$}
\label{algorithm:update}
\end{algorithm}

\begin{algorithm}[t]
$z_\mathrm{nearest} \leftarrow {\tt Nearest}(z, S_n,1)$\;
\Return{$\mu(z)=( \mu_n(z_\mathrm{nearest}), \Delta t_n(z_\mathrm{nearest}) )$}
\caption{${\tt Policy}(z \in S,n)$}
\label{algorithm:control}
\end{algorithm}

\subsection*{Complexity of iMDP}
The time complexity per iteration of the implementation in Algorithms \ref{algorithm:main}-\ref{algorithm:update} is either $O\big ( |S_n|^{\dummyasyn}\log{|S_n|} \big )$ or $O \big ( |S_n|^{\dummyasyn}(\log{|S_n|})^2 \big)$. In particular, if the procedure {\tt ComputeTranProb} solves a set of linear equations to construct $P_n$ such that the cardinality of $Z_\mathrm{near}$ can remain constant, the time complexity per iteration is $O\big ( |S_n|^{\dummyasyn}\log{|S_n|} \big )$ where $\log{|S_n|}$ accounts for the number of processed controls, and $|S_n|^{\dummyasyn}$ accounts for the number of updated states in one iteration. Otherwise, if the procedure {\tt ComputeTranProb} uses a local Gaussian distribution to construct $P_n$ such that the cardinality of $Z_\mathrm{near}$ increases as $\Theta(\log{|S_n|})$, the time complexity per iteration is $O \big ( |S_n|^{\dummyasyn}(\log{|S_n|})^2 \big)$. The processing time from the beginning until the iMDP algorithm stops after $n$ iterations is thus either $O\big ( |S_n|^{1+\dummyasyn}\log{|S_n|} \big )$ or $O \big ( |S_n|^{1+\dummyasyn}(\log{|S_n|})^2 \big)$. Since we only need to access locally consistent transition probability on demand, the space complexity of the iMDP algorithm is $O(|S_n|)$. Finally, the size of state space $S_n$ is $|S_n|=\Theta(n)$ due to our sampling strategy.

\subsection{Feedback Control} \label{section:algorithm:control}
As we will see in Theorems~\ref{theoremAlmostSureConvergence}-\ref{theoremAlmostSureConvergenceSampleControl}, the sequence of cost value functions $J_n$ arbitrarily approximates the original optimal cost-to-go $J^*$. Therefore, we can perform a Bellman update based on the approximated cost-to-go $J_n$ (using the stochastic continuous-time dynamics) to obtain a policy control for any $n$. However, we will discuss in Theorem~\ref{theoremControlQuality} that the sequence of $\mu_n$ also approximates arbitrarily well an optimal control policy. In other words, in the iMDP algorithm, we also incrementally construct an optimal control policy. In the following paragraph, we present an algorithm that converts a policy for a discrete system to a policy for the original continuous problem.

Given a level of approximation $n \in \naturals$, the control policy $\mu_n$ generated by the iMDP algorithm is used for controlling the original system described by Eq.~\eqref{eqn:system} using the procedure given in Algorithm~\ref{algorithm:control}. This procedure computes the state in ${\cal M}_n$ that is closest to the current state of the original system and applies the control attached to this closest state over the associated holding time. 

\section{Analysis}  \label{section:analysis}
In this section, let $({\cal M}_n = (S_n, U,P_n,G_n,H_n),\Delta t_n, J_n,\mu_n)$ denote the MDP, holding times, cost value function, and policy returned by Algorithm~\ref{algorithm:main} at the end $n$ iterations. The proofs of lemmas and theorems in this section can be found in Appendix. 

For large $n$, states in $S_n$ are sampled uniformly in the state space $S$~\cite{Karaman.Frazzoli:IJRR10}. Moreover, the dispersion of $S_n$ shrinks with the rate  $O((\log{|S_n|}/|S_n|)^{1/d_x})$ as described in the next lemma.
\begin{lemma} \label{lemma:dispersion}
Recall that $\dispersion_n$ measures of the dispersion of $S_n$ (Eq. \ref{eqn:dispersion}). We have the following event happens with probability one:
$$
\dispersion_n=O((\log{|S_n|}/|S_n|)^{1/d_x}).
$$
\end{lemma}
The proof is based on the fact that, if we partition $\reals^{d_x}$ into  cells of volume $O\left({\log(|S_n|)}/{|S_n|}\right)$, then, almost surely, every cell contains at least an element of $S_n$, as $|S_n|$ approaches infinity. The above lemma leads to the following results.
\begin{lemma} \label{lemma:local_consistency}
The MDP sequence $\{ {\cal M}_n \}_{n=0}^{\infty}$ and holding times $\{ \Delta t_n \}_{n=0}^{\infty}$ returned by Algorithm~\ref{algorithm:main} are locally consistent with the system described by Eq.~\eqref{eqn:system} for large $n$ with probability one.
\end{lemma}
Theorem~\ref{theorem:convergence_traj} and Lemma~\ref{lemma:local_consistency} together imply that the trajectories of the controlled Markov chains approximate those of the original stochastic dynamical system in Eq.~\eqref{eqn:system} arbitrarily well as $n$ approaches to infinity. Moreover, recall that  $||\cdot||_{S_n}$ is the sup-norm over $S_n$, the following theorem shows that $J_n^*$ converges uniformly, with probability one, to the original optimal value function $J^*$.
\begin{theorem} \label{theorem:JoptnJoptas}
Given $n \in \naturals$, for all $z \in S_n$, $J_n^*(z)$ denotes the optimal value function evaluated at state $z$ for the finite-state MDP ${\cal M}_n$ returned by Algorithm~\ref{algorithm:main}. Then, the following event holds with probability one:
$$
\lim_{n \rightarrow \infty} ||J_n^* - J^*||_{S_n} = 0.
$$
In other words, $J^*_n$ converges to $J^*$ uniformly. In particular, 
$$
||J_n^* - J^*||_{S_n}=O\Large((\log{|S_n|}/|S_n|)^{\rho/d_x}\Large).
$$
\end{theorem}
The proof follows immediately from Lemmas~\ref{lemma:dispersion}-\ref{lemma:local_consistency} and  Theorems~\ref{theorem:convergence_cost}-\ref{theorem:uniform_convergence}. The theorem suggests that we can compute $J_n^*$ for each discrete MDP ${\cal M}_n$ before sampling more states to construct ${\cal M}_{n+1}$. Indeed, in Algorithm \ref{algorithm:main}, when updated states are chosen randomly as subsets of $S_n$, and $L_n$ is large enough, we compute $J^*_n$ using asynchronous value iterations~\cite{Bertsekas:2001:DPO:516163,Bertsekas:1989:PDC:59912}. Subsequent theorems present stronger results.

We will prove the asymptotic optimality of the cost value $J_n$ returned by the iMDP algorithm when $n$ approaches infinity without directly approximating $J^*_n$ for each $n$. We first consider the case when we can solve the Bellman update (Eq.~\ref{eqn:Bellman}) exactly and $1 \leq L_n, \ K_n = \Theta(|S_n|^{\dummyasyn}) < |S_n|$.
\begin{theorem}
\label{theoremAlmostSureConvergence}
For all $z \in S_n$, $J_n(z)$ is the cost value of the state $z$ computed by Algorithm \ref{algorithm:main} and Algorithm \ref{algorithm:update} after $n$ iterations with $1 \leq L_n$, and $ K_n = \Theta(|S_n|^{\dummyasyn}) < |S_n| $. Let $J_{n,\mu_n}$ be the cost-to-go function of the returned policy $\mu_n$ on the discrete MDP ${\cal M}_n$. If the Bellman update at Eq.~\ref{eqn:Bellman} is solved exactly, then, the following events hold with probability one:
\begin{itemize}
\item[i.] $\lim_{n \rightarrow \infty}||J_n - J^*_n||_{S_n} = 0$, and $\lim_{n \rightarrow \infty}||J_n - J^*||_{S_n} = 0$,
\item[ii.] $\lim_{n \rightarrow \infty}|J_{n,\mu_n}(z) - J^*(z)| =0, \ \ \forall z \in S_n$.
\end{itemize}
\end{theorem}
Theorem \ref{theoremAlmostSureConvergence} enables an incremental computation of the optimal cost $J^*$ without the need to compute $J^*_n$ exactly before sampling more samples. Moreover, cost-to-go functions $J_{n,\mu_n}$ induced by approximating policies $\mu_n$ also converges pointwise to the optimal cost-to-go $J^*$ with probability one.

When we solve the Bellman update at Eq.~\ref{eqn:Bellman} via sampling, the following result holds.
\begin{theorem}
\label{theoremAlmostSureConvergenceSampleControl}
For all $z \in S_n$, $J_n(z)$ is the cost value of the state $z$ computed by Algorithm \ref{algorithm:main} and Algorithm \ref{algorithm:update} after $n$ iterations with $1 \leq L_n$, and $K_n = \Theta(|S_n|^{\theta})< |S_n|$. Let $J_{n,\mu_n}$ be the cost-to-go function of the returned policy $\mu_n$ on the discrete MDP ${\cal M}_n$. If the Bellman update at Eq.~\ref{eqn:Bellman} is solved via sampling such that $\lim_{n \rightarrow \infty} |U_n| = \infty$, then 
\begin{itemize}
\item[i.] $||J_n-J^*_n||_{S_n}$ converges to $0$ in probability. Thus, $J_n$ converges uniformly to $J^*$ in probability,
\item[ii.] $\lim_{n \rightarrow \infty}|J_{n,\mu_n}(z) - J^*(z)| =0$  for all $z \in S_n$ with probability one.
\end{itemize}
\end{theorem}
We emphasize that while the convergence of $J_n$ to $J^*$ is weaker than the convergence in Theorem \ref{theoremAlmostSureConvergence}, the convergence of $J_{n,\mu_n}$ to $J^*$ remains intact. Importantly, Theorem~\ref{theorem:convergence_traj} and Theorems~\ref{theoremAlmostSureConvergence}-\ref{theoremAlmostSureConvergenceSampleControl} together assert that starting from any initial state, trajectories and control processes provided by the iMDP algorithm approximate arbitrarily well optimal trajectories and optimal control processes of the original continuous problem. More precisely, with probability one, the induced random probability measures of approximating trajectories and approximating control processes converge weakly to the probability measures of optimal trajectories and optimal control processes of the continuous problem.

Finally, the next theorem evaluates the quality of any-time control policies returned by Algorithm~\ref{algorithm:control}.
\begin{theorem} \label{theoremControlQuality}
Let $\overline{\mu}_n:S \rightarrow U$ be the interpolated policy on $S$ of $\mu_n:S_n \rightarrow U$ as described in Algorithm~\ref{algorithm:control}:
$$
\forall z \in S: \ \ \overline{\mu}_n(z)= \mu_n(y_n) \text{ where } y_n=\argmin_{z' \in S_n} ||z'-z||_2.
$$
Then there exists an optimal control policy $\mu^*$ of the original problem\footnote{Otherwise, an optimal relaxed control policy $m^*$ exists~\cite{Kushner2000}, and $\overline{\mu}_n$ approximates $m^*$ arbitrarily well.} so that for all $z \in S$:
$$
\lim_{n \rightarrow \infty} \overline{\mu}_n(z) = \mu^*(z) \text{ w.p.1},
$$
if $\mu^*$ is continuous at $z$.
\end{theorem}

\section{Experiments}  \label{section:experiments}

\begin{figure*}[h]
\begin{center}
  \subfigure[Optimal and approximated cost. ]{
  \label{figsubLQR1}
  \includegraphics[width=51mm,bb= 108 240 485 530]{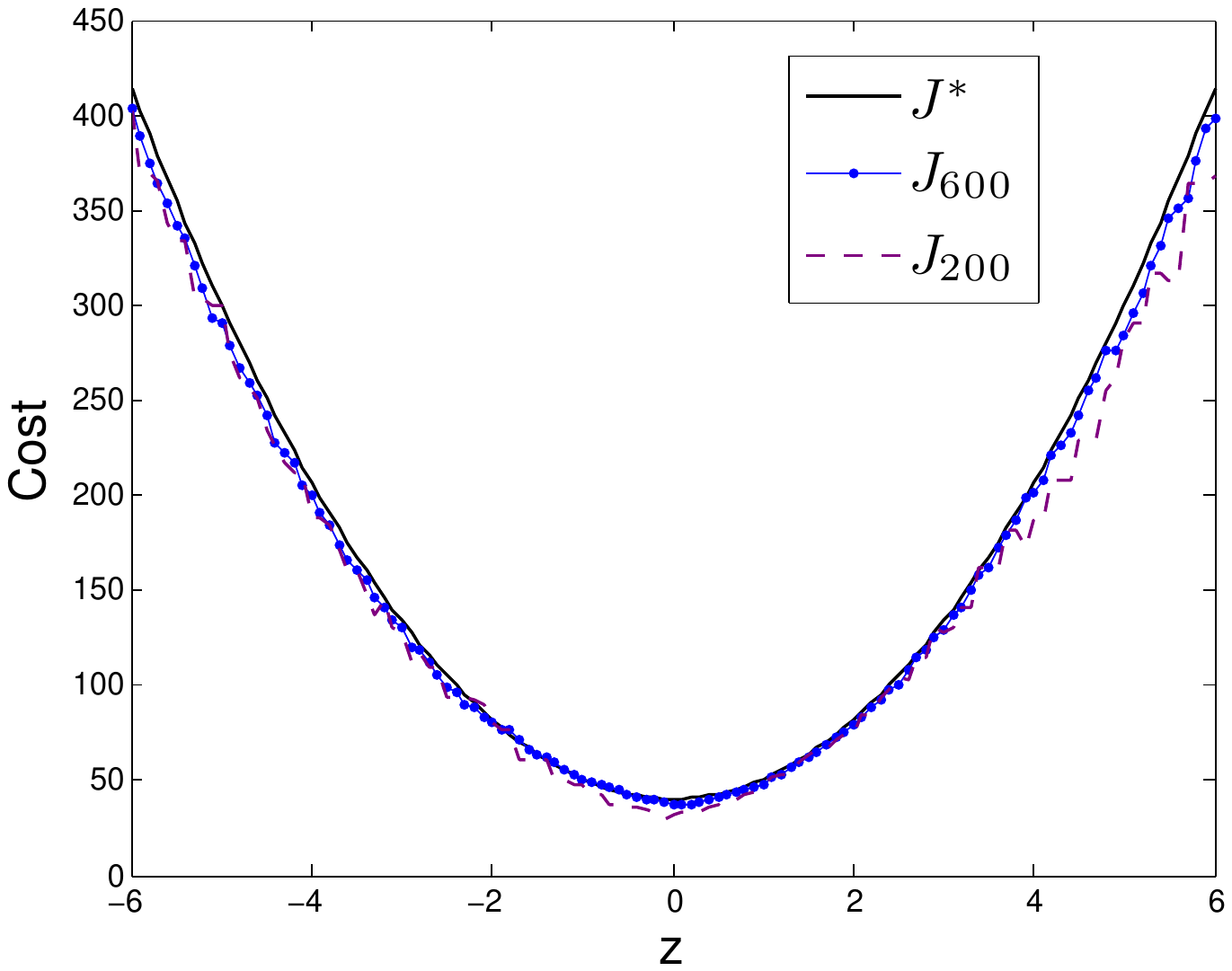}
  }
  \hfill
  \subfigure[After 200 iterations (0.39s).]{
  \label{figsubLQR2}
  \includegraphics[width=51mm,bb= 108 240 485 530]{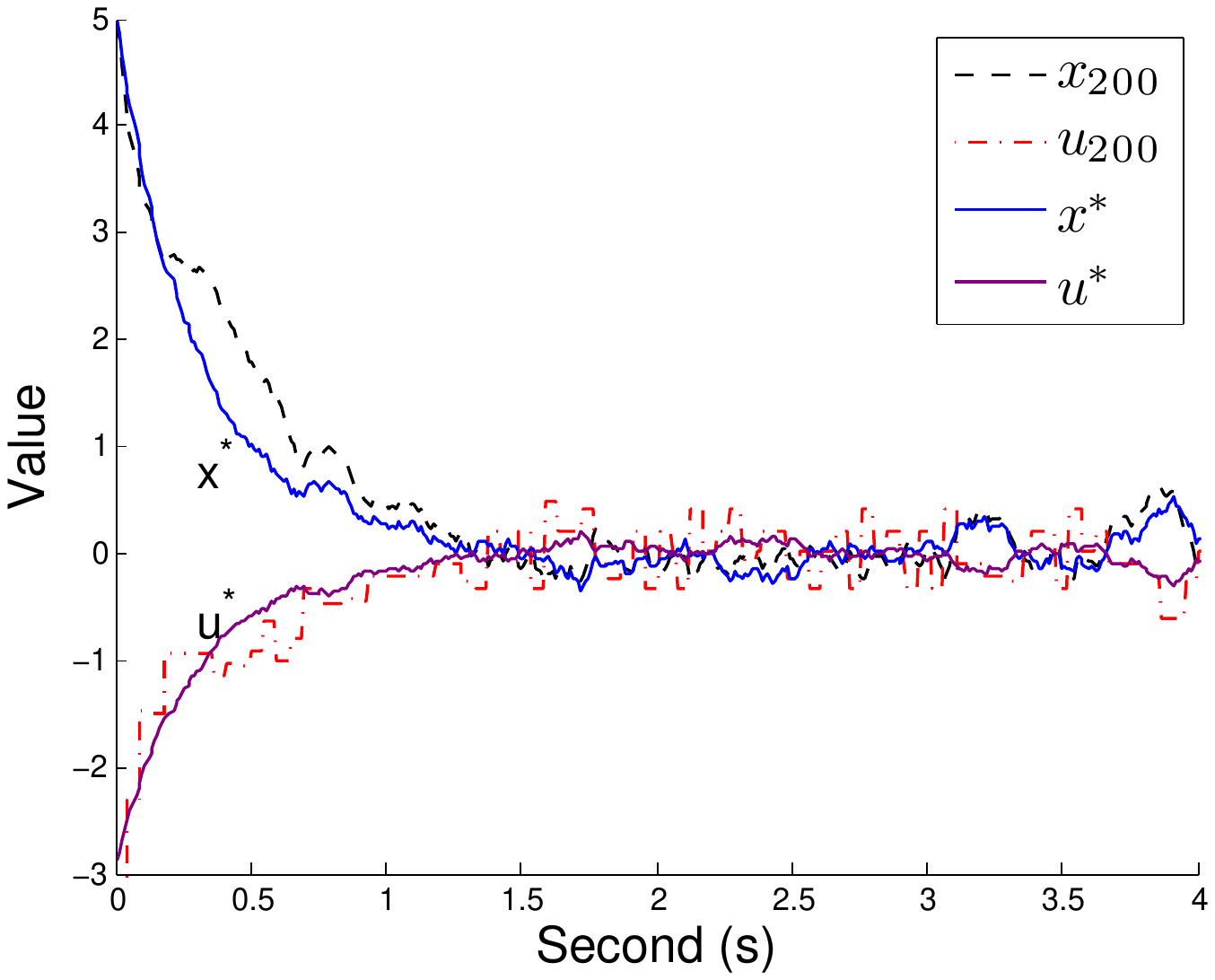}
  }
  \hfill
  \subfigure[After 600 iterations (2.16s).]{
   \label{figsubLQR3}
  \includegraphics[width=51mm,bb= 108 240 485 530]{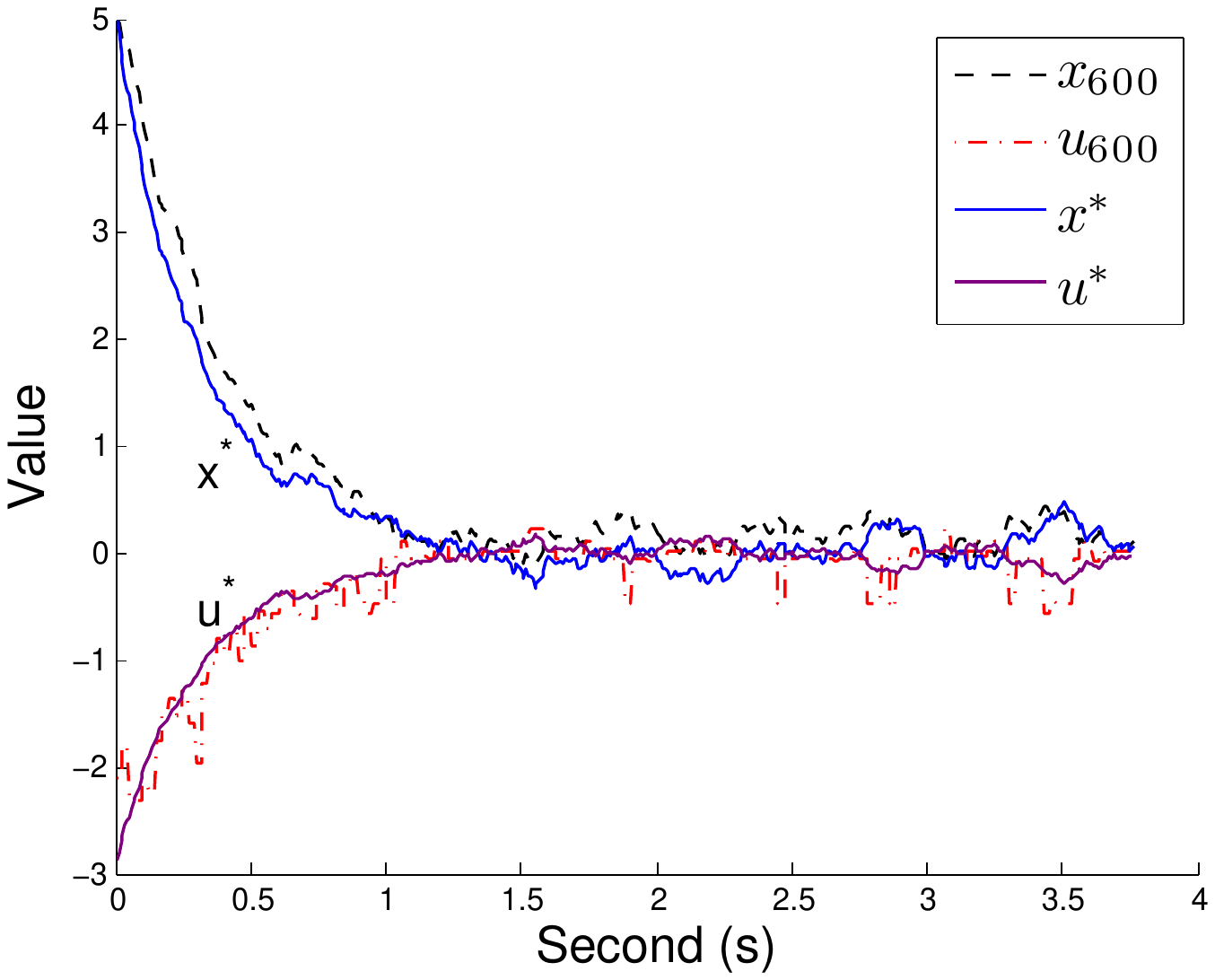}
   }
  \subfigure[Mean and 1-$\sigma$ interval of $||J_n-J^*||_{S_n}$.]{
  \label{figMeasureVariance}
  \includegraphics[scale=0.49,bb=75 320 605 470]{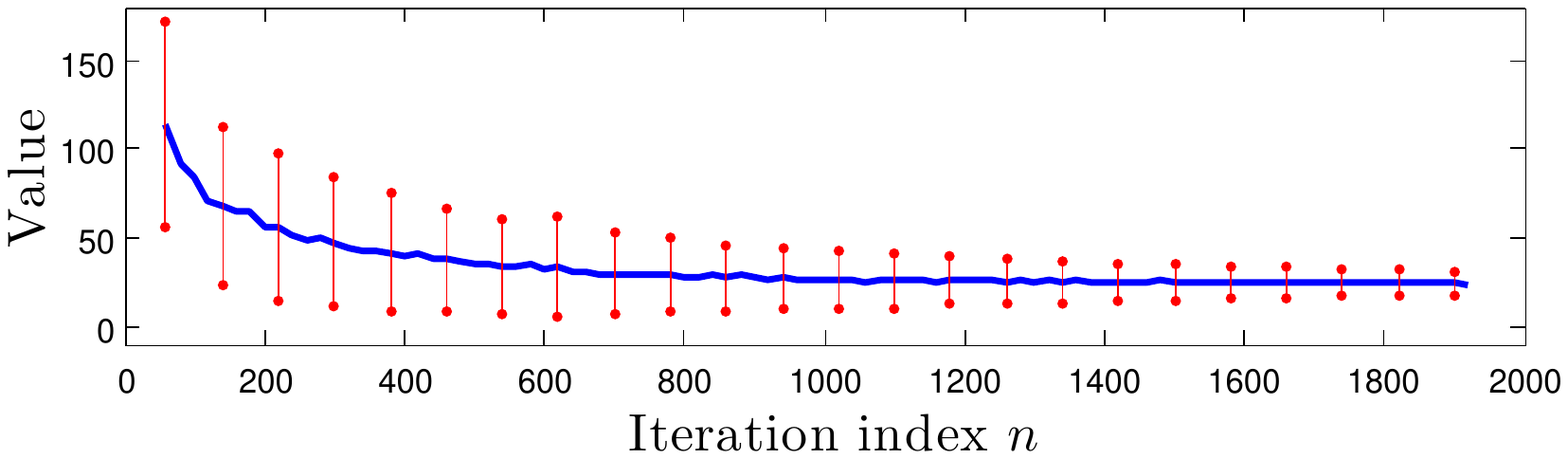}  }
  \subfigure[Log-log plot of Fig.~\ref{figMeasureVariance} .]{
  \label{figLoglogplot}
  \includegraphics[scale=0.49,bb= 142 320 525 470]{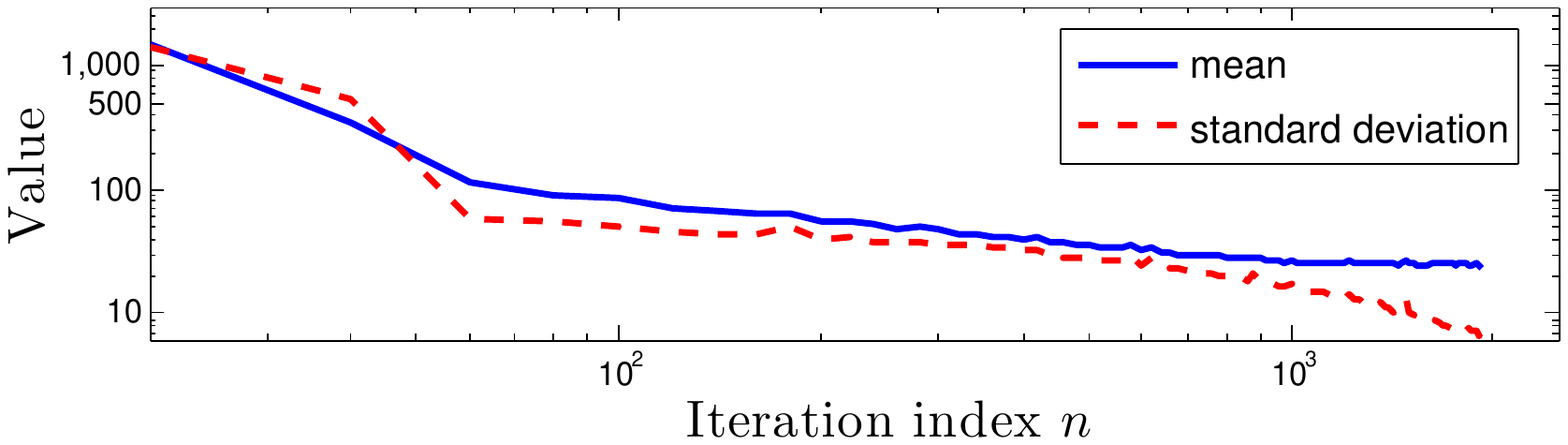}   }
  \subfigure[Plot of ratio $||J_n -J^*||_{S_n}/\big (\log(|S_n|)/|S_n| \big)^{0.5}$]{
  \label{figRatioConvergence}
  \includegraphics[scale=0.49,bb=75 320 605 470]{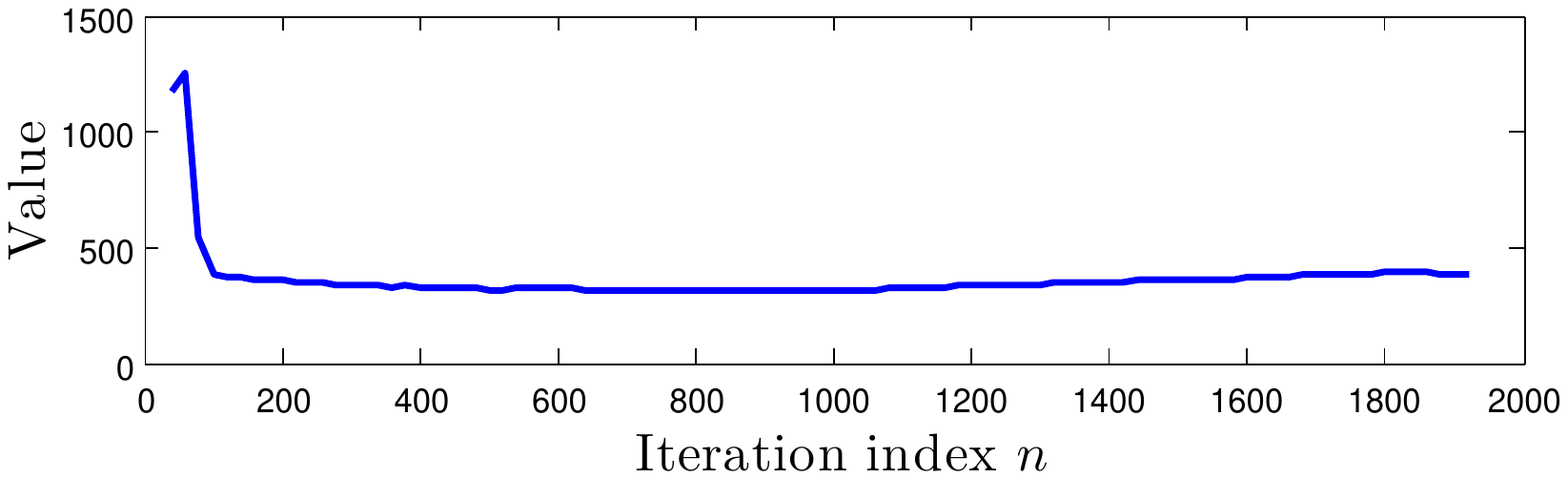}  }
  \subfigure[Plot of ratio $T_n/\big(|S_n|^{0.5}\log(|S_n|) \big )$.]{
  \label{figMeasureTime}
  \includegraphics[scale=0.49,bb= 142 320 525 470]{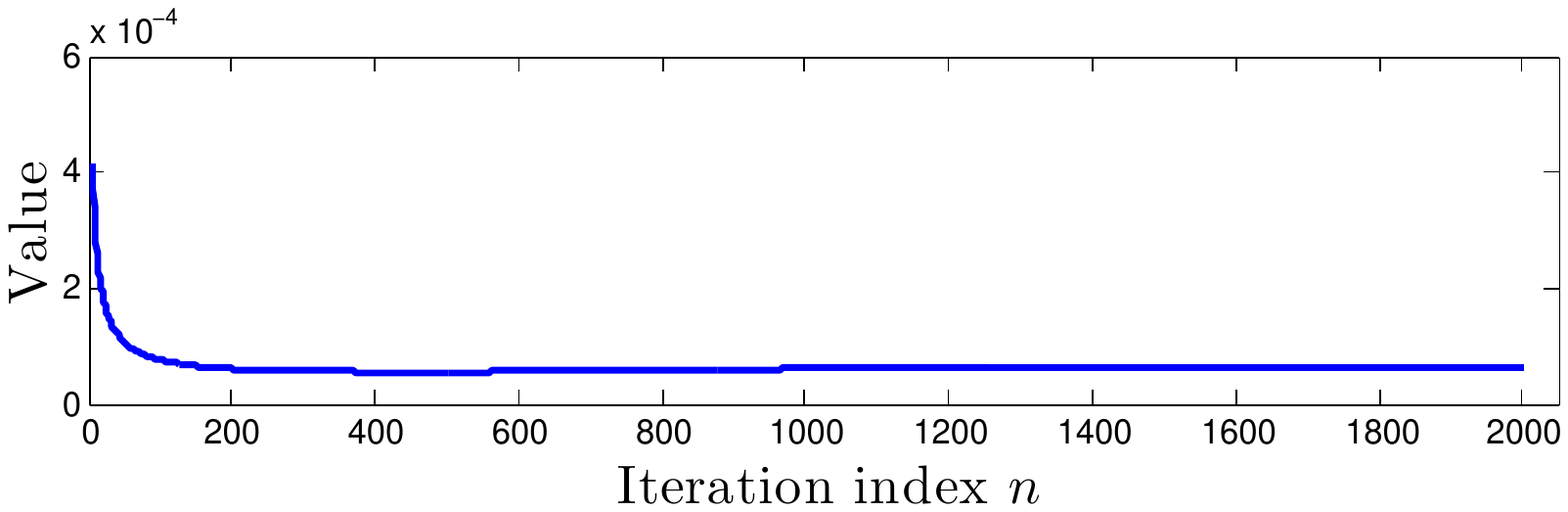}   }

\end{center}
  \label{figLQR}
  \vspace{-0.18in}
  \caption{Results of iMDP on a stochastic LQR problem. Figure \ref{figsubLQR1} shows the convergence of approximated cost-to-go to the optimal analytical cost-to-go over iterations. Anytime solutions are compared to the analytical optimal solution after 200 and 600 iterations in Figs. \ref{figsubLQR2}-\ref{figsubLQR3}. Mean and 1-$\sigma$ interval of the error $||J_n - J^*||_{S_n}$ are shown in \ref{figMeasureVariance} using 50 trials. The corresponding mean and standard deviation of the error $||J_n - J^*||_{S_n}$ are depicted on a log-log plot in Fig.~\ref{figLoglogplot}. In Fig.~\ref{figRatioConvergence}, we plot the ratio of $||J_n -J^*||_{S_n}$ to $(\log(|S_n|)/|S_n|)^{0.5}$ to show the convergence rate of $J_n$ to $J^*$. Figure \ref{figMeasureTime} shows the ratio of running time per iteration $T_n$ to $|S_n|^{0.5}\log(|S_n|)$. Ratios in Figs.~\ref{figRatioConvergence}-\ref{figMeasureTime} are averaged over 50 trials.}
\end{figure*}

\begin{figure*}
\begin{center}
  \vspace{-0.08in}
  \subfigure[Policy after 500 iterations (0.5s).]{
  \label{figsubControl1}
  \includegraphics[width=51mm,bb= 125 240 485 530]{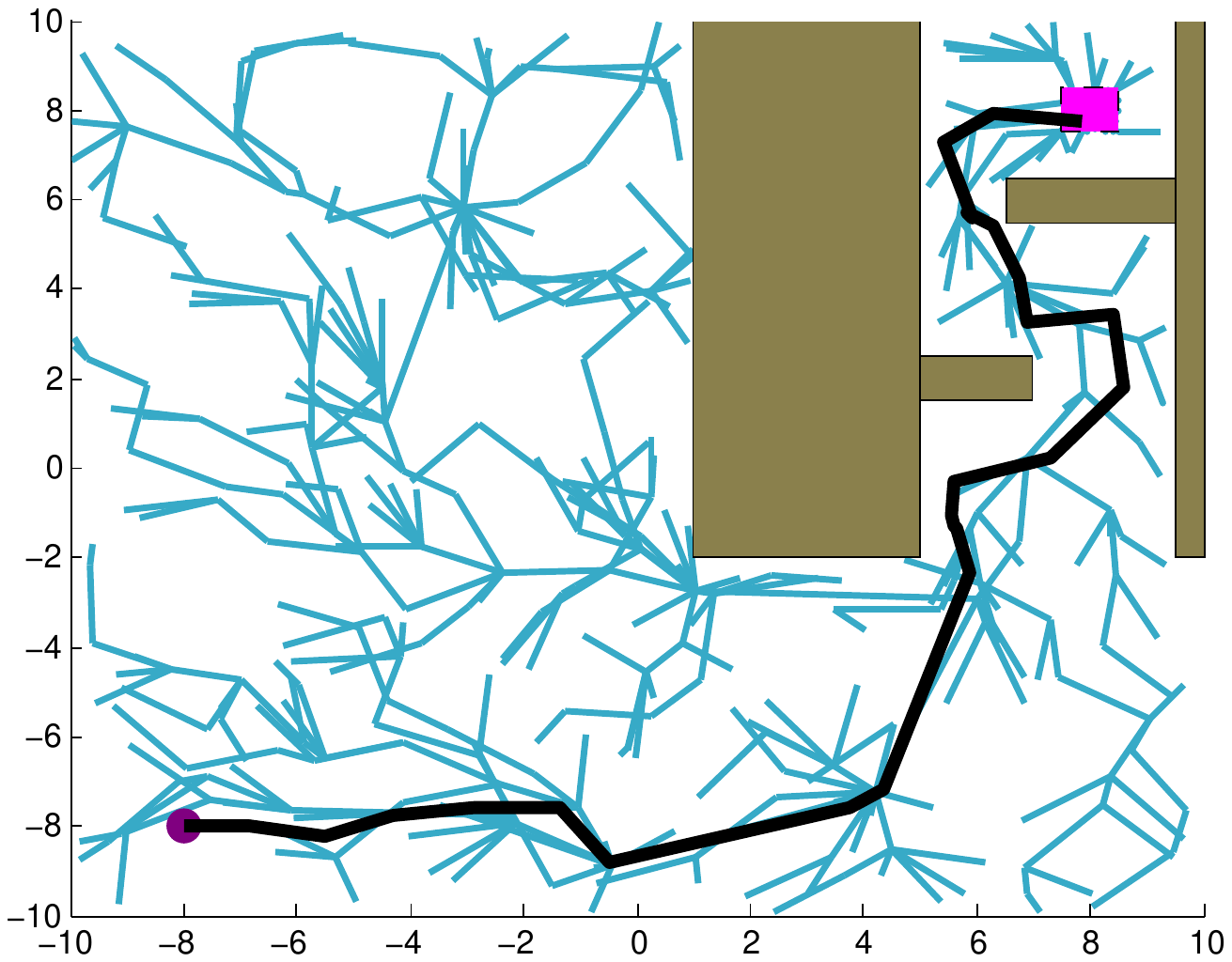}
  }
  \hfill
  \subfigure[Policy after 1,000 iterations (1.2s).]{
  \label{figsubControl2}
  \includegraphics[width=51mm,bb= 125 240 485 530]{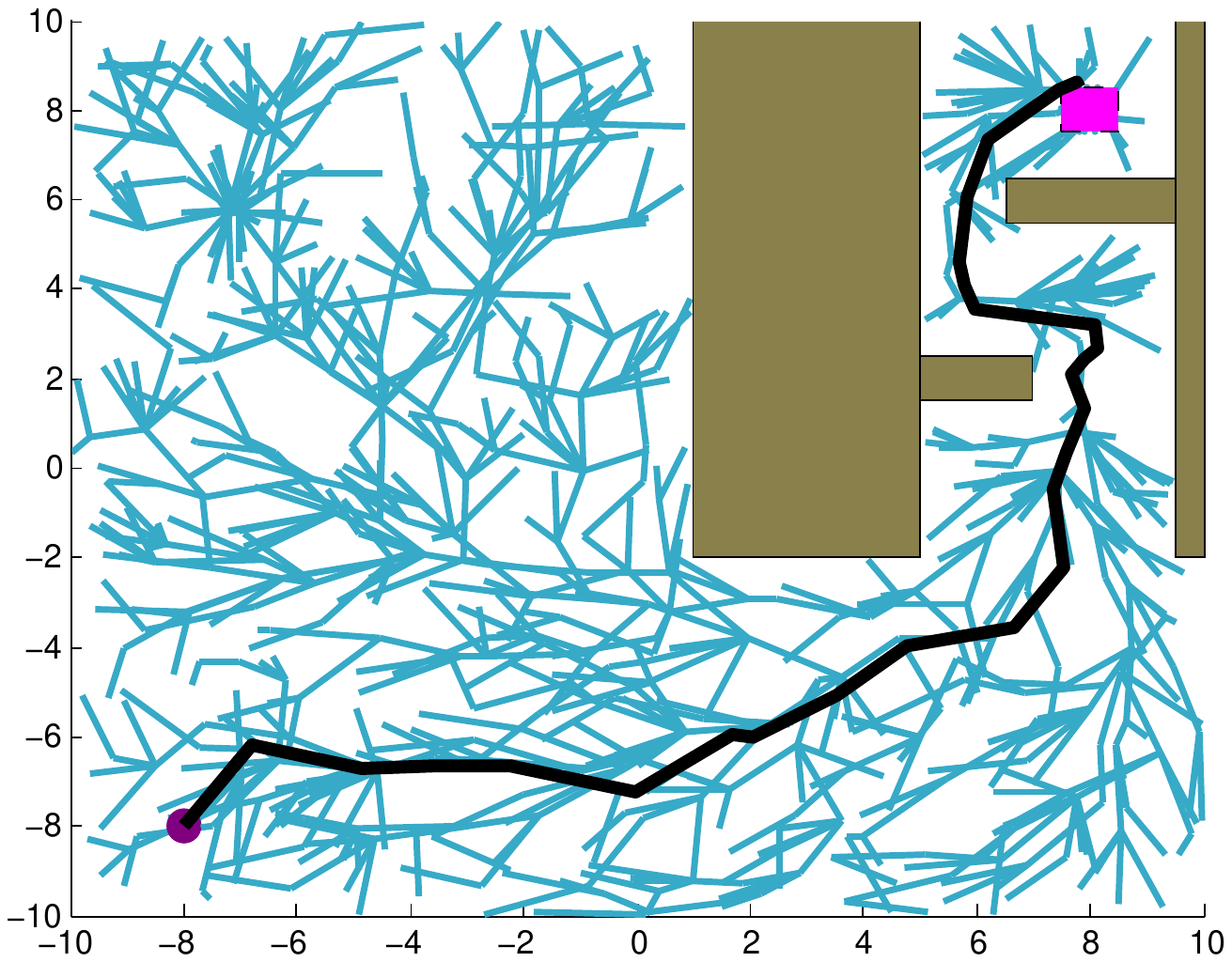}
  }
  \hfill
  \subfigure[Policy after 2,000 iterations (2.1s).]{
  \label{figsubControl3}
  \includegraphics[width=51mm,bb= 125 240 485 530]{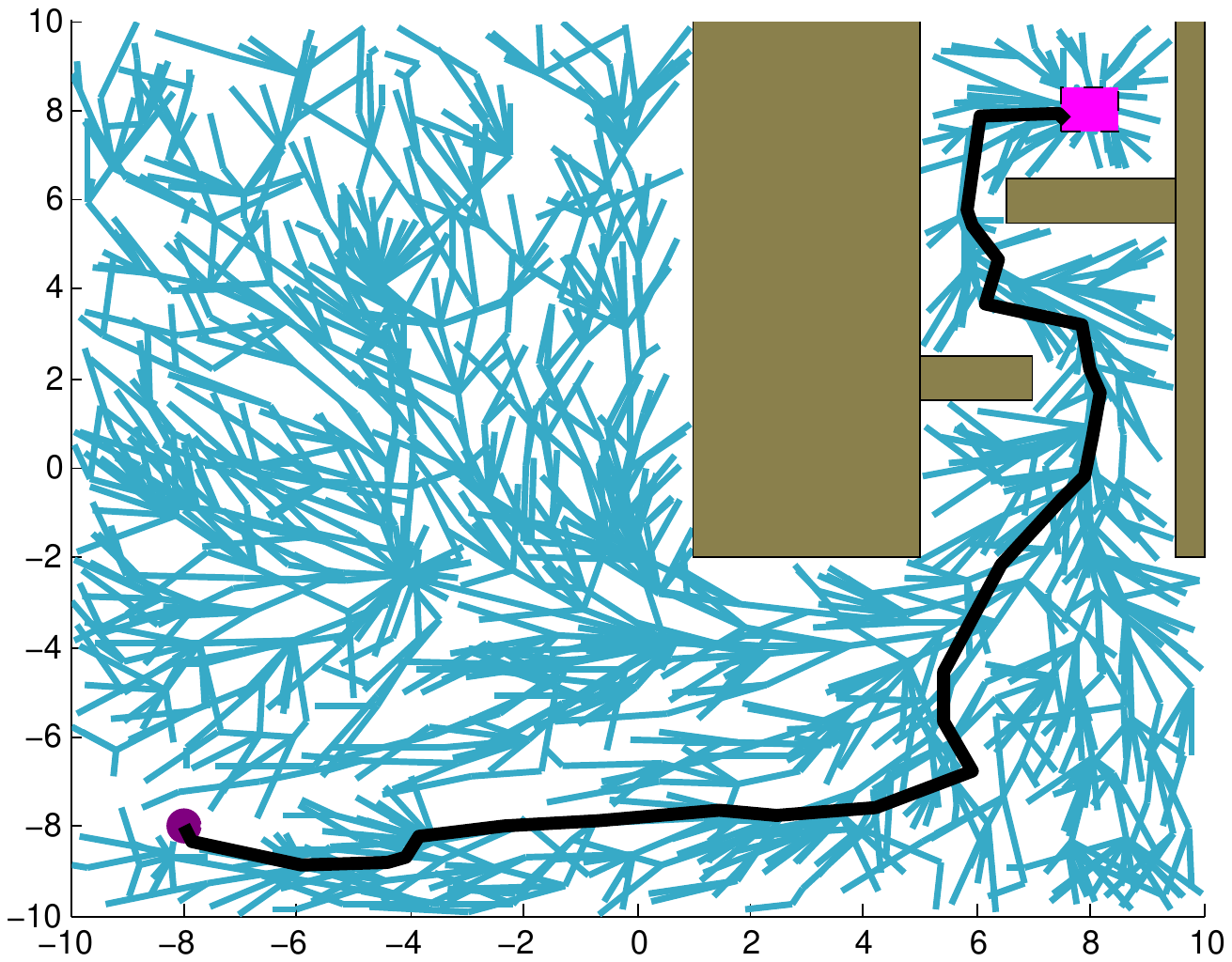}
  }
  \subfigure[Contour of $J_{500}$]{
  \label{figsubContour1}
  \includegraphics[width=51mm,bb= 125 240 485 530]{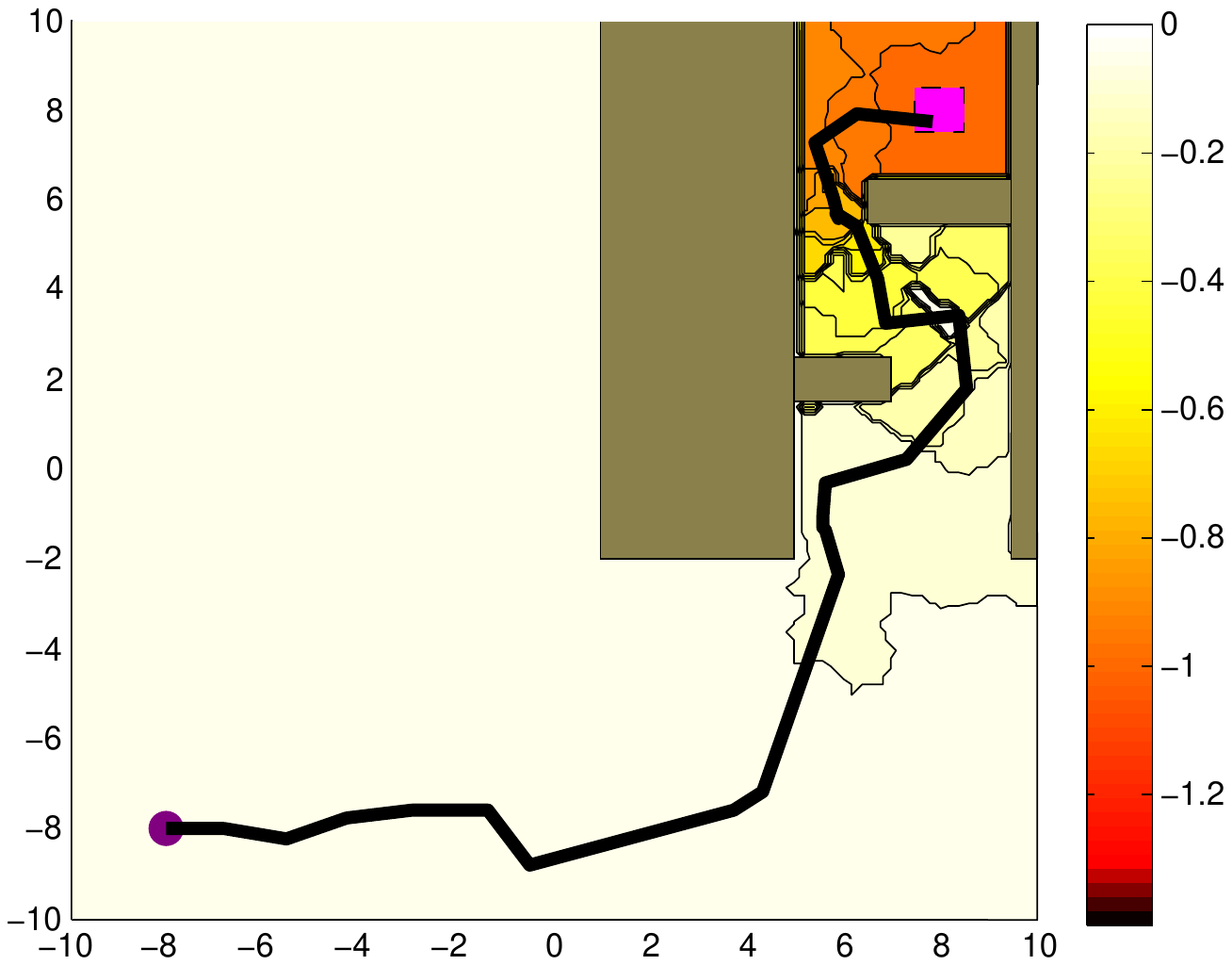}
  }
  \hfill
  \subfigure[Contour of $J_{1,000}$]{
  \label{figsubContour2}
  \includegraphics[width=51mm, bb= 125 240 485 530]{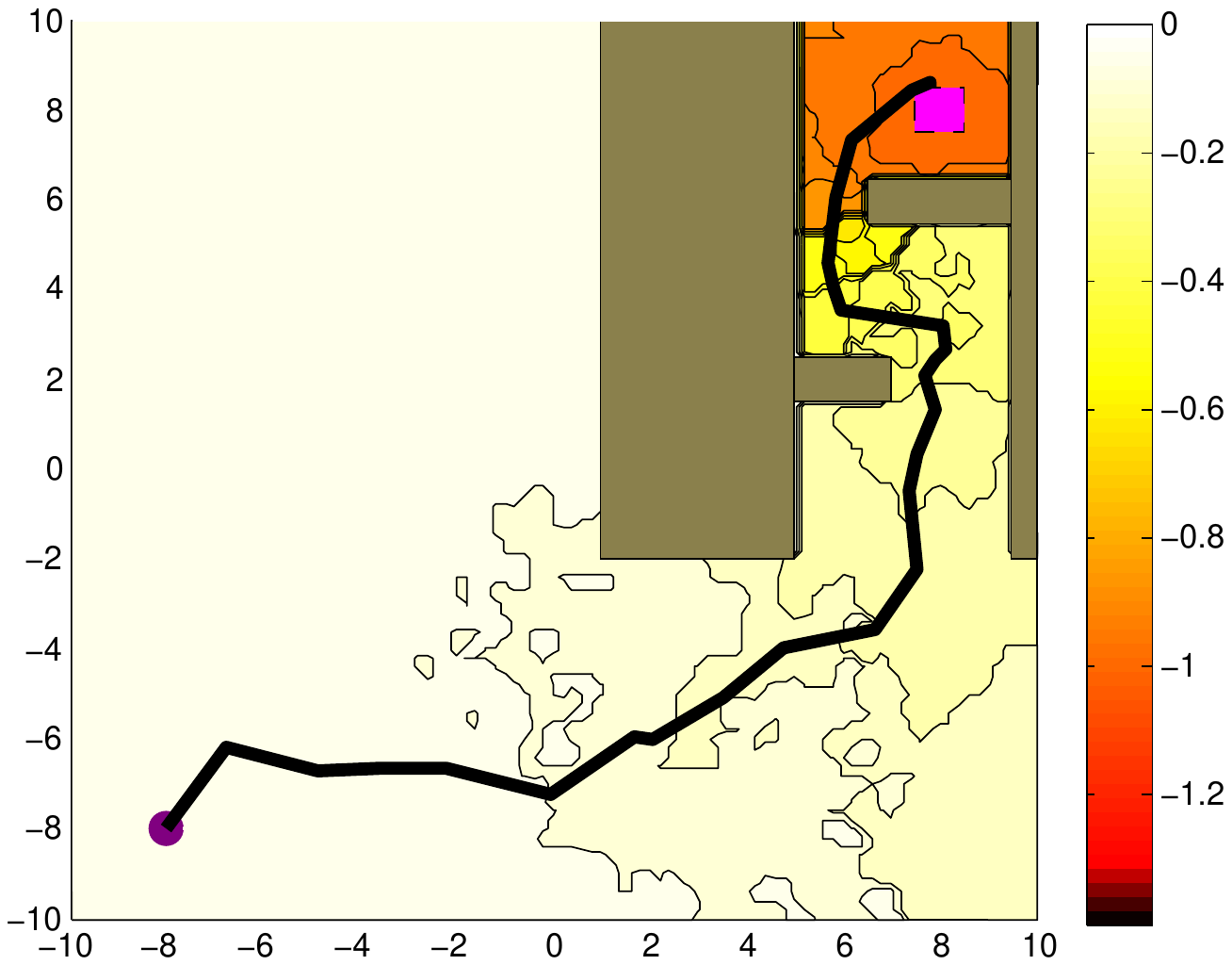}
  }
  \hfill
  \subfigure[Contour of $J_{2,000}$]{
  \label{figsubContour3}
  \includegraphics[width=51mm, bb= 125 240 485 530]{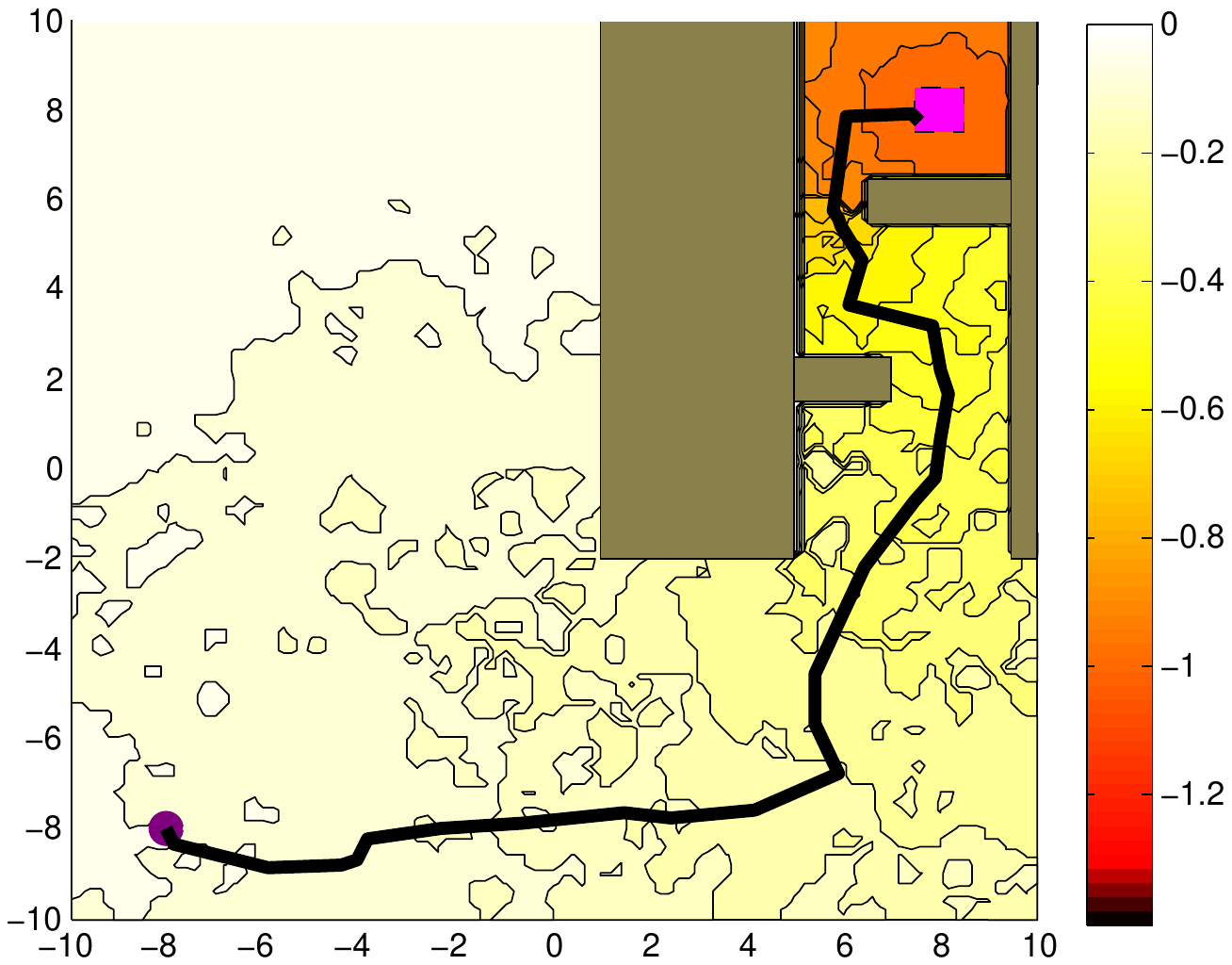}
  }
  \subfigure[Policy after 4,000 iterations (7.6s).]{
  \label{figsubControl4}
  \includegraphics[width=51mm,bb= 125 240 485 530]{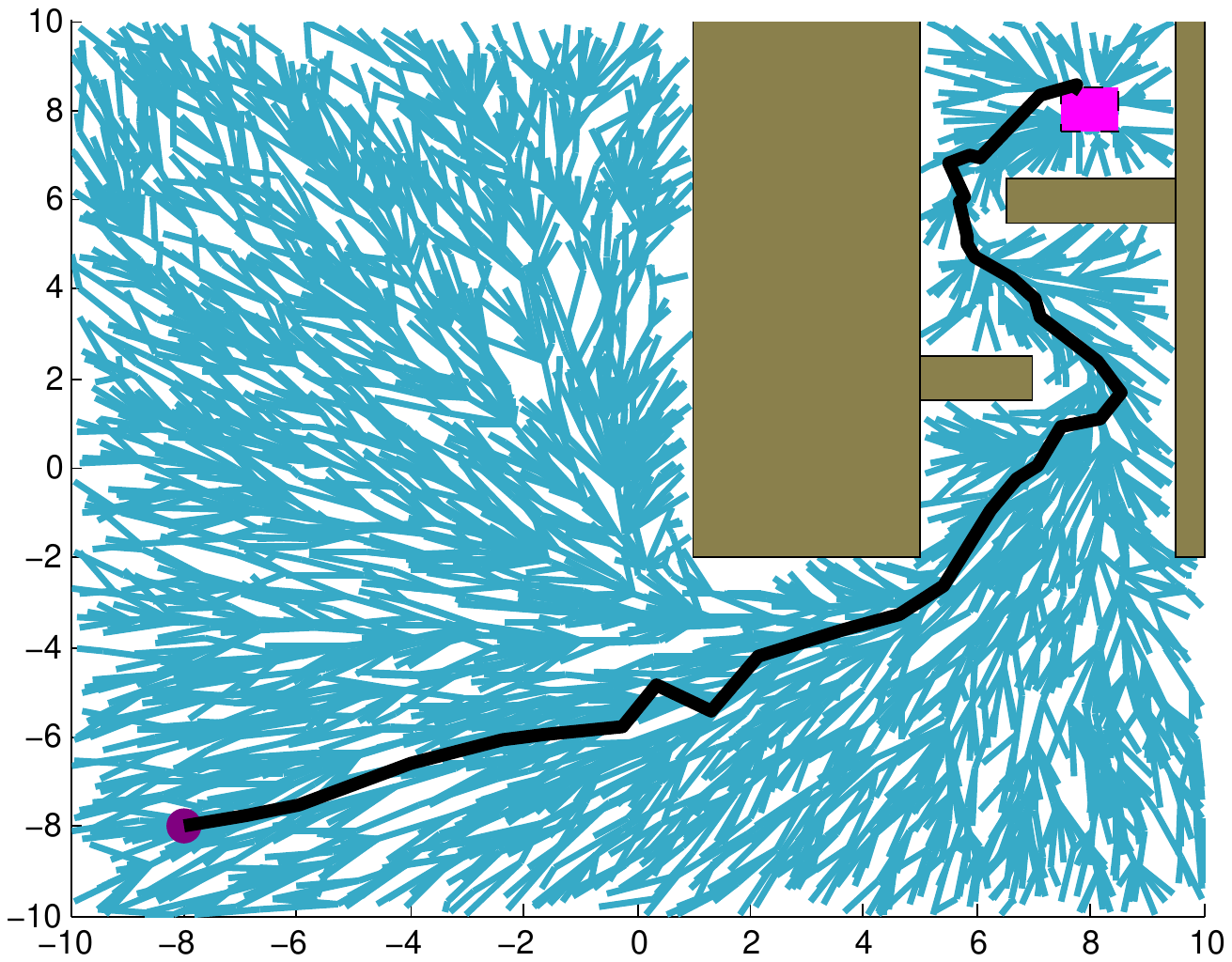}
  }
  \hfill
  \subfigure[Policy with 10,000 nodes (28s).]{
  \label{figsubControl5}
  \includegraphics[width=51mm,bb= 125 240 485 530]{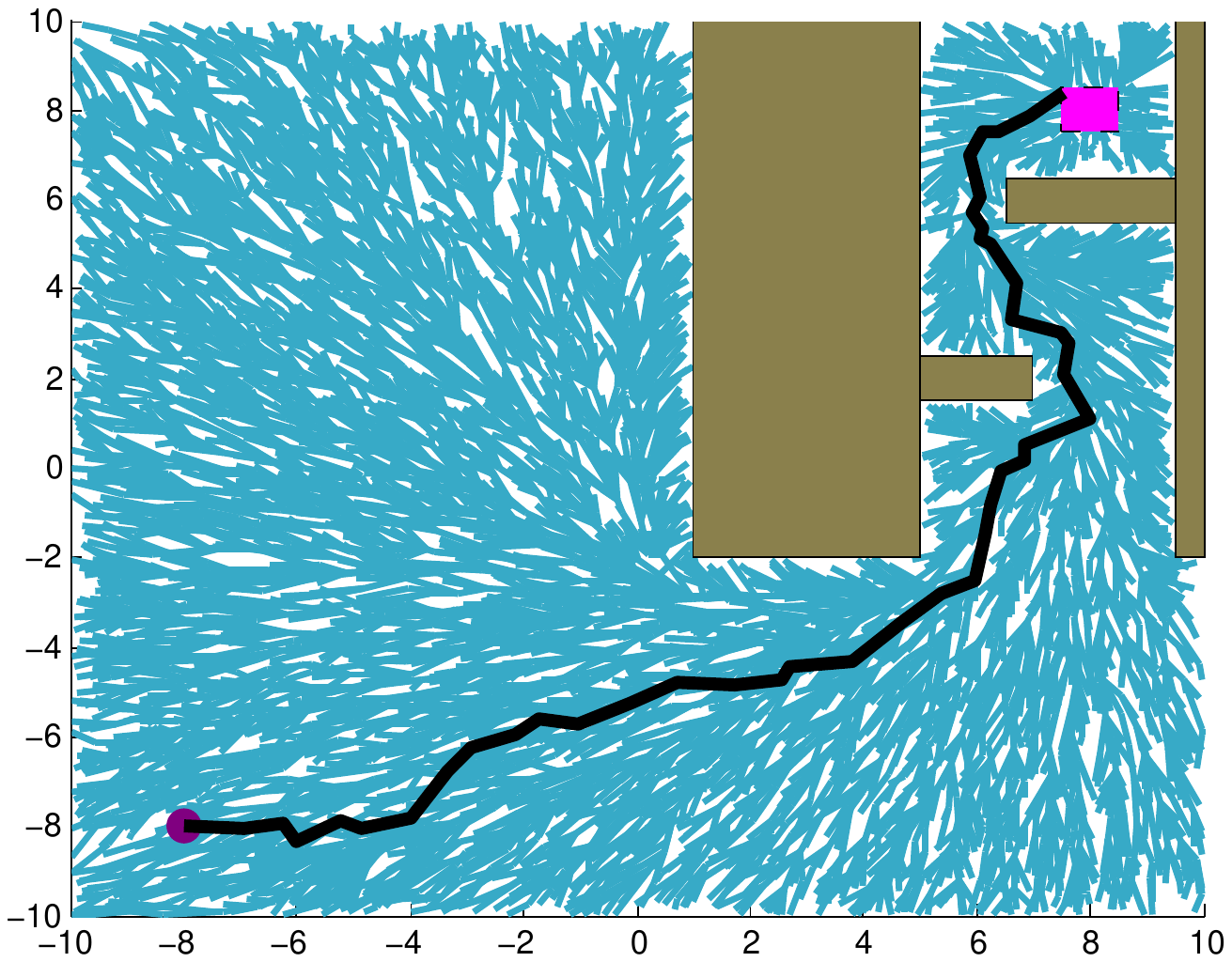}
  }
  \hfill
  \subfigure[Policy after 20,000 iterations (80s).]{
  \label{figsubControl6}
  \includegraphics[width=51mm,bb= 125 240 485 530]{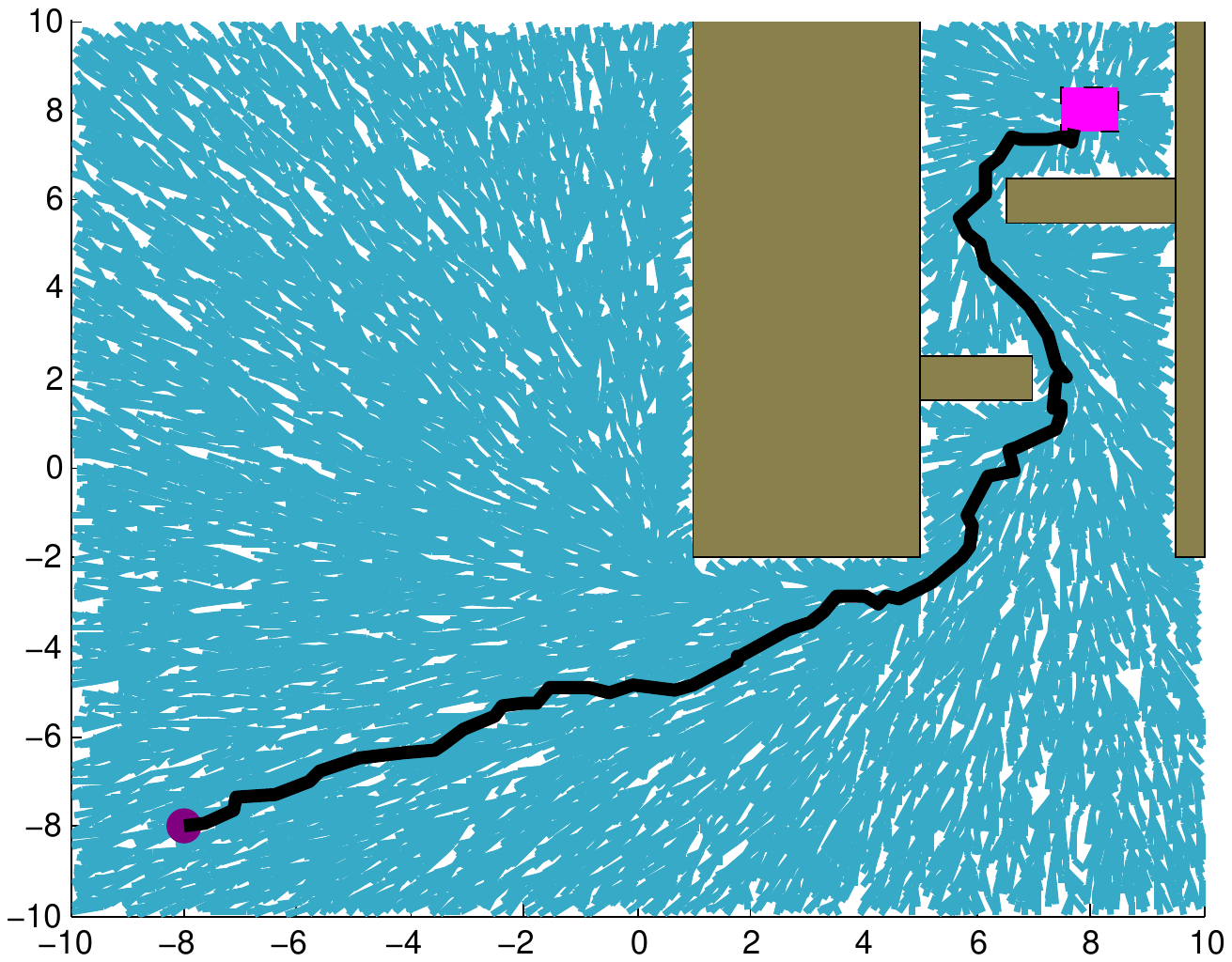}
  }
  \subfigure[Contour of $J_{4,000}$]{
  \label{figsubContour4}
  \includegraphics[width=51mm,bb= 125 240 485 530]{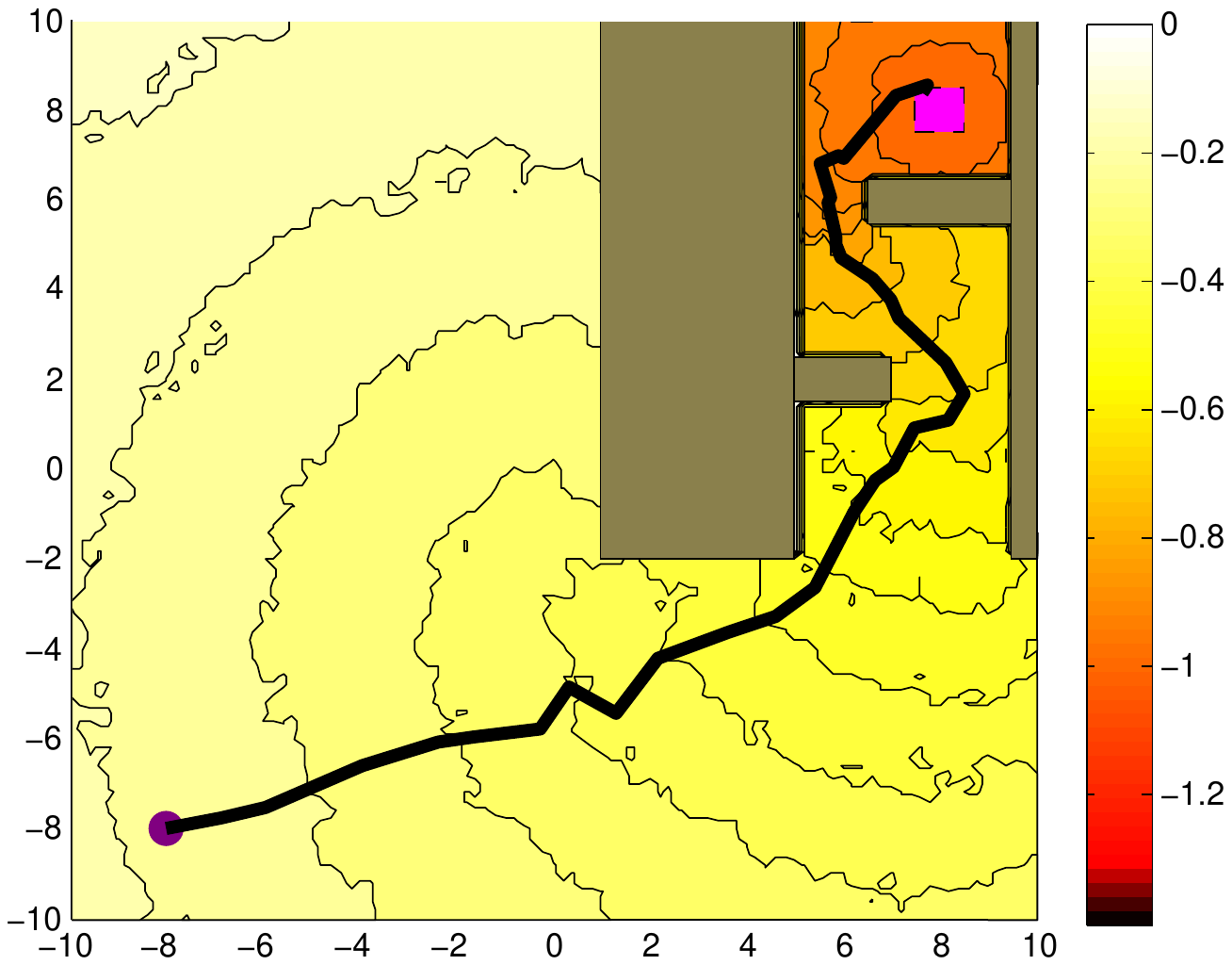}
  }
  \hfill
  \subfigure[Contour of $J_{10,000}$]{
  \label{figsubContour5}
  \includegraphics[width=51mm, bb= 125 240 485 530]{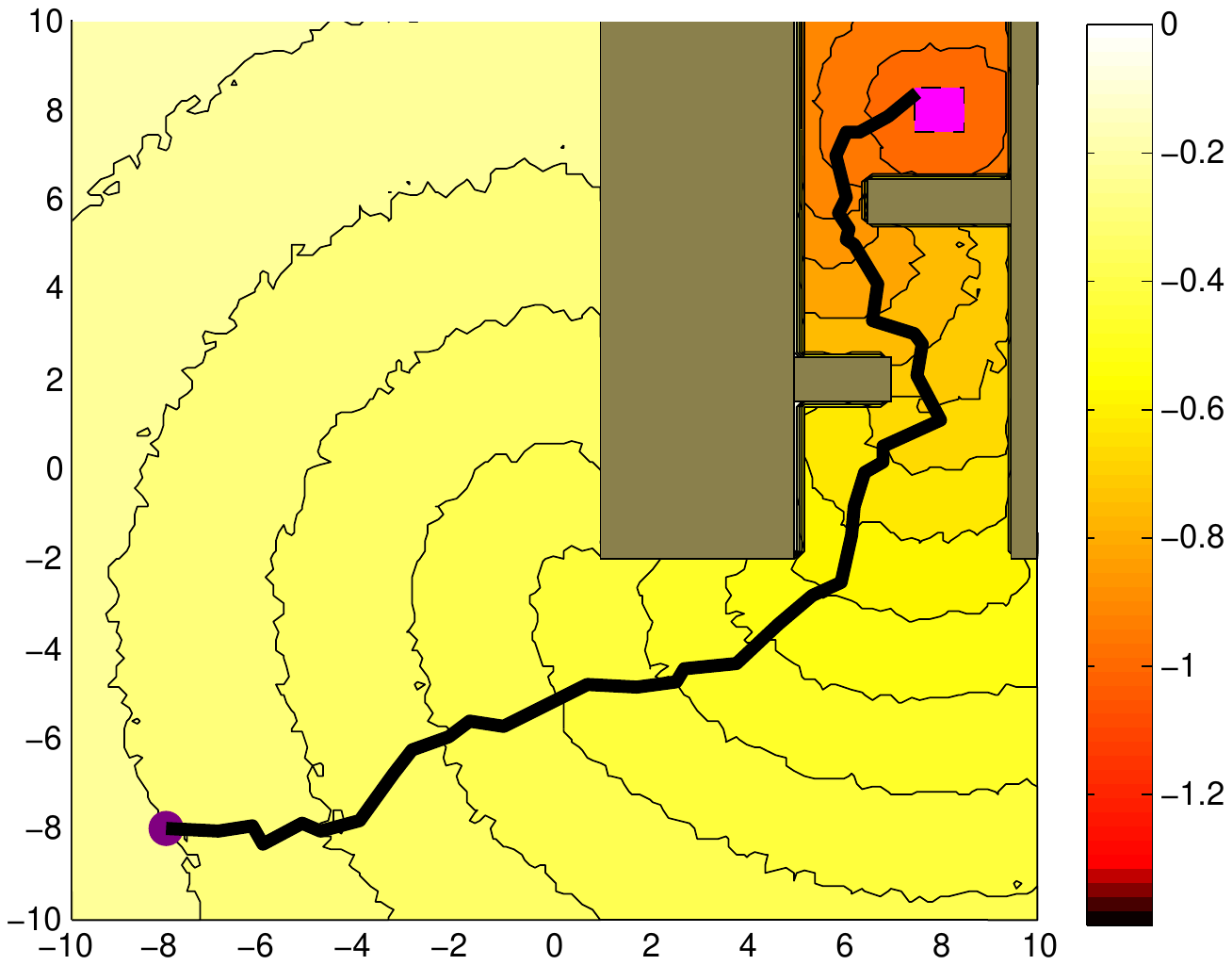}
  }
  \hfill
  \subfigure[Contour of $J_{20,000}$]{
  \label{figsubContour6}
  \includegraphics[width=51mm, bb= 125 240 485 530]{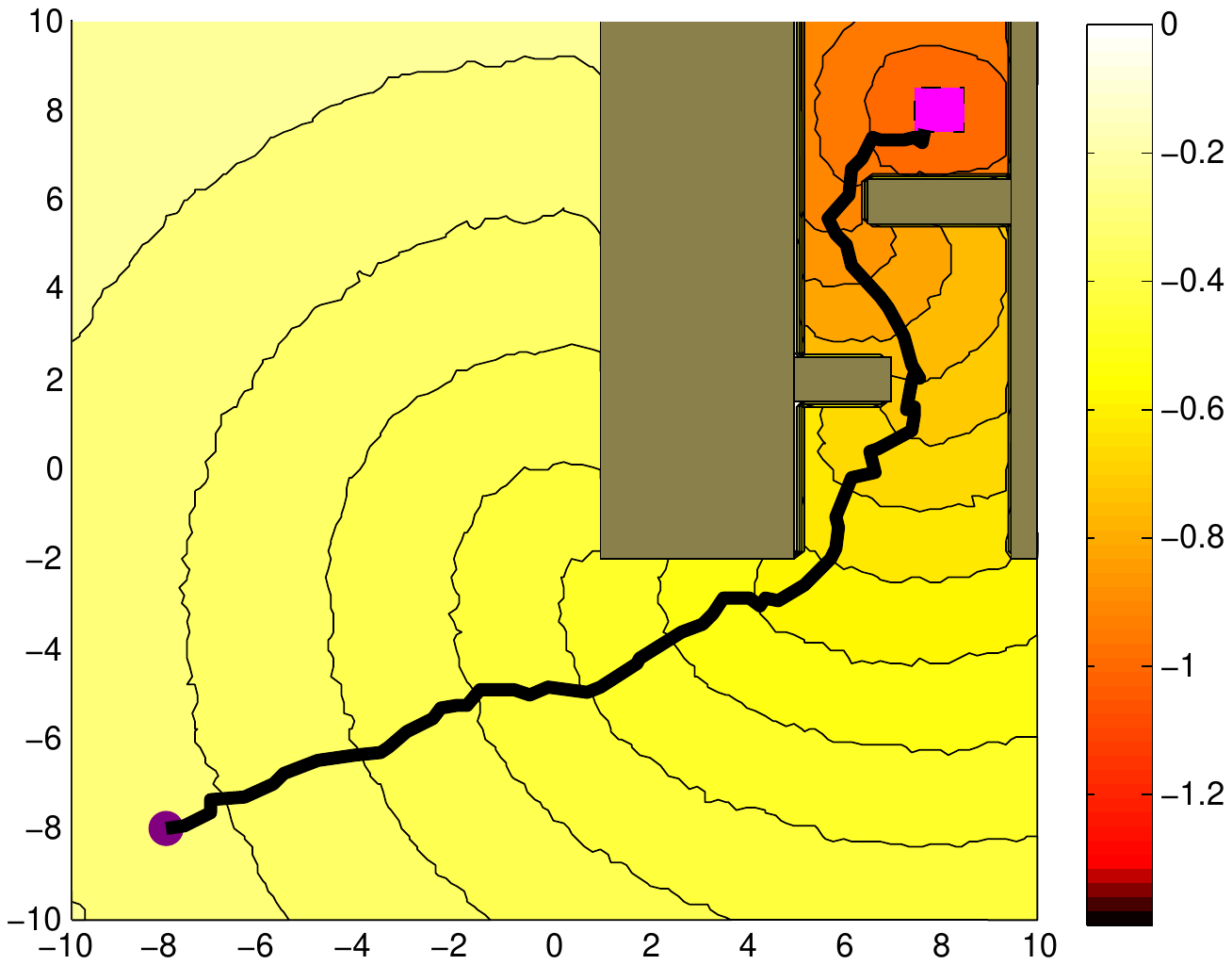}
  }
  \label{figSimple2D}
  \caption{A system with stochastic single integrator dynamics in a cluttered environment. With appropriate cost structure assigned to the goal and obstacle regions, the system reaches the goal in the upper right corner and avoids obstacles. The standard deviation of noise in x and y directions is 0.26. The maximum velocity is one. Anytime control policies and corresponding contours of approximated cost-to-go as shown in Figs. \ref{figsubControl1}-\ref{figsubContour6} indicate that iMDP quickly explores the state space and refines control policies over time.}
\end{center}
\end{figure*}
We used a computer with a 2.0-GHz Intel Core 2 Duo T6400 processor and $4$ GB of RAM to run experiments. In the first experiment, we investigated the convergence of the iMDP algorithm on a stochastic LQR problem: $\inf \EE \big[\int_{0}^{\tau}0.95^t \lbrace 3.5x(t)^2 + 200u(t)^2 \rbrace dt + h(x(\tau)) \big]$ such that $dx= (3x + 11u)dt + \sqrt{0.2}dw$ on the state space $S=[-6,6]$ where $\tau$ is the first hitting time to the boundary $\partial S=\{-6,6\}$, and $h(z)=414.55 $ for $z \in \partial S$ and $0$ otherwise. The optimal cost-to-go from $x(0)=z$ is $10.39z^2 + 40.51$, and  the optimal control policy is $u(t)=-0.5714x(t)$. Since the cost-rate function is bounded on $S$ and H\"{o}lder continuous with exponent $1.0$, we use $\rho=0.5$. In addition, we choose $\dummyasyn=0.5$, and $\dummyrate=0.99$ in the procedure ${\tt ComputeHoldingTime}$. We used the procedure ${\tt Update}$ as presented in Alogrithm \ref{algorithm:update} with $\log(n)$ sampled controls and transition probabilities having constant support size. Figures \ref{figsubLQR1}-\ref{figsubLQR3} show the convergence of approximated cost-to-go, anytime controls and trajectory to the optimal analytical counterparts over iterations. We observe that in Fig. \ref{figMeasureVariance}, both the mean and variance of cost-to-go error decreases quickly to zero. The log-log plot in Fig.~\ref{figLoglogplot} clearly indicates that both mean and standard deviation of the error $||J_n-J^*||_{S_n}$ continue to decrease. This observation is consistent with Theorems~\ref{theoremAlmostSureConvergence}-\ref{theoremAlmostSureConvergenceSampleControl}. Moreover, Fig.~\ref{figRatioConvergence} shows the ratio of $||J_n-J^*||_{S_n}$ to $(\log(|S_n|)/|S_n|)^{0.5}$ indicating the convergence rate of $J_n$ to $J^*$, which agrees with Theorem~\ref{theorem:JoptnJoptas}. Finally, Fig. \ref{figMeasureTime} plots the ratio of running time per iteration $T_n$ to $|S_n|^{0.5}\log(|S_n|)$ asserting that the time complexity per iteration is $O\big( |S_n|^{0.5}\log(|S_n|) \big )$.

\begin{figure*}
\begin{center}
  \subfigure[Noise-free: 1,000 iterations(1.2s).]{
  \label{figNoisefreeCor}
  \includegraphics[width=51mm ,bb= 125 240 485 530]{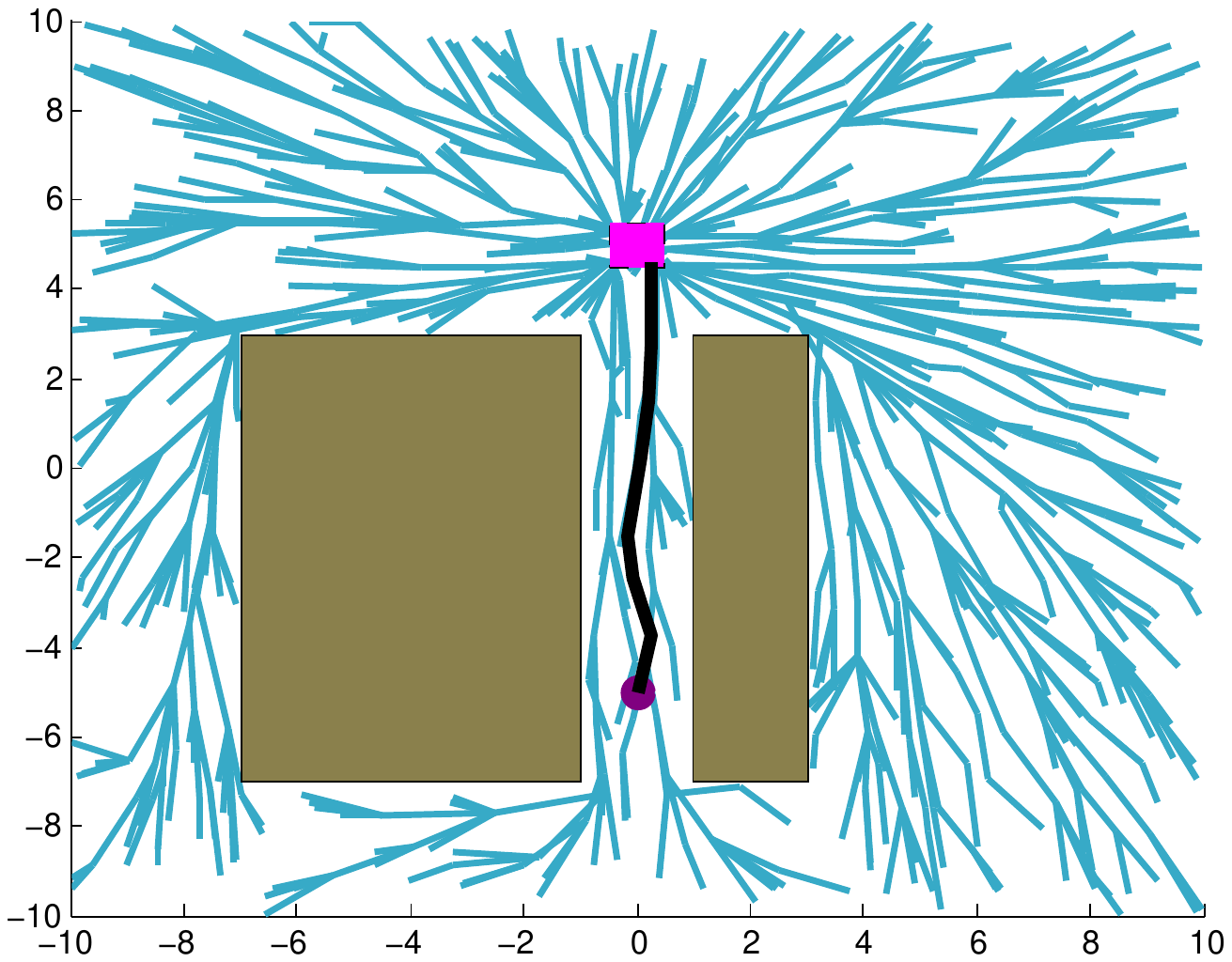}
  }
  \hfill
  \subfigure[Stochastic: 300 iterations (0.4s).]{
  \label{figNoisyCor}
 \includegraphics[width=51mm, bb= 125 240 485 530]{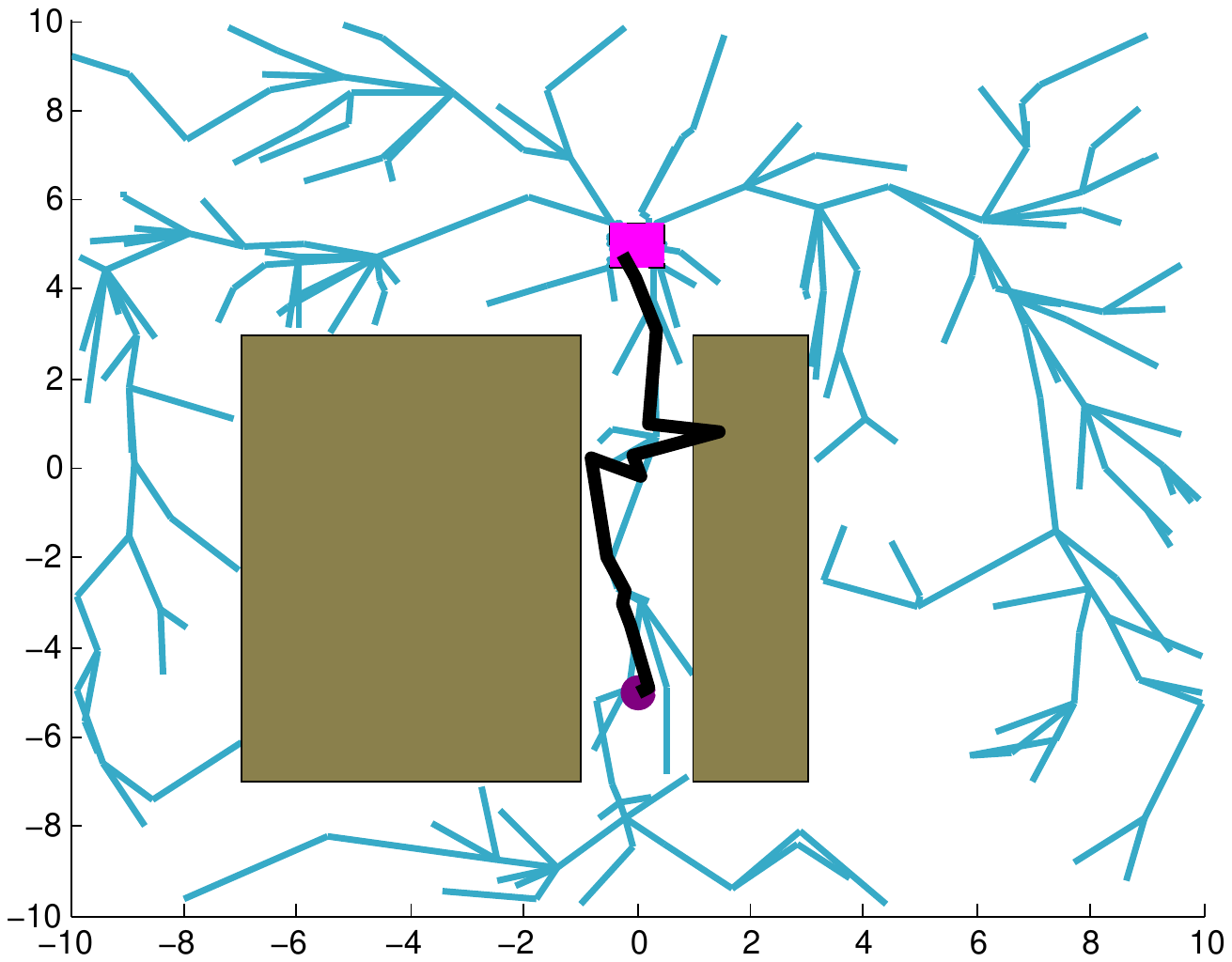}
  }
  \hfill
  \subfigure[Stochastic: 1,000 iterations (1.1s).]{
  \label{figNoisyCorMany}
  \includegraphics[width=51mm ,bb= 125 240 485 530]{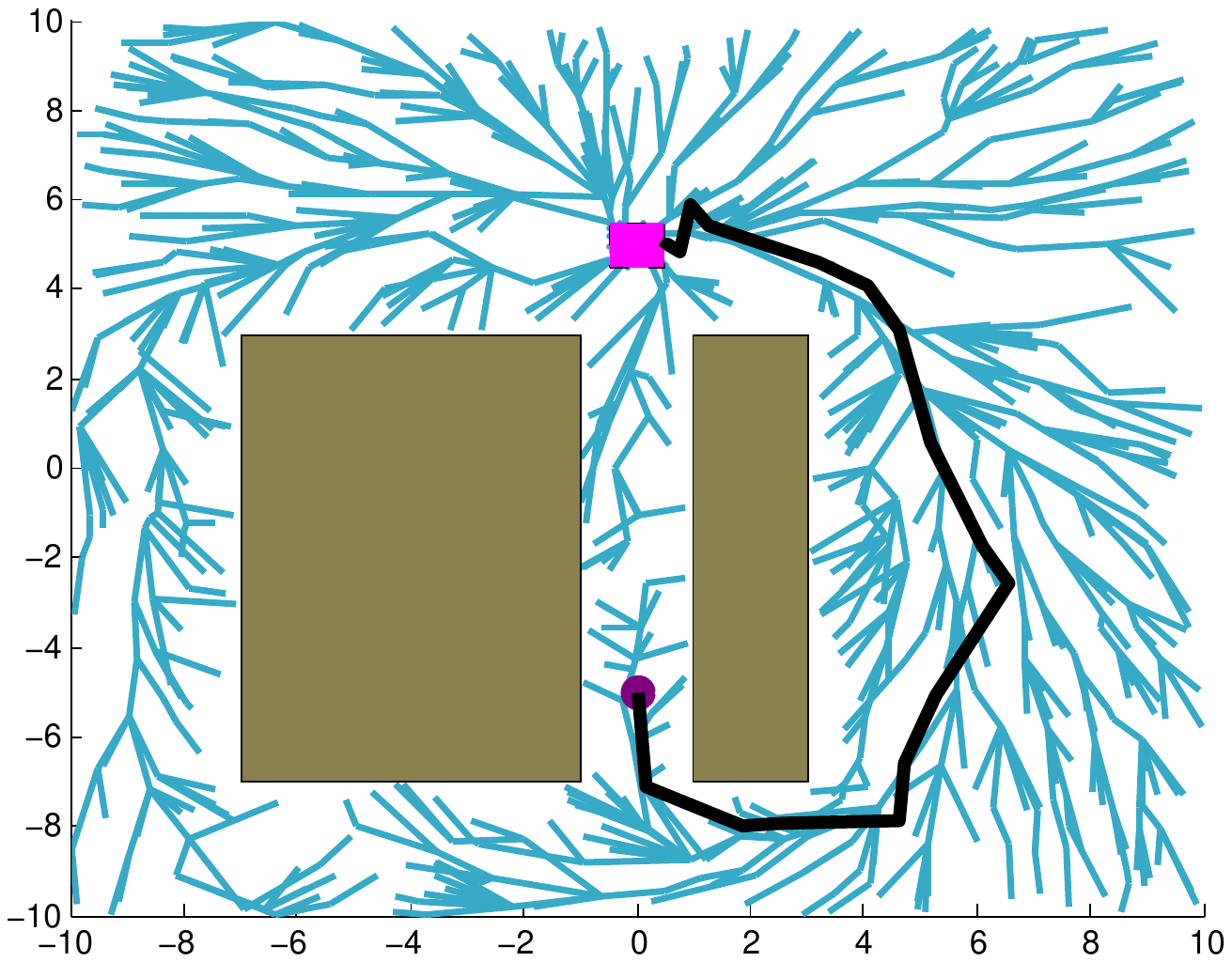}
  }
  \caption{Performance against different process noise magnitude. The system starts from (0,-5) to reach the goal. In Fig. \ref{figNoisefreeCor}, the environment is noise-free. In Figs. \ref{figNoisyCor}-\ref{figNoisyCorMany}, standard deviation of noise in x and y directions is 0.37. In the latter, the system first discovers an unsafe route that is prone to collisions and discovers a safer route after a few seconds. (In Fig. \ref{figNoisyCor}, we temporarily let the system continue even after collision to observe the entire trajectory.)}
\end{center}
\vspace{-0.15in}
\end{figure*}

\begin{figure*}
\begin{center}
  \subfigure[Trajectory snapshots after 3000 iterations (15.8s).]{
  \label{figManipulatorTraj}
  \includegraphics[width=75mm ,bb= 125 240 485 530]{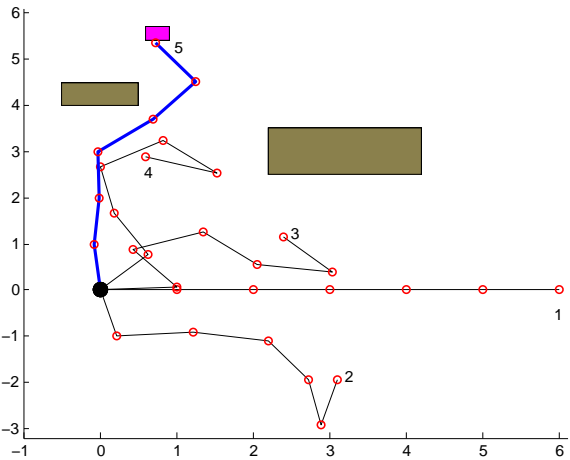}
  }
  \hfill
  \subfigure[Mean and standard deviation of cost values $J_n(x_0)$. ]{
  \label{figManipulatorCost}
 \includegraphics[width=75mm, bb= 125 240 485 530]{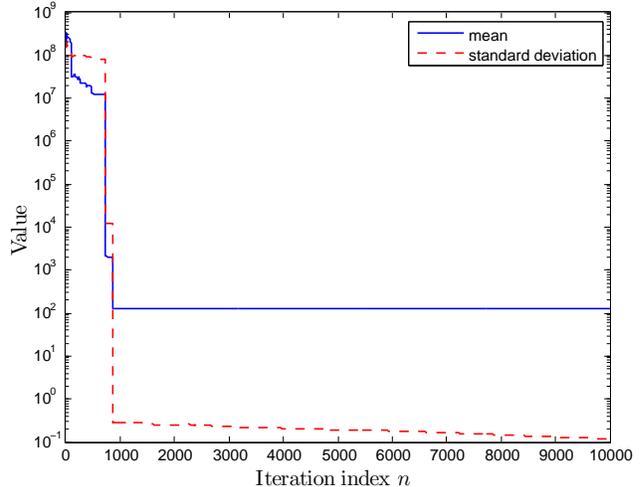}
  }
  \caption{Results of a 6D manipulator example. The system is modeled as a single integrator with states representing angles between segments and the horizontal line. Control magnitude is bounded by 0.3. The standard deviation of noise at each joint is 0.032 rad. In Fig. \ref{figManipulatorTraj}, the manipulator is controlled to reach a goal with the final upright position. In Fig.~\ref{figManipulatorCost}, the mean and standard deviation of the computed cost values for the initial position are plotted using 50 trials.}
\end{center}
\vspace{-0.15in}
\end{figure*}

In the second experiment, we controlled a system with stochastic single integrator dynamics to a goal region with free ending time in a cluttered environment. The cost objective function is discounted with $\alpha=0.95$. The system pays zero cost for each action it takes and pays a cost of -1 when reaching the goal region ${\cal{X}}_{goal}$. The maximum velocity of the system is one. The system stops when it collides with obstacles. 
We show how the system reaches the goal in the upper right corner and avoids obstacles with different anytime controls. Anytime control policies after up-to 2,000 iterations in Figs. \ref{figsubControl1}-\ref{figsubControl3}, which were obtained within $2.1$ seconds, indicate that iMDP quickly explores the state space and refines control policies over time. Corresponding contours of cost value functions are shown in Figs. \ref{figsubContour1}-\ref{figsubContour3} further illustrate the refinement and convergence of cost value functions to the original optimal cost-to-go over time. We observe that the performance is suitable for real-time control. Furthermore, anytime control policies and cost value functions after up-to 20,000 iterations are shown in Figs.~\ref{figsubControl4}-\ref{figsubControl6} and Figs.~\ref{figsubContour4}-\ref{figsubContour6} respectively. We note that the control policies seem to converge faster than cost value functions over iterations. The phenomenon is due to the fact that cost value functions $J_n$ are the estimates of the optimal cost-to-go $J^*$. Thus, when $J_n(z) - J^*(z)$ is constant for all $z \in S_n$, updated controls after a Bellman update are close to their optimal values. Thus, the phenomenon favors the use of the iMDP algorithm in real-time applications where only a small number of iterations are executed.

In the third experiment, we tested the effect of process noise magnitude on the solution trajectories. In Figs. \ref{figNoisefreeCor}-\ref{figNoisyCorMany}, the system wants to arrive at a goal area either by passing through a narrow corridor or detouring around the two blocks. In Fig. \ref{figNoisefreeCor}, when the dynamics is noise-free (by setting a small diffusion matrix), the iMDP algorithm quickly determines to follow a narrow corridor. In contrast, when the environment affects the dynamics of the system (Figs. \ref{figNoisyCor}-\ref{figNoisyCorMany}), the iMDP algorithm decides to detour to have a safer route. This experiment demonstrates the benefit of iMDP in handling process noise compared to RRT-like algorithms~\cite{Lavalle.98, Karaman.Frazzoli:IJRR10}. We emphasize that although iMDP spends slightly more time on computation per iteration, iMDP provides feedback policies rather than open-loop policies; thus, re-planning is not crucial in iMDP.

In the forth experiment, we examined the performance of the iMDP algorithm for high dimensional systems such as a manipulator with six degrees of freedom. The manipulator is modeled as a single integrator where states represents angles between segments and the horizontal line. The maximum control magnitude for all joints is 0.3. The standard deviation of noise at each joint is 0.032 rad. The manipulator is controlled to reach a goal with the final upright position in minimum time. In Fig. \ref{figManipulatorTraj}, we show a resulting trajectory after 3000 iterations computed in 15.8 seconds. In addition, we show the mean and standard deviation of the computed cost values for the initial position using 50 trials in Fig.~\ref{figManipulatorCost}. As shown in the plots, the solution converges quickly after about 1000 iterations. These results highlight the suitability of the iMDP algorithm to compute feedback policies for complex high dimensional systems in stochastic environments.

\section{Conclusions} \label{secConclusions}
We have introduced and analyzed the incremental sampling-based iMDP algorithm for stochastic optimal control. The algorithm natively handles continuous time, continuous state space as well as continuous control space. The main idea is to consistently approximate underlying continuous problems by discrete structures in an incremental manner. In particular, we incrementally build discrete MDPs by sampling and extending states in the state space. The iMDP algorithm refines the quality of anytime control policies from discrete MDPs in terms of expected costs over iterations and ensures almost sure convergence to an optimal continuous control policy. The iMDP algorithm can be implemented such that its time complexity per iteration grows as $O\big( k^{\dummyasyn} \log k \big )$ with $0 < \dummyasyn\leq 1$ leading to the total processing time $O\big( k^{1+\dummyasyn} \log k \big )$, where $k$ is the number of states in MDPs which increases linearly over iterations. Together with linear space complexity, iMDP is a practical incremental algorithm. The enabling technical ideas lie in novel methods to compute Bellman updates.

Further extension of the work is broad. In the future, we would like to study the effect of biased-sampling techniques on the performance of iMDP. The algorithm is also highly parallelizable, and efficient parallel versions of the iMDP algorithm are left for future study. Remarkably, Markov chain approximation methods are also tools to handle deterministic control and non-linear filtering problems. Thus, applications of the iMDP algorithm can be extended to classical path planning with deterministic dynamics. We emphasize that the iMDP algorithm would remove the necessity for exact point-to-point steering of RRT-like algorithms in path planning applications. 
In addition, we plan to investigate incremental sampling-based algorithms for online smoothing and estimation in the presence of sensor noise. The combination of incremental sampling-based algorithms for control and estimation will provide insights into addressing stochastic optimal control problems with imperfect state information, known as Partially Observable Markov Decision Processes (POMDPs). Although POMDPs are fundamentally more challenging than the problem that is studied in this paper, our approach differentiates itself from existing sampling-based POMDP solvers (see, e.g.,~\cite{KurHsu08,Prentice07ISRR}) with its incremental nature and computationally-efficient search. Hence, the research presented in this paper opens a new alley to handle POMDPs in our future work.

\section*{ACKNOWLEDGMENTS}
This research was supported in part by the National Science Foundation, grant CNS-1016213.
V. A. Huynh gratefully thanks the Arthur Gelb Foundation for supporting him during this work.

\bibliographystyle{IEEEtran}
\bibliography{IEEEabrv,mybib}

\appendix

\section*{Appendix} \label{appendix}

\section{Notations and Preliminaries} \label{appendix:preliminaries}
We denote $\naturals$ as a set of natural numbers and $\reals$ as a set of real numbers. A sequence on a set $X$ is a mapping from $\naturals$ to $X$, denoted as $\{x_n\}_{n=0}^{\infty}$, where $x_n \in X$ for each $n \in \naturals$. Given a metric space $X$ endowed with a metric $d$, a sequence $\{x_n\}_{n=0}^{\infty} \subset X$ is said to converge if there is a point $x \in X$, denoted as $\lim_{n \rightarrow \infty}x_n$, with the following property: For every $\epsilon > 0$, there is an integer $N$ such that $n \geq N$ implies that $d(x_n,x) < \epsilon$. On the one hand, a sequence of functions $\{ f_n \}_{n=1}^{\infty}$ in which each function $f_n$ is a mapping from $X$ to $\reals$ \textit{converges pointwise} to a function $f$ on $X$ if for every $x\in X$, the sequence of numbers $\{ f_n(x) \}_{n=0}^{\infty}$ converges to $f(x)$. On the other hand, a sequence of functions $\{ f_n \}_{n=1}^{\infty}$ \textit{converges uniformly} to a function $f$ on $X$ if the following sequence $\{ M_n \ | \ M_n= \sup_{x \in X} |f_n(x) - f(x)|\}_{n=0}^{\infty}$ converges to $0$.

Let us consider a probability space $(\Omega, \mathcal{F},\probmeasure)$ where $\Omega$ is a sample space, $\mathcal{F}$ is a $\sigma$-algebra, and $\probmeasure$ is a probability measure. A subset $A$ of $\mathcal{F}$ is called an event. The complement of an event $A$ is denoted as $A^c$. Given a sequence of events $\{ A_n \}_{n=0}^{\infty}$, we define $\limsup_{n \rightarrow \infty} A_n$ as $\cap_{n=0}^{\infty} \cup_{k=n}^{\infty}A_k$, i.e. the event that $A_n$ occurs infinitely often. In addition, the event $\liminf_{n \rightarrow \infty} A_n$ is defined as $\cup_{n=0}^{\infty} \cap_{k=n}^{\infty}A_k$. A random variable is a measurable function mapping from $\Omega$ to $\reals$. The expected value of a random variable $Y$ is defined as $\EE[Y]=\int_{\Omega}Yd\probmeasure$. A sequence of random variables $\{ Y_n \}_{n=0}^{\infty}$ \textit{converges surely} to a random variable $Y$ if $\lim_{n \rightarrow \infty}Y_n(\omega) = Y(\omega)$ for all $\omega \in \Omega$. A sequence of random variables $\{ Y_n \}_{n=0}^{\infty}$ \textit{converges almost surely} or \textit{with probability one} (w.p.1) to a random variable $Y$ if $\probmeasure( \omega \in \Omega \ | \ \lim_{n \rightarrow \infty}Y_n(\omega) = Y(\omega))=1$. Almost sure convergence of $\{Y_n \}_{n=0}^{\infty}$ to $Y$ is denoted as $Y_n \overset {a.s.}{\rightarrow} Y$. We say that a sequence of random variables $\{ Y_n \}_{n=0}^{\infty}$ \textit{converges in distribution} to a random variable $Y$ if $\lim_{n \rightarrow \infty}F_n(x)=F(x)$ for every $x \in \reals$ at which $F$ is continuous where $\{ F_n \}_{n=0}^{\infty}$ and $F$ are the associated CDFs of $\{Y_n \}_{n=0}^{\infty}$ and $Y$ srespectively. We denote this convergence as $Y_n \overset {d}{\rightarrow} Y$. Convergence in distribution is also called weak convergence. If $Y_n \overset {d}{\rightarrow} Y$, then $\lim_{n \rightarrow \infty}\EE[f(Y_n)]=\EE[f(Y)]$ for all bounded continuous functions $f$. As a corollary, when $\{Y_n\}_{n=0}^{\infty}$ converges in distribution to $0$, and $Y_n$ is bounded for all $n$, we have $\lim_{n \rightarrow \infty}\EE[Y_n]=0$ and $\lim_{n \rightarrow \infty}\EE[Y^2_n]=0$, which together imply $\lim_{n \rightarrow \infty}Var(Y_n)=0$. We say that a sequence of random variables $\{ Y_n \}_{n=0}^{\infty}$ \textit{converges in probability} to a random variable $Y$, denoted as $Y_n \overset {p}{\rightarrow} Y$, if for every $\epsilon >0$, we have $\lim_{n \rightarrow \infty}\probmeasure(|X_n-X| \geq \epsilon) =0$. For every continuous function $f(\cdot)$, if $Y_n \overset {p}{\rightarrow} Y$, then we also have $f(Y_n) \overset {p}{\rightarrow} f(Y)$. If $Y_n \overset {p}{\rightarrow} Y$ and $Z_n \overset {p}{\rightarrow} Z$, then $(Y_n,Z_n) \overset {p}{\rightarrow} (Y,Z)$ . If $|Z_n -Y_n|  \overset {p}{\rightarrow} 0$ and $Y_n \overset {d}{\rightarrow} Y$, we have $Z_n \overset {d}{\rightarrow} Y$. Finally, we say that a sequence of random variables $\{ Y_n \}_{n=0}^{\infty}$ \textit{converges in $r^{th}$ mean} to a random variable $Y$, denoted as $Y_n \overset {r}{\rightarrow} Y$, if $\EE[|X_n|^{r}] <\infty$ for all $n$, and $\lim_{n \rightarrow \infty}\EE[|X_n-X|^{r}]=0$. We have the following implications: (i) almost sure convergence or $r^{th}$ mean convergence ($r \geq 1$) implies convergence in probability, and (ii) convergence in probability implies convergence in distribution. The above results still hold for random vectors in higher dimensional spaces.

Let $f(n)$ and $g(n)$ be two functions with domain and range $\naturals$ or $\reals$. The function $f(n)$ is called $O(g(n))$ if there exists two constants $M$ and $n_0$ such that $f(n) \leq Mg(n)$ for all $n \geq n_0$. The function $f(n)$ is called $\Omega(g(n))$ if $g(n)$ is $O(f(n))$. Finally, the function $f(n)$ is called $\Theta(g(n))$ if $f(n)$ is both $O(g(n))$ and $\Omega(g(n))$.

\section{Proof of Lemma \ref{lemma:dispersion}} \label{appendix:proof:theorem:lemma:dispersion}
For each $n \in \naturals$, divide the state space ${S}$ into grid cells with side length $1/2 \gamma_r (\log |S_n| / |S_n|)^{1/d_x}$ as follows. Let $\integers$ denote the set of integers. Define the grid cell $i \in \integers^{d_x}$ as 
$$
W_{n}(i) := i \, \left(\frac{\gamma_r}{2} \frac{\log |S_n|}{|S_n|}\right)^{1/d_x} + \left[-\frac{1}{4} \, \gamma_r \left(\frac{\log |S_n|}{|S_n|}\right)^{1/d_x}, \frac{1}{4}  \, \gamma_r \left(\frac{\log |S_n|}{|S_n|}\right)^{1/d_x} \right]^{d_x},
$$
where $[-a,a]^{d_x}$ denotes the $d_x$-dimensional cube with side length $2\,a$ centered at the origin. Hence, the expression above translates the $d_x$-dimensional cube with side length $(1/2) \, \gamma_r (\log |S_n| / |S_n|)^{1/d_x}$ to the point with coordinates $i \, \frac{\gamma_r}{2} (\log n / n)^{1/d_x}$.

Let $Q_n$ denote the indices of set of all cells that lie completely inside the state space $S$, i.e., $Q_n = \{ i \in \integers^d : W_n(i) \subseteq {S}  \}$. Clearly, $Q_n$ is finite since ${S}$ is bounded. Let $\partial Q_n$ denote the set of all grid cells that intersect the boundary of ${S}$, i.e., $\partial Q_n = \{i \in \integers^d  : W_n(i) \cap \partial{S} \neq \emptyset\}$. 
We claim for all large $n$, all grid cells in $Q_n$ contain one vertex of $S_n$, and all grid cells in $\partial Q_n$ contain one vertex from $\partial S_n$. First, let us show that each cell in $Q_n$ contains at least one vertex. Given an event $A$, let $A^c$ denote its complement. Let $A_{n,k}$ denote the event that the cell $W_n(k)$, where $k \in Q_n$ contains a vertex from $S_n$, and let $A_n$ denote the event that all grid cells in $Q_n$ contain a vertex in $S_n$. Then, for all $k \in Q_n$,
$$
\PP \left( A_{n,k}^c \right) = \left( 1 - \frac{(\gamma_r/ 2)^{d_x}}{m(S)} \, \frac{\log |S_n|}{|S_n|} \right)^{|S_n|} \le \exp\left( - \big( (\gamma_r/2)^{d_x}/m(S) \big) \, \log |S_n| \right) = {|S_n|}^{-(\gamma_r/2)^{d_x}/m(S)},
$$
where $m(S)$ denotes Lebesgue measure assigned to ${S}$.
Then, 
$$
\PP(A_n^c) = \PP\left( \left( \bigcap\nolimits_{k \in Q_n} A_{n,k} \right)^c \right) = \PP\left( \bigcup\nolimits_{k \in Q_n} A_{n,k}^c \right) \le \sum\nolimits_{k \in Q_n}\PP\left( A_{n,k}^c \right) = \vert Q_n \vert \,  |S_n|^{-(\gamma_r/2)^{d_x}/m(S)}, 
$$
where the first inequality follows from the union bound and $\vert Q_n \vert$ denotes the cardinality of the set $Q_n$. By calculating the maximum number of cubes that can fit into ${S}$, we can bound $\vert Q_n \vert$:
$$
\vert Q_n\vert \le \frac{m(S)}{(\gamma_r/2)^{d_x} \, \frac{\log |S_n|}{|S_n|}} = \frac{m(S)}{(\gamma_r/2)^{d_x}} \, \frac{|S_n|}{\log |S_n|}.
$$
Note that by construction, we have $|S_n|=\Theta(n)$. Thus,
\begin{align*}
\PP \left( A_{n}^c \right) &\le \frac{m(S)}{(\gamma_r/2)^{d_x}} \, \frac{|S_n|}{\log |S_n|} \, |S_n|^{-(\gamma_r/2)^{d_x}/m(S)} = \frac{m(S)}{(\gamma_r/2)^{d_x}} \, \frac{1}{\log |S_n|} \, |S_n|^{1 - (\gamma_r/2)^{d_x} / m(S)} \\
&\le \frac{m(S)}{(\gamma_r/2)^{d_x}} \,|S_n|^{1 - (\gamma_r/2)^{d_x} / m(S)},
\end{align*}
which is summable for all $\gamma_r > 2 \, (2 \, m(S))^{1/{d_x}}$. Hence, by the Borel-Cantelli lemma, the probability that $A_n^c$ occurs infinitely often is zero, which implies that the probability that $A_n$ occurs for all large $n$ is one, i.e.,
$
\PP (\liminf_{n \rightarrow \infty} A_n) = 1.
$

Similarly, each grid cell in $\partial Q_n$ can be shown to contain at least one vertex from $\partial S_n$ for all large $n$, with probability one. This implies each grid cell in both sets $Q_n$ and $\partial Q_n$ contain one vertex of $S_n$ and $\partial S_n$, respectively, for all large $n$, with probability one. Hence the following event happens with probability one:
$$
\dispersion_n=\max_{z \in S_n}\min_{z' \in S_n}||z'-z||_2 = O((\log{|S_n|}/|S_n|)^{1/d_x}).
$$
\qed

\section{Proof of Lemma \ref{lemma:local_consistency}} \label{appendix:proof:theorem:lemma:local_consistency}
We show that each state that is added to the approximating MDPs is updated infinitely often. That is, for any $z \in S_n$, the set of all iterations in which the procedure ${\tt Update}$ is applied on $z$ is unbounded. Indeed, let us denote $ \dispersion_n(z)=\min_{z' \in S_n}||z'-z||_2.$ From Lemma~\ref{lemma:dispersion}, $\lim_{n \rightarrow \infty}\dispersion_n(z)=0$ happens almost surely. Therefore, with  probability one, there are infinitely many $n$ such that $\dispersion_n(z) < \dispersion_{n-1}(z)$ . In other words, with probability one, we can find infinitely many $z_{new}$ at Line~\ref{line:main:compute_update} of Algorithm~\ref{algorithm:main} such that $z$ is updated.
For those $n$, the holding time at $z$ is recomputed as $\Delta t_n(z)=\gamma_t \left(\frac{\log |S_n|}{|S_n|}\right)^{\dummyasyn \dummyrate\rho/d_x}$ at Line~\ref{line:holdingtime} of Algorithm~\ref{algorithm:update}. Thus, the following event happens with probability one:
$$
\lim_{n \to \infty} \Delta t_n(z) = 0,
$$
which satisfies the first condition of local consistency in Eq.~\ref{eqn:lc1}.

The other conditions of local consistency in Eqs.~\ref{eqn:lc2}-\ref{eqn:lc4} are satisfied immediately by the way that the transition probabilities are computed (see the description of the ${\tt ComputeTranProb}$ procedure given in Section~\ref{section:algorithm:primitiveprocedures}). 
Hence, the MDP sequence $\{ {\cal M}_n \}_{n=0}^{\infty}$ and holding times $\{ \Delta t_n \}_{n=0}^{\infty}$ are locally consistent for large n with probability one.
\qed

\section{Proof of Theorem \ref{theoremAlmostSureConvergence}} \label{appendix:proof:theorem:almostsure_converge}
To highlight the idea of the entire proof, we first prove the convergence under synchronous value iterations before presenting the convergence under asynchronous value iterations. As we will see, the shrinking rate of holding times plays a crucial role in the convergence proof. The outline of the proof is as follows.
\begin{itemize}
\item[S1:] Convergence under synchronous value iterations: In Algorithm \ref{algorithm:main}, we take $L_n \geq 1$ and $K_n=|S_n|-1$. In other words, in each iteration, we perform synchronous value iterations. Moreover, we assume that we are able to solve the Bellman equation (Eq. \ref{eqn:Bellman}) exactly. We show that $J_n$ converges uniformly to $J^*$ almost surely in this setting.
\item[S2:] Convergence under asynchronous value iterations: When $ K_n=\Theta(|S_n|^{\dummyasyn}) < |S_n|$, we only update a subset of $S_n$ in each of $L_n$ passes. We show that $J_n$ still converges uniformly to $J^*$ almost surely in this new setting.
\end{itemize}

In the following discussion and next sections, we need to compare functions on different domains $S_n$. To ease the discussion and simplify the notation, we adopt the following interpolation convention. Given $X \subset Y$ and $J: X \rightarrow \reals$, we interpolate $J$ to $\overline{J}$ on the entire domain $Y$ via nearest neighbor value: 
$$
\forall y \in Y : \ \ \ \overline{J}(y) = J(z) \text{ where } z=\argmin_{z' \in X}||z'-y||.
$$
To compare $J: X \rightarrow \reals$ and $J': Y \rightarrow \reals$ where $X,Y \subset S$, we define the sup-norm:
$$
||J-J'||_{\infty} =  ||\overline{J}-\overline{J'}||_{\infty},
$$
where $\overline{J}$ and $\overline{J'}$ are interpolations of $J$ and $J'$ from the domains $X$ and $Y$ to the entire domain $S$ respectively.
In particular, given $J_n: S_n \rightarrow \reals$, and $J:S \rightarrow \reals$, then $||J_n-J||_{S_n} \leq ||J_n-J||_{\infty}$. Thus, if $||J_n-J||_{\infty}$ approaches $0$ when $n$ approaches $\infty$, so does $||J_n-J||_{S_n}$. Hence, we will work with the (new) sup-norm $||\cdot||_{\infty}$ instead of $||\cdot||_{S_n}$ in the proofs of Theorems~\ref{theoremAlmostSureConvergence}-\ref{theoremAlmostSureConvergenceSampleControl}. The triangle inequality also holds for any functions $J, J', J''$ defined on subsets of $S$ with respect to the above sup-norm: 
$$
||J-J'||_{\infty} \leq ||J-J''||_{\infty} + ||J''-J'||_{\infty}.
$$

Let $B(X)$ denote a set of all real-valued bounded functions over a domain $X$. For $S_n \subset S_{n'}$ when $n < n'$, a function $J$ in $B(S_n)$ also belongs to $B(S_{n'})$, meaning that we can interpolate $J$ on $S_n$ to a function $J'$ on $S_{n'}$. In particular, we say that $J$ in $B(S_n)$ also belongs to $B(S)$.

Lastly, due to random sampling, $S_n$ is a random set, and therefore functions $J_n$ and $J^*_n$ defined on $S_n$ are random variables. In the following discussion, inequalities hold surely without further explanation when it is clear from the context, and inequalities hold almost surely if they are followed by ``w.p.1''.

\subsection*{S1: Convergence under synchronous value iterations}
In this step, we first set $L_n \geq 1$ and $K_n=|S_n|-1$ in Algorithm \ref{algorithm:main}. Thus, for all $z \in S_n$, the holding time $\Delta t_n (z)$ equals $\gamma_t \left(\frac{\log |S_n|}{|S_n|}\right)^{\dummyasyn\dummyrate \rho/d_x}$ and is denoted as $\Delta t_n$. We consider the MDP ${\cal M}_n= (S_n, U,P_n,G_n,H_n)$ at $n^{th}$ iteration and define the following operator $T_n:B(S_n) \rightarrow B(S_n)$ that transforms every $J \in B(S_n)$ after a Bellman update as:
\begin{align}
T_nJ(z)= \min_{v \in U}\{ G_n(z,v) + \alpha^{\Delta t_n} \EE_{P_n}\left[ J(y) |z,v\right]\}, \ \ \forall z \in S_n, \label{eqn:operator}
\end{align}
assuming that we can solve the minimization on the RHS of Eq.~\ref{eqn:operator} exactly. For each $k \geq 2$, operators $T^k_n$ are defined recursively as $T^k_n=T_nT^{k-1}_n$ and $T^1_n=T_n$. When we apply $T_n$ on $J \in B(S_k)$ where $k<n$, $J$ is interpolated to $S_n$ before applying $T_n$. Thus, in Algorithms \ref{algorithm:main}-\ref{algorithm:update}, we implement the next update
$$
J_{n}= T^{L_n}_nJ_{n-1}.
$$
From \cite{Bertsekas:2001:DPO:516163}, we have the following results: $J^*_n =T_nJ^*_n$, and $T_n$ is a contraction mapping. For any $J$ and $J'$ in $B(S_n)$, the following inequality happens surely:
\begin{align*}
||T_nJ - T_nJ'||_{\infty} \leq \alpha^{\Delta t_n} ||J-J'||_{\infty}.
\end{align*}
Combining the above results:
\begin{align*}
||J^*_n - J_n||_{\infty} &= ||T^{L_n}_nJ^*_n - T^{L_n}_nJ_{n-1}||_{\infty} \leq  \alpha^{L_n \Delta t_n} ||J^*_n-J_{n-1}||_{\infty} \\
                      &\leq \alpha^{\Delta t_n} ( ||J^*_n-J^*_{n-1}||_{\infty} +||J^*_{n-1}-J_{n-1}||_{\infty}),
\end{align*}
where the second inequality follows from the triangle inequality, and $L_n \geq 1, \alpha \in (0,1)$.

Thus, by iterating over $n$, for any $N \geq 1$ and $n > N$, we have:
\begin{align}
||J^*_n - J_n||_{\infty} \leq A_n + \alpha^{\Delta t_n+\Delta t_{n-1}...+\Delta t_{N+1}}||J^*_N - J_N||_{\infty}, \label{eqn:C1}
\end{align}
where $A_n$ are defined recursively:
\begin{align}
A_n &= \alpha^{\Delta t_n}(||J^*_n-J^*_{n-1}||_{\infty}  +  A_{n-1}),  \ \ \forall n > N+1, \label{eqn:An1} \\ 
A_{N+1} &= \alpha^{\Delta t_{N+1}}||J^*_{N+1}-J^*_{N}||_{\infty}. \label{eqn:An2}
\end{align}
Note that for any $N \geq 1$:
$$
\lim_{n \rightarrow \infty} \Delta t_n+\Delta t_{n-1}...+\Delta t_{N+1} = \infty,
$$
due to the choice of holding times $\Delta t_n$ in the procedure {\tt ComputeHoldingTime}. Therefore,
$$
\lim_{n \rightarrow \infty} \alpha^{\Delta t_n+...+\Delta t_{N+1}}||J^*_N - J_N||_{\infty} =0.
$$
By Theorem \ref{theorem:JoptnJoptas}, the following event happens with probability 1 (w.p.1):
$$
\lim_{n \rightarrow \infty} ||J^*_n-J^*||_{\infty} = 0,
$$
hence,
$$
\lim_{n \rightarrow \infty} ||J^*_n-J^*_{n-1}||_{\infty} = 0 \text{ w.p.1. }
$$
Thus, for any fixed $\epsilon >0$, we can choose $N$ large enough such that:
\begin{align}
&||J^*_n-J^*_{n-1}||^{1-\dummyrate}_{\infty} < \epsilon \text{ w.p.1 for all } n > N, \text{ and }\\
&\alpha^{\Delta t_n+...+\Delta t_{N+1}}||J^*_N - J_N||_{\infty}  < \epsilon \text{ surely}, \label{eqn:C2}
\end{align}
where $\dummyrate \in (0,1)$ is the constant defined in the procedure {\tt ComputeHoldingTime}.

Now, for all $n > N$, we rearrange Eqs.\ref{eqn:An1}-\ref{eqn:An2} to have 
$$
A_n \leq \epsilon B_n \text{ w.p.1, }
$$  
where
\begin{align*}
B_n &= \alpha^{\Delta t_n}(||J^*_n-J^*_{n-1}||^{\dummyrate}_{\infty}  +  B_{n-1}),  \ \ \forall n > N+1, \\
B_{N+1} &= \alpha^{\Delta t_{N+1}}||J^*_{N+1}-J^*_{N}||^{\dummyrate}_{\infty}.
\end{align*}
We can see that for $n > N+1$:
\begin{align}
&B_n = \alpha^{\Delta t_n}(||J^*_n-J^*_{n-1}||^{\dummyrate}_{\infty}  +  B_{n-1}) < \ \ \epsilon^{\dummyrate/(1-\dummyrate)}+ B_{n-1}  \label{eqn:BB1} \text{ w.p.1,} \\
&B_{N+1} = \alpha^{\Delta t_{N+1}}||J^*_{N+1}-J^*_{N}||^{\dummyrate}_{\infty} < \ \ \epsilon^{\dummyrate/(1-\dummyrate)}  \label{eqn:B3} \text{ w.p.1.}
\end{align}
We now prove that almost surely, $B_n$ is bounded for all $n \geq N$. Indeed, we derive the conditions so that $B_{n-1} < B_n$ or $B_{n-1} \geq B_n$ as follows:
\begin{align*}
&B_{n-1} < B_{n} \\
\Leftrightarrow \ \ &B_{n-1}  < \alpha^{\Delta t_n}(||J^*_n-J^*_{n-1}||^{\dummyrate}_{\infty}  +  B_{n-1})\\
\Leftrightarrow \ \ &B_{n-1} < \frac{\alpha^{\Delta t_n}||J^*_n-J^*_{n-1}||^{\dummyrate}_{\infty} }{1 - \alpha^{\Delta t_n}} \\
\Rightarrow \ \ &B_{n-1} < {\cal K} \frac{\alpha^{\gamma_t \left(\frac{\log |S_n|}{|S_n|}\right)^{\dummyasyn \dummyrate\rho/d_x}} \left( \frac{\log |S_n|}{|S_n|} \right)^{\dummyrate \rho/d_x} }{1 - \alpha^{\gamma_t \left(\frac{\log |S_n|}{|S_n|}\right)^{\dummyasyn \dummyrate\rho/d_x}}} \text{ w.p.1.}
\end{align*}
The last inequality is due to Theorem \ref{theorem:JoptnJoptas} and $|S_n|=\Theta(n)$, $|S_{n-1}|=\Theta(n-1)$:
$$
||J^*_n-J^*_{n-1}||_{\infty}= O((\log{|S_{n-1}|}/|S_{n-1}|)^{\rho/d_x}) < {\cal K} \left(\frac{\log |S_n|}{|S_n|} \right)^{\rho/d_x} \text{ w.p.1, }
$$
for large $n$ where ${\cal K}$ is some finite constant. Let $\beta=\alpha^{\gamma_t} \in (0,1)$. For large $n, \ \frac{\log |S_n|}{|S_n|}$ are in $(0,1)$ and $\dummyasyn \in (0,1]$. Let us define
$$
x_n=\left( \frac{\log |S_n|}{|S_n|} \right)^{\dummyasyn \dummyrate\rho/d_x},\text{ and} \ \ \
y_n=\left( \frac{\log |S_n|}{|S_n|} \right)^{\dummyrate\rho/d_x}.
$$ 
Then, $x_n \geq y_n >0$. The above condition is simplified to
$$
B_{n-1} < {\cal K} \frac{\beta^{x_n}y_n}{1-\beta^{x_n}} \leq {\cal K} \frac{\beta^{x_n}x_n}{1-\beta^{x_n}}, \text{ w.p.1.}
$$
Consider the function $r:[0,\infty) \rightarrow \reals$ such that $r(x)=\frac{\beta^{x}x}{1-\beta^{x}}$, we can verify that $r(x)$ is non-increasing and is bounded by $r(0)=-1/\log(\beta)$. Therefore: 
\begin{align}
&B_{n-1} < B_{n} \ \ \Rightarrow \ \ B_{n-1} < -\frac{{\cal K}}{\log(\beta)}=-\frac{{\cal K}}{\gamma_t \log(\alpha)} \label{eqn:B1} \ \ \text{ w.p.1.} 
\end{align}
Or conversely, 
\begin{align}
& B_{n-1} \geq -\frac{{\cal K}}{\gamma_t \log(\alpha)} \ \ \text{ w.p.1} \ \ \Rightarrow \ \ B_{n-1} \geq B_{n} \ \ \text{ w.p.1}. \label{eqn:BB3}
\end{align}
\begin{figure*}
\begin{center}
  \includegraphics[scale=0.5]{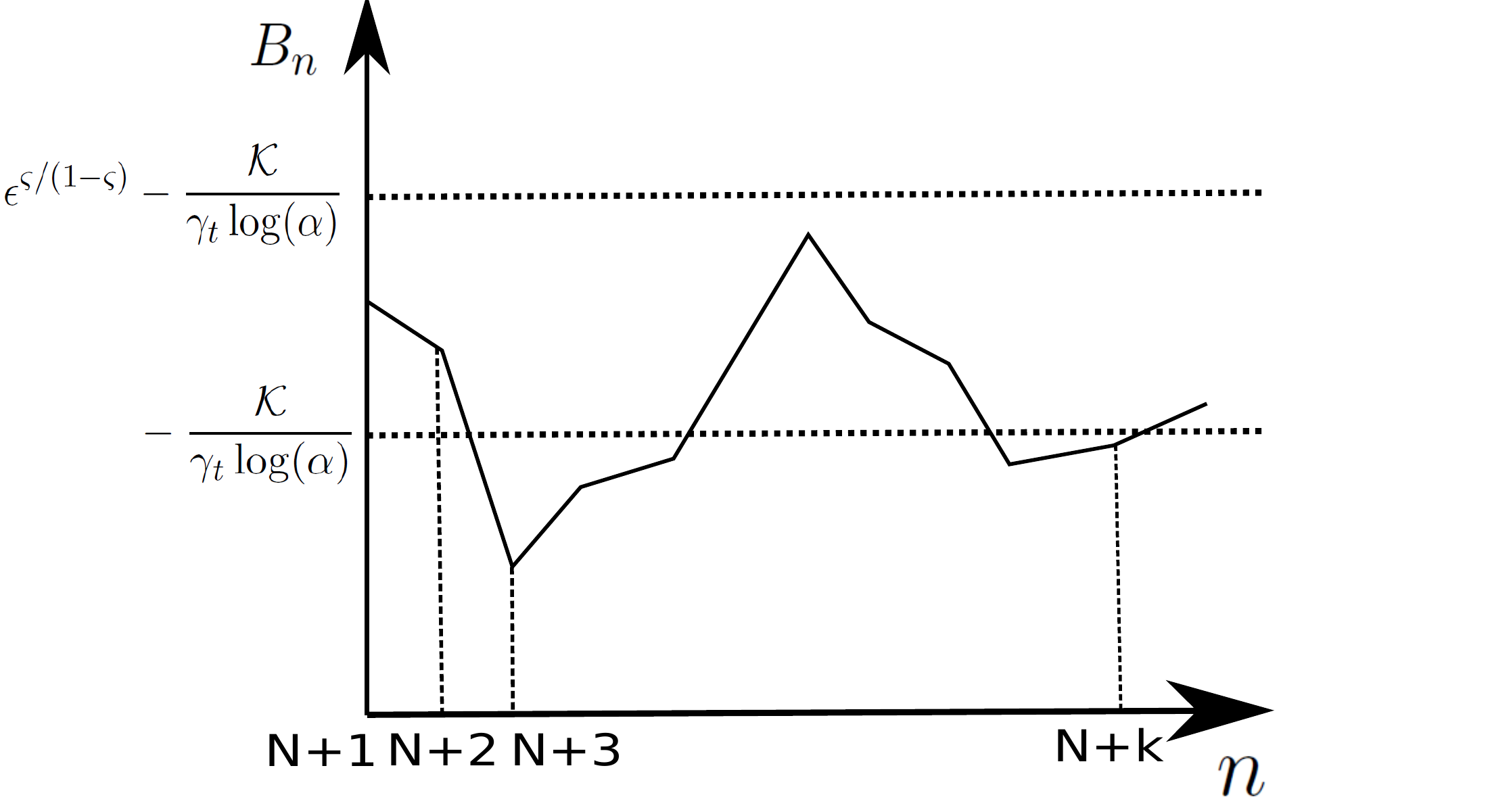}
\end{center}
  \vspace{-0.1in}
  \caption{A realization of the random sequence $B_n$. We have $B_{N+1}$ less than $\epsilon^{\dummyrate/(1-\dummyrate)}$ w.p.1. For $n$ larger than $N+1$, when $B_{n-1} \geq -\frac{{\cal K}}{\gamma_t \log(\alpha)}$ w.p.1, the sequence is non-increasing w.p.1, i.e. $B_{n-1} \geq B_n$ w.p.1. Conversely, when the sequence is increasing, i.e. $B_{n-1} < B_n$, we have $B_{n-1} < -\frac{{\cal K}}{\gamma_t \log(\alpha)}$ w.p.1, and the increment is less than $\epsilon^{\dummyrate/(1-\dummyrate)}$. Hence, the random sequence $B_n$ is bounded by $\epsilon^{\dummyrate/(1-\dummyrate)}- \frac{{\cal K}}{\gamma_t \log(\alpha)}$ w.p.1.}
  \label{BnDraw}
\end{figure*}
The above discussion characterizes the random sequence $B_n$. In particular, Fig.~\ref{BnDraw} shows a possible realization of the random sequence $B_n$ for $n > N$. As shown visually in this plot, $B_{N+1}$ is less than $\epsilon^{\dummyrate/(1-\dummyrate)}$ w.p.1 and thus is less than $\epsilon^{\dummyrate/(1-\dummyrate)} - \frac{{\cal K}}{\gamma_t \log(\alpha)}$ w.p.1. For $n > N+1$, assume that we have already shown that $B_{n-1}$ is bounded from above by $\epsilon^{\dummyrate/(1-\dummyrate)} - \frac{{\cal K}}{\gamma_t \log(\alpha)}$ w.p.1. When $B_{n-1} \geq -\frac{{\cal K}}{\gamma_t \log(\alpha)}$ w.p.1, the sequence is non-increasing w.p.1. Conversely, when the sequence is increasing, i.e. $B_{n-1} < B_n$, we assert that $B_{n-1} < -\frac{{\cal K}}{\gamma_t \log(\alpha)}$ w.p.1 due to Eq.~\ref{eqn:B1}, and the increment is less than $\epsilon^{\dummyrate/(1-\dummyrate)}$ due to Eq.~\ref{eqn:BB1}. In both cases, we conclude that $B_{n}$ is also bounded by $\epsilon^{\dummyrate/(1-\dummyrate)} - \frac{{\cal K}}{\gamma_t \log(\alpha)}$ w.p.1. Hence, from Eqs.~\ref{eqn:BB1}-\ref{eqn:BB3}, we infer that $B_n$ is bounded w.p.1 for all $n >N$:
$$
B_n < \epsilon^{\dummyrate/(1-\dummyrate)} -\frac{{\cal K}}{\gamma_t \log(\alpha)} \text{ w.p.1.} 
$$
Thus, for all $n>N$:
\begin{align}
A_n \leq \epsilon B_n < \epsilon \Big(\epsilon^{\dummyrate/(1-\dummyrate)} -\frac{{\cal K}}{\gamma_t \log(\alpha)} \Big) \label{eqn:C3} \text{ w.p.1.} 
\end{align}

Combining Eqs.~\ref{eqn:C1},\ref{eqn:C2}, and \ref{eqn:C3}, we conclude that for any $\epsilon >0$, there exists $N \geq 1$ such that for all $n > N$, we have
$$
||J_n^* - J_n||_{\infty} < \epsilon\Big( \epsilon^{\dummyrate/(1-\dummyrate)} -\frac{{\cal K}}{\gamma_t \log(\alpha)}+ 1 \Big) \text{ w.p.1}.
$$
Therefore,
$$
\lim_{n \rightarrow \infty} ||J^*_n - J_n||_{\infty} =0 \text{ w.p.1}.
$$
Combining with
$$
\lim_{n \rightarrow \infty} ||J^*_n - J^*||_{\infty} =0 \text{ w.p.1},
$$
we obtain
$$
\lim_{n \rightarrow \infty} ||J_n - J^*||_{\infty} =0 \text{ w.p.1}.
$$

In the above analysis, the shrinking rate $\left(\frac{\log |S_n|}{|S_n|}\right)^{\dummyasyn \dummyrate\rho/d_x}$ of holding times plays an important role to construct an upper bound of the sequence $B_n$. This rate must be slower than the convergence rate $\left(\frac{\log |S_n|}{|S_n|}\right)^{\rho/d_x}$ of $J^*_n$ to $J^*$ so that the function $r(x)$ is bounded, enabling the convergence of cost value functions $J_n$ to the optimal cost-to-go $J^*$. Remarkably, we have accomplished this convergence by carefully selecting the range $(0,1)$ of the parameter $\dummyrate$. The role of the parameter $\dummyasyn$ in this convergence will be clear in Step S2. Lastly, we note that if we are able to obtain a faster convergence rate of $J^*_n$ to $J^*$, we can have faster shrinking rate for holding times.

\subsection*{S2: Convergence under asynchronous value iterations}
When $1 \leq L_n$ and $ K_n=\Theta(|S_n|^{\dummyasyn}) < |S_n|$, we first claim the following result:
\begin{lemma} \label{lemma:sumtozero}
Consider any increasing sequence $\{ n_k \}_{k=0}^{\infty}$ as a subset of $\naturals$ such that $n_0=0$ and $k\leq |S_{n_k}| \leq k^{1/\dummyasyn}$. For $J \in B(S)$, we define:
\begin{align*}
A\big( \{n_j\}_{j=0}^k\big) =  \alpha^{\Delta t_{n_k} + \Delta t_{n_{k-1}} +...+\Delta t_{n_1}}||J^*_{n_1}-J||_{\infty} & + \alpha^{\Delta t_{n_k} + \Delta t_{n_{k-1}}+...+\Delta t_{n_2}}||J^*_{n_2}-J^*_{n_1}||_{\infty} \\
& + ... + \alpha^{\Delta t_{n_k}}||J^*_{n_k}-J^*_{n_{k-1}}||_{\infty}.
\end{align*}
The following event happens with probability one:
$$
\lim_{k \rightarrow \infty} A\big( \{n_j\}_{j=0}^k\big) =0. 
$$
\end{lemma}
\begin{proof}
We rewrite
$
A\big( \{n_j\}_{j=0}^k\big) =  A_{n_k}
$
where
$A_{n_k}$ are defined recursively:
\begin{align}
A_{n_k} &= \alpha^{\Delta t_{n_k}}(||J^*_{n_k}-J^*_{n_{k-1}}||_{\infty}  +  A_{n_{k-1}}),  \ \ \forall k > K,\\ 
A_{n_{K}} &= A\big( \{n_j\}_{j=0}^{K}\big), \ \ \forall K \geq 1. \label{eqn:An2unused}
\end{align}
We note that 
\begin{align*}
&\Delta t_{n_k}+ \Delta t_{n_{k-1}}+...+\Delta t_{n_{K}} \\
&= \gamma_t \left(\frac{\log |S_{n_{k}}|}{|S_{n_{k}}|}\right)^{\dummyasyn \dummyrate\rho/d_x} + \gamma_t \left(\frac{\log |S_{n_{k-1}}|}{|S_{n_{k-1}}|}\right)^{\dummyasyn \dummyrate\rho/d_x} +...+\gamma_t \left(\frac{\log |S_{n_{K}}|}{|S_{n_{K}}|}\right)^{\dummyasyn \dummyrate\rho/d_x} \\
&\geq \gamma_t \left(\frac{1}{|S_{n_{k}}|}\right)^{\dummyasyn \dummyrate\rho/d_x} + \gamma_t \left(\frac{1}{|S_{n_{k-1}}|}\right)^{\dummyasyn \dummyrate\rho/d_x} +...+\gamma_t \left(\frac{1}{|S_{n_{K}}|}\right)^{\dummyasyn \dummyrate\rho/d_x} \\
&\geq \gamma_t \frac{1}{k^{\dummyrate\rho/d_x}} + \gamma_t \frac{1}{(k-1)^{\dummyrate\rho/d_x}} +...+\gamma_t \frac{1}{(K)^{\dummyrate\rho/d_x}} \geq \gamma_t (\frac{1}{k}+\frac{1}{k-1}+...+\frac{1}{K}),
\end{align*}
where the second inequality uses the given fact that $|S_{n_{k}}| \leq k^{1/\dummyasyn}$. Therefore, for any $K \geq 1$:
$$
\lim_{k \rightarrow \infty} \alpha^{\Delta t_{n_k}+\Delta t_{n_{k-1}}...+\Delta t_{n_{K}}}=0.
$$
We choose a constant $\varrho >1$ such that $\varrho \dummyrate < 1$. For any fixed $\epsilon >0$, we can choose $K$ large enough such that:
\begin{align}
&||J^*_{n_k}-J^*_{n_{k-1}}||^{1-\varrho\dummyrate}_{\infty} < \epsilon \text{ w.p.1 for all } k > K.
\end{align}
For all $k > K$, we can write
$$
A_{n_k} \leq \epsilon B_{n_k} + \alpha^{\Delta t_{n_k}+...+\Delta t_{n_{K+1}}}{A\big( \{n_j\}_{j=0}^{K}\big)}.
$$  
where
\begin{align*}
B_{n_k} &= \alpha^{\Delta t_{n_k}}(||J^*_{n_k}-J^*_{n_{k-1}}||^{\varrho \dummyrate}_{\infty}  +  B_{n_{k-1}}),  \ \ \forall k > K, \\
B_{n_{K}} &= 0.
\end{align*}
Furthermore, we can choose $K'$ sufficiently large such that $K'\geq K$ and for all $k > K'$:
$$
\alpha^{\Delta t_{n_k}+...+\Delta t_{n_{K+1}}}{A\big( \{n_j\}_{j=0}^{K}\big)} \leq \epsilon.
$$
We obtain:
$$
A_{n_k} \leq \epsilon B_{n_k} + \epsilon , \ \ \forall k > K' \geq K \geq 1.  
$$
We can also see that for $k > K$:
\begin{align}
B_{n_k} = \alpha^{\Delta t_{n_k}}(||J^*_{n_k}-J^*_{n_{k-1}}||^{\varrho \dummyrate}_{\infty}  +  B_{n_{k-1}}) < \ \ \epsilon^{\varrho \dummyrate/(1-\varrho \dummyrate)}+ B_{n_{k-1}}  \label{eqn:B2} \text{ w.p.1.}
\end{align}
Similar to Step S1, we characterize the random sequence $B_{n_{k}}$ as follows:
\begin{align*}
&B_{n_{k-1}} < B_{n_k} \\
\Leftrightarrow \ \ &B_{n_{k-1}} < \frac{\alpha^{\Delta t_{n_k}}||J^*_{n_k}-J^*_{n_{k-1}}||^{\varrho \dummyrate}_{\infty} }{1 - \alpha^{\Delta t_{n_k}}} \\
\Rightarrow \ \ &B_{n_{k-1}} < {\cal K} \frac{\alpha^{\gamma_t \left(\frac{\log |S_{n_k}|}{|S_{n_k}|}\right)^{\dummyasyn \dummyrate\rho/d_x}} \left( \frac{\log |S_{n_{k-1}}|}{|S_{n_{k-1}}|} \right)^{\varrho \dummyrate \rho/d_x} }{1 - \alpha^{\gamma_t \left(\frac{\log |S_{n_k}|}{|S_{n_k}|}\right)^{\dummyasyn \dummyrate\rho/d_x}}} \text{ w.p.1.}
\end{align*}
Let $\beta=\alpha^{\gamma_t} \in (0,1)$. We define:
$$
x_k=\left( \frac{\log |S_{n_k}|}{|S_{n_k}|} \right)^{\dummyasyn \dummyrate\rho/d_x}, \text{ and } \ \ \ y_k=\left( \frac{\log |S_{n_{k-1}}|}{|S_{n_{k-1}}|} \right)^{\varrho \dummyrate\rho/d_x}.
$$ 
We note that $\frac{\log{x}}{x}$ is a decreasing function for positive $x$. Since $|S_{n_{k-1}}| \geq k-1$ and $|S_{n_{k}}| \leq k^{1/\dummyasyn}$, we have the following inequalities:
$$
x_k \geq \left( \frac{(\frac{\log{k}}{\dummyasyn})^{\dummyasyn}}{k} \right)^{\dummyrate \rho /d_x}, \ \ \ y_k \leq \left( \frac{(\log (k-1))^{\varrho}}{(k-1)^{\varrho}} \right)^{\dummyrate \rho /d_x}.
$$
Since $\dummyasyn \in (0,1]$ and $\varrho >1$, we can find a finite constant ${\cal K}_1$ such that $y_k  < {\cal K}_1 x_k$ for large $k$. Thus, the above condition leads to
$$
B_{n_{k-1}} < {\cal K} \frac{\beta^{x_k}y_k}{1-\beta^{x_k}} < {\cal K}{\cal K}_1 \frac{\beta^{x_k}x_k}{1-\beta^{x_k}}, \text{ w.p.1.}
$$
Therefore: 
\begin{align*}
B_{n_{k-1}} < B_{n_{k}} \ \ \Rightarrow \ \ B_{n_{k-1}} < -\frac{{\cal K}{\cal K}_1}{\log(\beta)}=-\frac{{\cal K}{\cal K}_1}{\gamma_t \log(\alpha)}  \ \ \text{ w.p.1.}
\end{align*}
Or conversely,
\begin{align*}
B_{n_{k-1}} \geq -\frac{{\cal K}{\cal K}_1}{\gamma_t \log(\alpha)} \ \ \text{ w.p.1} \ \ \Rightarrow \ \ B_{n-1} \geq B_{n} \ \ \text{ w.p.1}.
\end{align*}
Arguing similarly to Step S1, we infer that for all $k >K' \geq K \geq 1$:
$$
B_{n_k} < \epsilon^{\varrho \dummyrate/(1-\varrho \dummyrate)} -\frac{{\cal K}{\cal K}_1}{\gamma_t \log(\alpha)} \text{ w.p.1.} 
$$
Thus, for any $\epsilon >0$, we can find $K' \geq 1$ such that for all $k>K'$:
\begin{align*}
A_{n_k} \leq \epsilon B_{n_k} + \epsilon < \epsilon \Big( \epsilon^{\varrho \dummyrate/(1-\varrho \dummyrate)} -\frac{{\cal K}{\cal K}_1}{\gamma_t \log(\alpha)} +1\Big) \text{ w.p.1.} 
\end{align*}
We conclude that
$$
\lim_{k \rightarrow \infty} A\big( \{n_j\}_{j=0}^k\big) =0. \text{ w.p.1}.
$$
\qed
\end{proof}

Returning to the main proof, we use the tilde notation to indicate asynchronous operations to differentiate with our synchronous operations in Step S1. We will also assume that $L_n=1$ for all $n$ to simplify the following notations. The proof for general $L_n \geq 1$ is exactly the same. 
We define the following (asynchronous) mappings $\widetilde{T}_n: B(S_n) \rightarrow B(S_n)$ as the restricted mappings of $T_n$ on $D_n$, a non-empty random subset of $S_n$, such that for all $J \in B(S_n)$:
\begin{align}
\widetilde{T}_nJ(z) &=  \min_{ v \in U} \Big \lbrace G_n(z,v) + \alpha^{\Delta t_n}{\EE}_{P_n} \big [J(y)|z,v\big] \Big \rbrace, \ \ \forall z \in D_n \subset S_n,\\
\widetilde{T}_nJ(z) &= J(z), \ \ \forall z \in S_n \backslash D_n.
\end{align}
We require that 
\begin{align}
 \cap_{n=1}^{\infty}\cup_{k=n}^{\infty}D_k=S.
\end{align}
In other words, every state in $S$ are sampled infinitely often. We can see that in Algorithm~\ref{algorithm:main}, if the set $Z_\mathrm{update}$ is assigned  to $D_n$ in every iteration (Line~\ref{line:main:compute_update}), the sequence $\{ D_n \}_{n=1}^{\infty}$ has the above property, and $|D_n|= \Theta(|S_n|^{\dummyasyn}) < |S_n|$.

Starting from any $\widetilde{J}_0 \in B(S_0)$, we perform the following asynchronous iteration 
\begin{align}
\widetilde{J}_{n+1}=\widetilde{T}_{n+1}\widetilde{J}_n,  \ \ \forall n \geq 0. 
\end{align}

Consider the following sequence $\{m_k\}_{k=0}^{\infty}$ such that $m_0=0$ and for all $k\geq 0$, from $m_{k}$ to $m_{k+1}-1$, all states in $S_{m_{k+1}-1}$ are chosen to be updated at least once, and a subset of states in $S_{m_{k+1}-1}$ is chosen to be updated exactly once. We observe that as the size of $S_n$ increases linearly with $n$, if we schedule states in $D_n \subset S_n$ to be updated in a round-robin manner, we have $k \leq S_{m_k} \leq k^{1/ \dummyasyn}$. When $D_n$ is chosen as shown in Algorithm~\ref{algorithm:main}, with high probability, $k \leq S_{m_k} \leq k^{1/ \dummyasyn}$. However, we will assume that the event $k \leq S_{m_k} \leq k^{1/ \dummyasyn}$ happens surely because we can always schedule a fraction of $D_n$ to be updated in a round-robin manner. 

We define $W_{n}$ as the set of increasing sub-sequences of the sequence $\{0,1,...,n\}$ such that each sub-sequence contains $\{m_j\}_{j=0}^{k}$ where $m_k \leq n < m_{k+1}$:
$$
W_n = \Big \{ \{i_j\}_{j=0}^T \ \big | \ \{ m_j \}_{j=0}^{k} \subset  \{i_j\}_{j=0}^T \subset \{0,1,...,n\} \wedge T \geq 2  \wedge m_k \leq n < m_{k+1} \Big \}.
$$
Clearly, if $\{i_j\}_{j=0}^T \in  W_n$, we have $i_0 =0$. For each $\{i_j\}_{j=0}^T \in W_n$, we define
\begin{align*}
A\big( \{i_j\}_{j=0}^T\big)=  \alpha^{\Delta t_{i_T} + \Delta t_{i_{T-1}} +...+\Delta t_{i_1}}||J^*_{i_1}-\widetilde{J}_0||_{\infty} & + \alpha^{\Delta t_{i_T} + \Delta t_{i_{T-1}}+...+\Delta t_{i_2}}||J^*_{i_2}-J^*_{i_1}||_{\infty} \\
& + ... + \alpha^{\Delta t_{i_T}}||J^*_{i_T}-J^*_{i_{T-1}}||_{\infty}.
\end{align*}
We will prove by induction that
\begin{align}
\forall z \in D_n \Rightarrow |\widetilde{J}_{n}(z) - J^*_{n}(z)| \leq \max_{\{i_j\}_{j=0}^T \in W_n}  A\big(\{i_j\}_{j=0}^T\big). \label{eqn:asyninduction}
\end{align}
When $n=1$, the only sub-sequence is $\{i_j\}_{j=0}^T=\{0,1\} \in W_1$. It is clear that for $z \in D_1$, due to the contraction property of $T_{1}$:
$$
|J^*_1(z)-\widetilde{J}_1(z)| \leq \max_{\{i_j\}_{j=0}^T \in W_1}  A\big(\{i_j\}_{j=0}^T\big) = \alpha^{\Delta t_1} ||J^*_{1}-\widetilde{J}_{0}||_{\infty}.
$$
Assuming that Eq.~\ref{eqn:asyninduction} holds upto $n=m_k$, we need to prove that the equation also holds for those $n \in (m_{k},m_{k+1})$ and $n=m_{k+1}$. 
Indeed, let us assume that Eq.~\ref{eqn:asyninduction} holds for some $n \in [m_k,m_{k+1}-1)$. Denote $n_z \leq n$ as the index of the most recent update of $z$. For $z \in D_n$, we compute new values for $z$ in $\widetilde{J}_{n+1}$, and by the contraction property of $T_{n+1}$, it follows that
\begin{align*}
|\widetilde{J}_{n+1}(z) - J^*_{n+1}(z)| & \leq \alpha^{\Delta t_{n+1}} ||J^*_{n+1}-\widetilde{J}_n||_{\infty}  \\
& = \alpha^{\Delta t_{n+1}} \max_{z \in S_{n+1}}|J^*_{n+1}(z)-\widetilde{J}_n(z)| \\
&= \alpha^{\Delta t_{n+1}} \max_{z \in S_{n+1}}|J^*_{n+1}(z)-\widetilde{J}_{n_z}(z)|\\
& \leq \alpha^{\Delta t_{n+1}} \max_{z \in S_{n+1}} \big( | J^*_{n_z}(z) - \widetilde{J}_{n_z}(z)| + ||J^*_{n+1}-J^*_{n_z} ||_{\infty} \big)\\
&\leq \max_{z \in S_{n+1}} \Big( \alpha^{\Delta t_{n+1}} 
\max_{\{i_j\}_{j=0}^T \in W_{n_z}}  A\big(\{i_j\}_{j=0}^T\big)
 + \alpha^{\Delta t_{n+1}} ||J^*_{n+1}-J^*_{n_z}||_{\infty} \Big)\\
&=\max_{\{i_j\}_{j=0}^T \in W_{n+1}}  A\big(\{i_j\}_{j=0}^T\big).
\end{align*} 
The last equality is due to $n+1 \leq m_{k+1} -1 ,$ and $\{ m_j \}_{j=0}^{k} \subset  \{ \{i_j\}_{j=0}^T,n+1 \} \subset \{0,1,...,n+1\}$ for any $\{i_j\}_{j=0}^T \in W_{n_z}$. Therefore, Eq.~\ref{eqn:asyninduction} holds for all $n \in (m_k,m_{k+1}-1]$. When $n=m_{k+1}-1$, we also have the above relation for all $z \in D_{n+1}$:
\begin{align*}
|\widetilde{J}_{n+1}(z) - J^*_{n+1}(z)| & \leq \max_{z \in S_{n+1}} \Big( \alpha^{\Delta t_{n+1}} 
\max_{\{i_j\}_{j=0}^T \in W_{n_z}}  A\big(\{i_j\}_{j=0}^T\big)
 + \alpha^{\Delta t_{n+1}} ||J^*_{n+1}-J^*_{n_z}||_{\infty} \Big) \\
&=\max_{\{i_j\}_{j=0}^T \in W_{n+1}}  A\big(\{i_j\}_{j=0}^T\big).
\end{align*}
The last equality is due to $n+1=m_{k+1}$ and thus $\{ m_j \}_{j=0}^{k+1} \subset  \{ \{i_j\}_{j=0}^T,n+1 \} \subset \{0,1,...,n+1\}$ for any $ \{i_j\}_{j=0}^T \in W_{n_z}$. Therefore, Eq.~\ref{eqn:asyninduction} also holds for $n=m_{k+1}$ and this completes the induction.

We see that all $\{i_j\}_{j=0}^T \in W_n$, we have $j \leq i_j \leq m_j$, and thus $ j \leq S_{i_j} \leq j^{1/\dummyasyn}$. By Lemma \ref{lemma:sumtozero}, 
$$
\lim_{n \rightarrow \infty} A\big(\{i_j\}_{j=0}^T \in W_n \big) =0 \text{ w.p.1.}
$$ 
Therefore,
$$
\lim_{n \rightarrow \infty}\sup_{z \in D_n}|\widetilde{J}_{n}(z) - J^*_{n}(z)| =0 \text{ w.p.1.}
$$
Since all states are updated infinitely often, and $J^*_n$ converges uniformly to $J^*$ with probability one, we conlude that:
$
\lim_{n \rightarrow \infty} ||\widetilde{J}_n -J_n^*||_{\infty}=0 \text{ w.p.1.}
$
and
$
\lim_{n \rightarrow \infty} ||\widetilde{J}_n -J^*||_{\infty}=0 \text{ w.p.1.}
$

In both Steps S1 and S2, we have $\lim_{n \rightarrow \infty} ||J_n -J_n^*||_{\infty}=0$ w.p.1 \footnote{The tilde notion is dropped at this point.}, therefore $\mu_n$ converges to $\mu^*_n$ pointwise w.p.1 as $\mu_n$ and $\mu^*_n$ are induced from Bellman updates based on $J_n$ and $J^*_n$ respectively. Hence, the sequence of policies $\{ \mu_n \}_{n=0}^{\infty}$ has each policy $\mu_n$ as an $\epsilon_n$-optimal policy for the MDP $\mathcal{M}_n$ such that $\lim_{n\rightarrow \infty} \epsilon_n=0$. By Theorem~\ref{theorem:convergence_cost}, we conclude that
$$
\lim_{n \rightarrow \infty} |J_{n,\mu_n}(z) -J^*(z)| =0, \ \forall z \in S_n \text{ w.p.1.}
$$
\qed

\section{Proof of Theorem \ref{theoremAlmostSureConvergenceSampleControl}} 
We fix an initial starting state $x(0)=z$. In Theorem~\ref{theoremAlmostSureConvergence}, starting from an initial state $x(0)=z$, we construct a sequence of Markov chains $\{ \xi^n_i; i \in \naturals \}_{n=1}^{\infty}$ under minimizing control sequences $\{u^n_i;i \in \naturals \}_{n=1}^{\infty}$. By convention, we denote the associated interpolated continuous time trajectories and control processes as $\{ \xi^n(t); t \in \reals \}_{n=1}^{\infty}$ and $\{u^n(t); t \in \reals \}_{n=1}^{\infty}$ repsectively. 
By Theorem \ref{theorem:convergence_traj}, $\{ \xi^n(t); t \in \reals \}_{n=1}^{\infty}$ converges in distribution to an optimal trajectory $\{ x^*(t); t \in \reals \}$ under an optimal control process $\{ u^*(t); t \in \reals \}$ with probability one. In other words, $(\xi^n(\cdot),u^n(\cdot)) \overset {d}{\rightarrow} (x^*(\cdot),u^*(\cdot))$ w.p.1. We will show that this result can hold even when the Bellman equation is not solved exactly at each iteration.

In this theorem, we solve the Bellman equation (Eq.~\ref{eqn:Bellman}) by sampling uniformly in $U$ to form a control set $U_n$ such that $\lim_{n \rightarrow \infty}|U_n|=\infty$. Let us denote the resulting Markov chains and control sequences due to this modification as $\{ {\overline{\xi}}^n_i; i \in \naturals \}_{n=1}^{\infty}$ and $\{ {\overline{u}}^n_i; i \in \naturals \}_{n=1}^{\infty}$ with associated continuous time interpolations $\{ {\overline{\xi}}^n(t); t \in \reals \}_{n=1}^{\infty}$ and $\{ {\overline{u}}^n(t); t \in \reals \}_{n=1}^{\infty}$. In this case, randomness is due to both state and control sampling. We will prove that there exists minimizing control sequences $\{u^n_i; i \in \naturals \}_{n=1}^{\infty}$ and the induced sequence of Markov chains $\{ \xi^n_i; i \in \naturals \}_{n=1}^{\infty}$ in Theorem~\ref{theoremAlmostSureConvergence} such that
\begin{align}
(\overline{\xi}^n(\cdot) - \xi^n(\cdot),\overline{u}^n(\cdot) - u^n(\cdot)) \overset {p}{\rightarrow} (0,0), \label{eqn:convergesip1}
\end{align}
where $(0,0)$ denotes a pair of zero processes. To prove Eq.~\ref{eqn:convergesip1}, we first prove the following lemmas. In the following analysis, we assume that the Bellman update (Eq.~\ref{eqn:Bellman}) has minima in a neighborhood of positive Lebesgue measure. We also assume additional continuity of cost functions for discrete MDPs.

\begin{lemma} \label{lemma:bestcontrolinprob}
Let us consider the sequence of approximating MDPs $\{ {\cal M}_n \}_{n=0}^{\infty}$. For each $n$ and a state $z \in S_n$, let $v_n^*$ be an optimal control minimizing the Bellman update, which is refered to as an optimal control from $z$:
\begin{align*}
v_n^* \in V^*_n&=\argmin_{v \in U}\{ G_n(z,v) + \alpha^{\Delta t_n(z)} \EE_{P_n}\left[ J_{n-1}(y) |z,v\right]\},\\
J_n(z,v_n^*)&= J_n^*(z)=G_n(z,v_n^*) + \alpha^{\Delta t_n(z)} \EE_{P_n}\left[ J_{n-1}(y) |z,v_n^*\right], \ \ \forall v^*_n \in V^*_n.
\end{align*}
Let $\overline{v}_n$ be the best control in a sampled control set $U_n$ from $z$:
\begin{align*}
\overline{v}_n &= \argmin_{v \in U_n}\{ G_n(z,v) + \alpha^{\Delta t_n(z)} \EE_{P_n}\left[ J_{n-1}(y) |z,v\right]\}, \\
J_n(z,\overline{v}_n)&=G_n(z,\overline{v}_n) + \alpha^{\Delta t_n(z)} \EE_{P_n}\left[ J_{n-1}(y) |z,\overline{v}_n \right].
\end{align*}
Then, when $\lim_{n \rightarrow \infty} |U_n|=\infty$, we have $|J_n(z,\overline{v}_n) - J_n^*(z)| \overset {p}{\rightarrow} 0$ as $n$ approaches $\infty$, and there exists a sequence $\{ v_n^* \ | \ v_n^*  \in V^*_n \}_{n=0}^{\infty}$ such that $||\overline{v}_n - v_n^*||_2 \overset {p}{\rightarrow} 0$.
\end{lemma}
\begin{proof}
We assume that for any $\epsilon>0$, the set $A^n_{\epsilon}=\{v \in U |\ |J_n(z,v)-J^*_n(z)| \leq \epsilon \}$ has positive Lebesgue measure. That is, $m(A^n_{\epsilon})>0$ for all $\epsilon >0$ where $m$ is Lebesgue measure assigned to $U$. For any $\epsilon >0$, we have:
$$
\PP\big(\{|J_n(z,\overline{v}_n) - J^*_n(z)| \geq \epsilon \}\big)= \big( 1-m(A^n_{\epsilon})/m(U) \big)^{|U_n|}.
$$
Since $1-m(A^n_{\epsilon})/m(U)  \in [0,1)$ and $\lim_{n \rightarrow \infty}|U_n|=\infty$, we infer that:
$$
\lim_{n \rightarrow \infty} \PP\big(\{|J_n(z,\overline{v}_n) - J^*_n(z)| \geq \epsilon \}\big) = 0.
$$
Hence, we conclude that $|J_n(z,\overline{v}_n) - J^*_n(z)| \overset {p}{\rightarrow} 0$ as $n \rightarrow \infty$. Under the mild assumption that $J_n(z,v)$ is continuous on $U$ for all $z \in S_n$, thus there exists a sequence $\{ v_n^* \ | \ v_n^* \in V^*_n \}_{n=0}^{\infty}$ such that $||\overline{v}_n - v_n^*||_2 \overset {p}{\rightarrow} 0$ as $n$ approaches $\infty$. 
\qed
\end{proof}
By Lemma~\ref{lemma:bestcontrolinprob}, we conclude that $||J_n-J^*_n||_{\infty}$ converges to $0$ in probability. Thus, $J_n$ returned from the iMDP algorithm when the Bellman update is solved via sampling converges uniformly to $J^*$ in probability. We, however, claim that $J_{n,\mu_n}$ still converges pointwise to $J^*$ almost surely in the next discussion. 

\begin{lemma} \label{lemma:controlconverge2}
With the notations in Lemma \ref{lemma:bestcontrolinprob}, consider two states $\xi^n_{0}$ and $\overline{\xi}^n_0$ such that $||\overline{\xi}^n_0  - \xi^n_{0}||_2 \overset {p}{\rightarrow} 0$ as $n$ approaches $\infty$.  Let $\overline{\xi}^n_1$ be the next random state of $\overline{\xi}^n_0$ under the best sampled control $\overline{v}_n$ from $\overline{\xi}^n_0$. Then, there exists a sequence of optimal controls $v^*_n$ from $\xi^n_0$ such that $||\overline{v}_n - v_n^*||_2 \overset {p}{\rightarrow} 0$ and $||\overline{\xi}^n_1 - \xi^n_1||_2 \overset {p}{\rightarrow} 0$ as $n$ approaches $\infty$, where $\xi^n_1$ is the next random state of $\xi^n_{0}$ under the optimal control $v_n^*$ from $\xi^n_{0}$.
\end{lemma}
\begin{figure*}
\begin{center}
  \includegraphics[scale=0.25]{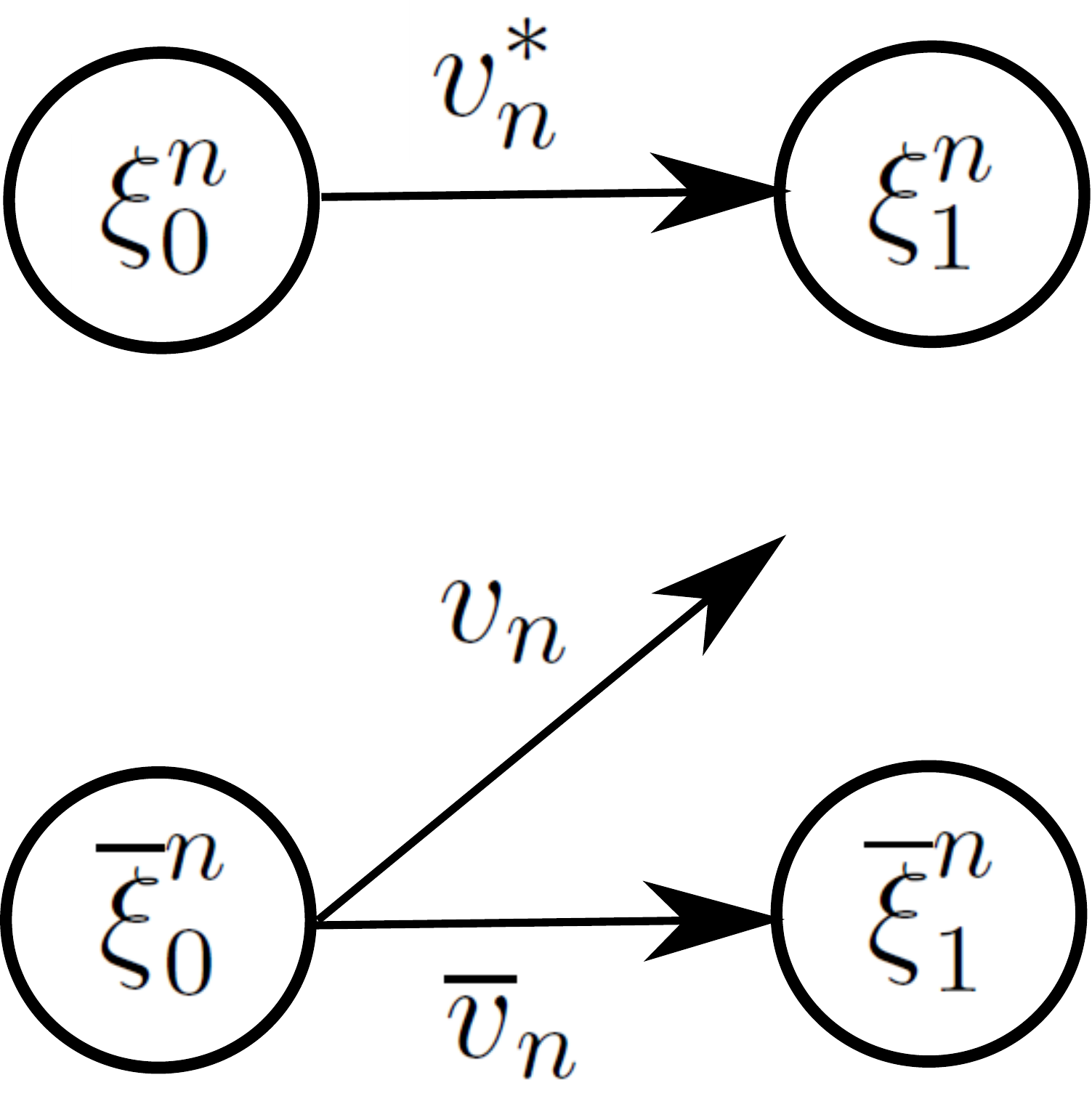}
\end{center}
  \vspace{-0.1in}
  \caption{An illustration for Lemma~\ref{lemma:controlconverge2}. We have $\overline{\xi}^n_0$ converges in probability to $\xi^n_0$. From $\xi^n_0$, the optimal control is $v^*_n$ that results in the next random state $\xi^n_1$. From $\overline{\xi}^n_0$, the optimal control and the best sampled control are $v_n$ and $\overline{v}_n$ respectively. The next random state from $\overline{\xi}^n_0$ due to the control $\overline{v}_n$ is $\overline{\xi}^n_1$.}
  \label{figSampleUDraw}
\end{figure*}
\begin{proof}
We have $\overline{v}_n$ as the best sampled control from $\overline{\xi}^n_0$. By Lemma~\ref{lemma:bestcontrolinprob}, there exists a sequence of optimal controls $v_n$ from $\overline{\xi}^n_0$ such that $||\overline{v}_n - v_n||_2 \overset {p}{\rightarrow} 0 $. 
We assume that the mapping from state space $S_n$, which is endowed with the usual Euclidean metric, to optimal controls in $U$ is continuous. As $||\overline{\xi}^n_0  - \xi^n_{0}||_2 \overset {p}{\rightarrow} 0$, there exists a sequence of optimal controls ${v^*_n}$ from $\xi^n_{0}$ such that $||v_n  - v^*_n||_2 \overset {p}{\rightarrow} 0$. Now, $||\overline{v}_n  - v_n||_2 \overset {p}{\rightarrow} 0$ and $||v_n  - v^*_n||_2 \overset {p}{\rightarrow} 0$ lead to $||\overline{v}_n  - v^*_n||_2 \overset {p}{\rightarrow} 0$ as $n \rightarrow \infty$. Figure~\ref{figSampleUDraw} illustrates how $\overline{v}_n,v_n$, and $v^*_n$ relate $\overline{\xi}^n_1$ and $\xi^n_1$.

Using the probability transition $P_n$ of the MDP ${\cal M}_n$ that is locally consistent with the original continuous system, we have:
\begin{align*}
&\EE[ \xi^n_1 \ | \ \xi^n_0,u^n_0=v^*_n] = \xi^n_0+f(\xi^n_0,v^*_n)\Delta t_n(\xi^n_0) + o(\Delta t_n(\xi^n_0)),\\
&\EE [ \overline{\xi}^n_1 \ | \ \overline{\xi}^n_0,\overline{u}^n_0=\overline{v}_n] = \overline{\xi}^n_0+f(\overline{\xi}^n_0,\overline{v}_n)\Delta t_n(\overline{\xi}^n_0) + o(\Delta t_n(\overline{\xi}^n_0)),\\
&Cov[ \xi^n_1 \ | \ \xi^n_0,u^n_0=v^*_n] = F(\xi^n_0,v^*_n)F(\xi^n_0,v^*_n)^T\Delta t_n(\xi^n_0) + o(\Delta t_n(\xi^n_0)),\\
&Cov [ \overline{\xi}^n_1 \ | \ \overline{\xi}^n_0,\overline{u}^n_0=\overline{v}_n] = F(\overline{\xi}^n_0),\overline{v}_n)F(\overline{\xi}^n_0),\overline{v}_n)^T\Delta t_n(\overline{\xi}^n_0)) + o(\Delta t_n(\overline{\xi}^n_0))),
\end{align*}
where $f(\cdot,\cdot)$ is the nominal dynamics, and $F(\cdot,\cdot)F(\cdot,\cdot)^T$ is the diffusion of the original system that are assumed to be continuous almost everywhere. We note that $\Delta t_n(\overline{\xi}^n_0) = \Delta t_n(\xi^n_0) = \gamma_t \big(\log(|S_n|)/|S_n| \big)^{\dummyasyn \dummyrate \rho / d_x}$ as $\overline{\xi}^n_0$ and $\xi^n_0$ are updated at the $n^{th}$ iteration in this context, and the holding times converge to $0$ as $n$ approaches infinity.
Therefore, when $||\overline{\xi}^n_0 - \xi^n_{0}||_2 \overset {p}{\rightarrow} 0$, $||\overline{v}_n  - v^*_n||_2 \overset {p}{\rightarrow} 0$, we have:
\begin{align}
&\EE [ \overline{\xi}^n_1 - {\xi}^n_1 \ | \ \xi^n_0,\overline{\xi}^n_0, {u}^n_0=v^*_n, \overline{u}^n_0=\overline{v}_n] \overset {p}{\rightarrow} 0, \label{eqn:mclim1}\\
&Cov(\overline{\xi}^n_1 - {\xi}^n_1 \ | \ \xi^n_0,\overline{\xi}^n_0, {u}^n_0=v^*_n, \overline{u}^n_0=\overline{v}_n) \overset {p}{\rightarrow} 0. \label{eqn:mclim2}
\end{align}
Since $\overline{\xi}^n_1$ and ${\xi}^n_1$ are bounded, the random vector $\EE [ \overline{\xi}^n_1 - {\xi}^n_1 \ | \ \xi^n_0,\overline{\xi}^n_0, {u}^n_0=v^*_n, \overline{u}^n_0=\overline{v}_n]$ and random matrix $Cov (\overline{\xi}^n_1 - {\xi}^n_1 \ | \ \xi^n_0,\overline{\xi}^n_0, {u}^n_0=v^*_n, \overline{u}^n_0=\overline{v}_n)$ are bounded. We recall that if $Y_n \overset {p}{\rightarrow} 0$, and hence $Y_n \overset {d}{\rightarrow} 0$, when $Y_n$ is bounded for all $n$, $\lim_{n \rightarrow \infty}\EE[Y_n]=0$ and $\lim_{n \rightarrow \infty}Cov(Y_n)=0$. Therefore, Eqs.~\ref{eqn:mclim1}-\ref{eqn:mclim2} imply:
\begin{align}
&\lim_{n \rightarrow \infty} \EE \Big [ \EE [ \overline{\xi}^n_1 - {\xi}^n_1 \ | \ \xi^n_0,\overline{\xi}^n_0, {u}^n_0=v^*_n, \overline{u}^n_0=\overline{v}_n] \Big ] = 0, \label{eqn:mclim3}\\
&\lim_{n \rightarrow \infty} \Cov \Big ( \EE [ \overline{\xi}^n_1 - {\xi}^n_1 \ | \ \xi^n_0,\overline{\xi}^n_0, {u}^n_0=v^*_n, \overline{u}^n_0=\overline{v}_n] \Big ) = 0, \label{eqn:mclim4}\\
&\lim_{n \rightarrow \infty} \EE \Big [ Cov(\overline{\xi}^n_1 - {\xi}^n_1 \ | \ \xi^n_0,\overline{\xi}^n_0, {u}^n_0=v^*_n, \overline{u}^n_0=\overline{v}_n) \Big ] = 0. \label{eqn:mclim5}
\end{align}
The above outer expectations and covariance are with resepect to the randomness of states $\xi^n_0$, $\overline{\xi}^n_0$ and sampled controls $U_n$.
Using the iterated expectation law for Eq.~\ref{eqn:mclim3}, we obtain:
$$
\lim_{n \rightarrow \infty} \EE [ \overline{\xi}^n_1 - \xi^n_1] = 0. 
$$
Using the law of total covariance for Eqs.~\ref{eqn:mclim4}-\ref{eqn:mclim5}, we have:
$$
\lim_{n \rightarrow \infty} Cov[ \overline{\xi}^n_1 - \xi^n_1] = 0. 
$$ 
Since
$$
\EE[||\overline{\xi}^n_1 - {\xi}^n_1||^2_2]=\EE[(\overline{\xi}^n_1 - {\xi}^n_1)^T(\overline{\xi}^n_1 - {\xi}^n_1)] = ||\EE[\overline{\xi}^n_1 - {\xi}^n_1)]||^2_2 + tr(Cov[ \overline{\xi}^n_1 - \xi^n_1]),
$$
the above limits together imply:
$$
\lim_{n \rightarrow \infty} \EE[||\overline{\xi}^n_1 - {\xi}^n_1||^2_2] =0.
$$
In other words, $\overline{\xi}^n_1$ converges in $2^{th}$-mean to $\xi^n_1$, which leads to $||\overline{\xi}^n_1 - \xi^n_1||_2 \overset {p}{\rightarrow} 0$ as $n$ approaches $\infty$.
\qed
\end{proof}
Returning to the proof of Eq.~\ref{eqn:convergesip1}, we know that $\xi^n_0=\overline{\xi}^n_0=z$ as the starting state. 
From any $y \in S_n$, an optimal control from $y$ is denoted as $v^{*}(y)$, and the best sampled control from the same state $y$ is denoted as $\overline{v}(y)$. 

By Lemma~\ref{lemma:controlconverge2}, as $\overline{u}^n_0=\overline{v}(\overline{\xi}^n_0)$, there exists $u^n_0=v^{*}(\xi^n_0)$ such that $||\overline{u}^n_0 - u^n_0||_2 \overset {p}{\rightarrow} 0$ and $||\overline{\xi}^n_1 - \xi^n_1 ||_2\overset {p}{\rightarrow} 0$. Let us assume that  $(||\overline{u}^n_{k-1} - u^n_{k-1}||_2, ||\overline{\xi}^n_k -  \xi^n_k||_2)$ converges in probability to $(0,0)$ upto index $k$. We have $\overline{u}^n_k=\overline{v}(\overline{\xi}^n_k)$. Using Lemma~\ref{lemma:controlconverge2}, there exists $u^n_{k}=v^{*}(\xi^n_k)$ such  that $(||\overline{u}^n_{k} - u^n_{k}||_2,||\overline{\xi}^n_{k+1} - \xi^n_{k+1}||_2) \overset {p}{\rightarrow} (0,0)$. Thus, for any $i \geq 1$, we can construct a minimizing control $u^n_i$ in Theorem~\ref{theoremAlmostSureConvergence} such that $(||\overline{\xi}^n_i - \xi^n_i||_2,||\overline{u}^n_i - u^n_i||_2) \overset {p}{\rightarrow} (0,0)$ as $n \rightarrow \infty$. Hence, Eq.~\ref{eqn:convergesip1} follows immediately:
$$
(\overline{\xi}^n(\cdot) - \xi^n(\cdot),\overline{u}^n(\cdot) - u^n(\cdot)) \overset {p}{\rightarrow} (0,0).
$$
We have $(\xi^n(\cdot),u^n(\cdot)) \overset {d}{\rightarrow} (x^*(\cdot),u^*(\cdot))$ w.p.1. Thus, by hierarchical convergence of random variables \cite{GrimmetProb}, we achieve
$$
(\overline{\xi}^n(\cdot),\overline{u}^n(\cdot)) \overset {d}{\rightarrow} (x^*(\cdot),u^*(\cdot)) \text{ w.p.1.}
$$ 
Therefore, for all $z \in S_n$:
$$
\lim_{n \rightarrow \infty} |J_{n,\mu_n}(z) - J^*(z)|=0 \text{ w.p.1.}
$$
\qed

\section{Proof of Theorem \ref{theoremControlQuality}}
Fix $n \in \naturals$, for all $z \in S$, and $y_n=\argmin_{z' \in S_n} ||z'-z||_2$, we have
$$
\overline{\mu}_n(z)= \mu_n(y_n).
$$
We assume that optimal policies of the original continuous problem are obtainable. By Theorems~\ref{theoremAlmostSureConvergence}-\ref{theoremAlmostSureConvergenceSampleControl}, we have:
$$
\lim_{n \rightarrow \infty} |J_{n,\mu_n}(y_n) - J^*(y_n))|=0 \text{ w.p.1.}
$$ 
Thus, $\mu_n(y_n)$ converges to $\mu^*(y_n)$ almost surely where $\mu^*$ is an optimal policy of the original continuous problem. Thus, for all $\epsilon >0$, there exists $N$ such that for all $n > N$:
$$
||\mu_n(y_n)  - \mu^*(y_n)||_2 \leq \frac{\epsilon}{2} \text{ w.p.1.}
$$
Under the assumption that $\mu^*$ is continuous at $z$, and due to $\lim_{n \rightarrow \infty} y_n =z$ almost surely, we can choose $N$ large enough such that for all $n > N$:
$$
||\mu^*(y_n)  - \mu^*(z)||_2 \leq \frac{\epsilon}{2} \text{ w.p.1.}
$$
From the above inequalities:
$$
||\mu_n(y_n)  - \mu^*(z)||_2 \leq ||\mu_n(y_n)  - \mu^*(y_n)||_2 + ||\mu^*(y_n)  - \mu^*(z)||_2 \leq \epsilon, \ \forall n>N \text{ w.p.1.}
$$
Therefore,
$$
\lim_{n \rightarrow \infty} ||\overline{\mu}_n(z)  - \mu^*(z)||_2= \lim_{n \rightarrow \infty} ||\mu_n(y_n)  - \mu^*(z)||_2 =0 \text{ w.p.1.}
$$
\qed
\end{document}